\RequirePackage{fix-cm}
\documentclass[twocolumn, natbib]{svjour3}          
\smartqed  

\usepackage[T1]{fontenc}
\usepackage[utf8]{inputenc}

\usepackage{mathptmx}      

\usepackage{graphicx}

\usepackage{subfiles}
\usepackage{hyperref}

\usepackage{amsmath}
\usepackage{amsfonts}
\usepackage{amssymb}
\usepackage{siunitx} 

\usepackage{array} 
\usepackage{makecell} 

\usepackage{stmaryrd} 
\SetSymbolFont{stmry}{bold}{U}{stmry}{m}{n} 

\usepackage{varwidth} 

\usepackage{calc} 
\usepackage[dvipsnames]{xcolor}
\usepackage{tikz}

\usepackage{textcomp}
\usepackage{gensymb}

\usepackage{multirow}

\usepackage{algorithmic, algorithm}

\usepackage{bm} 
\newcommand{\mat}[1]{\vec{#1}} 
\newcommand*{\matLambda}{\bm{\Lambda}}
\newcommand*{\matSigma}{\bm{\sigma}}

\DeclareMathOperator{\Tr}{Tr} 
\DeclareMathOperator{\vect}{vec} 
\DeclareMathOperator*{\argmin}{\arg\!\min}

\DeclareMathOperator{\diag}{diag}
\DeclareMathOperator{\rank}{rank}

\DeclareMathOperator{\proj}{proj}
\DeclareMathOperator{\mean}{mean}
\DeclareMathOperator{\mydist}{d}
\DeclareMathOperator{\mydistrot}{d_{rot}}
\DeclareMathOperator{\mydisttrans}{d_{trans}}
\DeclareMathOperator{\mydistrotsquare}{d_{rot}^2} 
\DeclareMathOperator{\mydistnosym}{d_{no\_sym}}
\DeclareMathOperator{\mydistnosymsquare}{d_{no\_sym}^2}

\DeclareMathOperator{\frechetvariance}{\Phi}

\usepackage{xspace}
\newcommand*{\eg}{e.g.\@\xspace}
\newcommand*{\ie}{i.e.\@\xspace}

\makeatletter
\newcommand*{\etc}{%
    \@ifnextchar{.}%
        {etc}%
        {etc.\@\xspace}%
}

\definecolor{color1}{RGB}{0, 90, 212} 
\definecolor{color2}{RGB}{255, 200, 50} 
\definecolor{color3}{RGB}{120, 209, 81} 

\newcommand\myurl[1]{\nolinkurl{#1}}

\newlength{\mynegativeindentforframeddefinitions}
\newcommand*{\vsepfbox}[1]{%
  \begingroup
    \sbox0{\fbox{#1}}%
    \setlength{\fboxrule}{0pt}%
    \mbox{\kern-\fboxsep\fbox{\unhbox0}\kern-\fboxsep}%
  \endgroup
}
\newcommand*{\myframedminipage}[1]{%
\par\noindent\vsepfbox{%
\noindent\begin{minipage}{\dimexpr\linewidth-2\fboxrule-2\fboxsep}#1\end{minipage}}}

\newcommand*{\bibliographyInSubfile}{%
\bibliographystyle{spbasic}      
\bibliography{biblio}   
}

\journalname{}
\begin{document}
\renewcommand*{\bibliographyInSubfile}{}

\title{Defining the Pose of any 3D Rigid Object and an Associated Distance}



\author{Romain Brégier \and Frédéric Devernay \and Laetitia Leyrit \and James L. Crowley}


\institute{R. Brégier, L. Leyrit \at
              Siléane, Saint-Étienne, FRANCE \\
              \email{r.bregier at sileane.com} \\
           \and
           R. Brégier, F. Devernay, J. L. Crowley \at
              Inria Grenoble Rhône-Alpes \\
              UGA (Université Grenoble Alpes)
}

\date{}

\maketitle

\begin{abstract}

The pose of a rigid object is usually regarded as a rigid transformation, described by a translation and a rotation. 
However, equating the pose space with the space of rigid transformations is in general abusive, as it does not account for objects with proper symmetries -- which are common among man-made objects.
In this article, we define pose as a distinguishable static state of an object, and equate a pose with a set of rigid transformations.
Based solely on geometric considerations, we propose a frame-invariant metric on the space of possible poses, valid for any physical rigid object, and requiring no arbitrary tuning. This distance can be evaluated efficiently using a representation of poses within a Euclidean space of at most 12 dimensions depending on the object's symmetries. This makes it possible to efficiently perform neighborhood queries such as \emph{radius searches} or \emph{k-nearest neighbor searches} within a large set of poses using off-the-shelf methods. Pose averaging considering this metric can similarly be performed easily, using a projection function from the Euclidean space onto the pose space.
The practical value of those theoretical developments is illustrated with an application of pose estimation of instances of a 3D rigid object given an input depth map, via a Mean Shift procedure.
\end{abstract}


\section{Introduction}
Rigid body models play an important role in many technical and scientific fields, including physics, mechanical engineering, computer vision or 3D animation.
Under the rigid body assumption, the static state of an object is referred to as a \emph{pose}, and is often described in term of a \emph{position} and an \emph{orientation}.

Poses of a 3D rigid object are in general regarded as rigid transformations and the set of poses is identified with the set of rigid transformations $SE(3)$, the special Euclidean group.
The Lie group structure of $SE(3)$ makes the relative displacement of the object between two poses explicit and thus enables the definition of a distance between poses as the length of a shortest motion performing this displacement.
This identification is particularly meaningful for applications where the motion of a rigid body is considered  -- such as motion planning \citep{sucan2012} or object tracking~\citep{tjaden2016}.
However while there already exists numerous works regarding $SE(3)$ metrics, choosing how to deal with poses of an object still remains challenging, as practitioners face questions such as how to tune the relative importance of position and orientation, even in the current deep learning era~\citep{kendall2015}.
	
There are applications in which motion considerations are irrelevant, and for which only a notion of \emph{similarity} between poses is required. Pose estimation of instances of a rigid object based on a noisy set of votes is a good example of such a problem.
While motion-based applications rely on local properties of the pose space which have been the subject of a large amount of research work, applications based on similarity have to deal with numerous poses at once, performing operations such as neighborhood queries -- \ie finding poses in a set of poses similar to a given one --  or pose averaging and have not gathered as much theoretical interest.
Consequently, similarity measures suffering from major flaws are still used in practical applications.
Following the work of \citet{fanelli2011}, \citet{tejani2014} use a Mean Shift procedure based on the Euclidean distance between Euler angles as a representation of poses in their state-of-the-art object pose estimation method. Such a measure is fast to compute and enables the use of efficient tools developed for Euclidean spaces to perform the neighborhood queries and pose averaging required for Mean Shift, but it is not a distance. The parametrization of a rotation based on Euler angles notoriously suffers from border effects, singularities, and is dependent on the choice of frame.
These issues may only have limited effects on the results announced by the authors, thanks to an appropriate choice of frames orientation and to the low variability of objects orientations within their datasets. Nonetheless, they cannot be avoided when dealing with the general case of poses having arbitrary orientations.
Such an example expresses the lack of tools for dealing efficiently with large sets of poses.
	
Lastly, there are cases where the pose of a rigid object cannot be identified as a single rigid transformation and therefore, for which existing results cannot be applied. Such cases occur when dealing with objects showing symmetry properties such as revolution objects or cuboids, and are, in fact, common among manufactured objects.
The existing literature on object pose estimation does not usually discuss how such objects are handled, and the most widespread validation method used for symmetrical objects~\citep{hinterstoisser2012} consists in a relaxed similarity measure that cannot distinguish between poses such as a cylindrical can being flipped up or down.

Our goal in this paper is to address those issues by providing a consistent and general framework for dealing with any kind of physically admissible rigid object in practical applications. 
To this end, we propose a pose definition valid for any bounded rigid object, equivalent to a set of rigid transformations (section \ref{sec:pose_definition}). We then propose a physically meaningful distance over the pose space (section \ref{sec:proposed_metric}), and show how poses can be represented in a Euclidean space to enable fast distance computations and neighborhood queries (section~\ref{sec:pose_representation}). We show how the pose averaging problem can be solved quite efficiently (section~\ref{sec:averaging_poses}) for this metric using a projection technique (section~\ref{sec:projection}) and lastly we propose an example application for the problem of pose estimation of instances of a rigid object given a set of votes.

\bibliographyInSubfile

\section{A definition for pose}
\label{sec:pose_definition}
While the notion of pose of a rigid object is widely used, \eg in robotics or computer vision, we have not found in the literature a general definition. We therefore propose the following one:
\myframedminipage{%
\begin{definition}
A \emph{pose} of a rigid object is a distinguishable static state of this object.
\end{definition}}
We will refer to the set of possible poses as a \emph{pose space} which we will denote $\mathcal{C}$ for consistency with the notion of \emph{configuration space} in robotics literature.

\subsection{Link between the pose space and  \texorpdfstring{$SE(3)$}{SE(3)}}
\label{sec:pose_and_se3}

A pose space is highly related to the group of rigid transformations $SE(3)$. Let us consider a rigid object, and $\mathcal{P}_0 \in \mathcal{C}$ an arbitrary reference pose for this object.

A rigid transformation applied to the object at its reference pose defines a static state of the object, \ie a pose. In a similar way, a pose $\mathcal{P} \in \mathcal{C}$  of the object can be reached through a rigid displacement from the reference pose $\mathcal{P}_0$, and therefore $\mathcal{P}$ can be described completely by the rigid transformation corresponding to this displacement. 
		
We will denote $\mathcal{P} \in \mathcal{C}$ and $\mat{T} = \left( \mat{R}, \mat{t} \right) \in SE(3)$ as a couple of pose and rigid transformation -- with $\mat{R} \in SO(3)$ a rotation matrix and $ \mat{t} \in \mathbb{R}^3$ a translation vector. The transformation considered here is such that each point $\mat{x} \in \mathbb{R}^3$ linked to an object instance at reference pose $\mathcal{P}_0$ is transformed by $\mat{T}$ into the corresponding point $\mat{T}(\mat{x})$ of an instance at pose $\mathcal{P}$ as follows and such as depicted in figure~\ref{fig:corresponding_points}:
\begin{equation}
\label{eq:pose_as_rigid_motion}
\mat{T} (\mat{x}) = \mat{R} \mat{x} + \mat{t}
\end{equation}

However, the rigid transformation corresponding to a given pose is not necessarily unique and therefore the identification of $SE(3)$ with the pose space is in the general case incorrect.	
Objects --- and especially manufactured ones --- may indeed show some proper symmetry properties that make them invariant to some rigid displacements.
	
\subsection{Pose as equivalence class of \texorpdfstring{$SE(3)$}{SE(3)}}
\label{sec:pose_as_equivalency_class_of_se3}
Let $M \subset SE(3)$ be the set of rigid transformations representing the same pose as a rigid transformation $\mat{T}$. 
For the \emph{bunny} object figure~\ref{tab:symmetries_classes}d, $M$ typically consists in the singleton $\lbrace \mat{T} \rbrace$. But $M$ can also contain a continuum of poses in the case of a revolution object such as the \emph{candlestick} figure~\ref{tab:symmetries_classes}a, or even be a discrete set, such as for the \emph{rocket} object depicted figure~\ref{tab:symmetries_classes}e where the same pose can be represented by 3 different transformations.

 By definition of $M$, $G \triangleq \{ \mat{T}^{-1} \circ \mat{M}, \mat{M} \in M \}$ is the set of rigid transformations that have no effect on the static state of the object. This set therefore does not depend on the arbitrary transformation $\mat{T}$ considered. It is moreover a subgroup of $SE(3)$. Indeed, combinations and inversions of such transformations can be applied to the object while leaving it unchanged, and the identity transformation has obviously no effects on the pose of the object. We will refer to the elements of this group as the \emph{proper symmetries} of the object and to $G \subset SE(3)$ as the group of \emph{proper symmetries} of the object.

Given a rigid transformation $\mat{T}$ defining a pose $\mathcal{P}$, we can therefore identify $\mathcal{P}$ to the following equivalence class $[\mat{T}] \subset SE(3)$, consisting in the combination of $\mat{T}$ with any rigid transformation that has no effect on the pose of the object:
\begin{equation}
\label{eq:pose_as_equivalence_class}
\boxed{
\mathcal{P} = [\mat{T}] \triangleq \left\lbrace \mat{T} \circ \mat{G}, \mat{G} \in G \right\rbrace.
}
\end{equation}	

\subsection{The proper symmetry group}
\label{sec:group_of_proper_symmetry}

In the following, we propose a classification of the potential groups of proper symmetries for a physically meaningful bounded object. While models of infinite objects are commonly used \eg for plane detection in 3D scene analysis, we do not consider those in this article as they do not correspond to actual physical objects and the definition of a suitable metric on the pose space of such objects is typically very dependent on the application. This classification will be helpful to derive the practical results associated with our proposed distance.

All proper symmetries of a bounded object necessarily have a common fixed point, thus we can consider the group of proper symmetries as a subgroup of the rotation group $SO(3)$ by choosing such a point as the origin of the object frame.
Subgroups of $SO(3)$ are sometimes referred to as chiral point groups, and have been widely studied, notably in the context of crystallography. The interested reader is referred to \citet{vainsthein1994} for more insight on the theory of symmetry.
	
Ignoring the pathological case of infinite subgroups of $SO(3)$ that are not closed under the usual topology as they do not make sense physically, 
the potential groups of proper symmetries for a bounded object can actually be classified in a few categories.

In the 2D case, a bounded object will either show a circular symmetry -- \ie an invariance by any 2D rotation -- or a cyclic symmetry of order $n \in \mathbb{N}^{*}$ -- \ie an invariance by rotation of $1/n$ turn. The special case $n=1$ actually corresponds to a 2D object without any proper symmetry. Table~\ref{tab:symmetry_classification_2d} provides examples of such objects. 

Similarly, we distinguish  in the 3D case between five classes of proper symmetry groups, synthesized in table \ref{tab:symmetries_classes}. A 3D bounded object can show a spherical symmetry -- \ie an invariance by any rotation -- or a revolution symmetry  -- \ie an invariance by rotation along a given axis of any angle.
This latter class can actually be split into two, depending on whether the object is also invariant under reflection across a plane that is orthogonal to the revolution axis or not. We respectively refer to these classes as revolution symmetry with or without rotoreflection invariance.
In addition, we should also consider finite groups of proper symmetry, but there are an infinite number of them therefore they are considered in a general manner.
We nevertheless distinguish the case of an object without proper symmetry (\ie for which $G$ contains only the identity transformation)
from the other ones, because it is essential in our theoretical developments.
	
	Note that potential indirect symmetries of the object such as reflection symmetries are not accounted for. This is due to the fact that we consider an oriented 3D space -- \eg through the right-hand rule -- in which reflections are not physically feasible through rigid displacements. Revolution symmetry with rotoreflection invariance is nonetheless considered since it is a proper symmetry group: the reflection symmetry can indeed be generated by the introduction of a rotational invariance of 180° along an arbitrary axis orthogonal to the revolution axis.
\begin{table*}
\small\sf\centering
\caption{\label{tab:symmetries_classes}Classification of the potential groups of proper symmetries for a 3D bounded physical object.}
\renewcommand{\arraystretch}{1.3}
\begin{tabular}{c|c|c|c|c}
\noalign{\rule{\linewidth}{1.5pt}}
\multicolumn{3}{c|}{Infinite groups} & \multicolumn{2}{c}{Finite groups} \\
\noalign{\rule{\linewidth}{1pt}}
\multicolumn{2}{c|}{Revolution symmetry} &  &  & \\
\cline{1-2}
\includegraphics[width=0.1\linewidth]{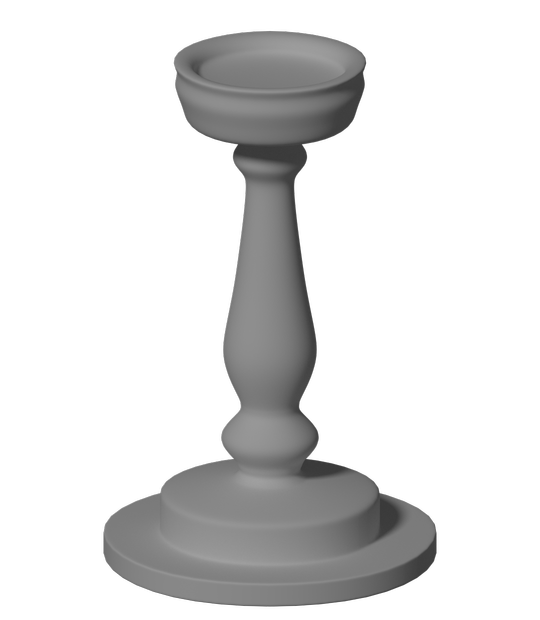} & 
\includegraphics[width=0.1\linewidth]{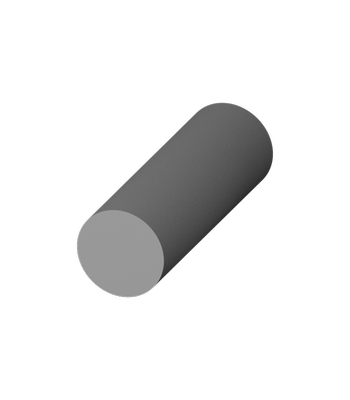} & 
\includegraphics[width=0.1\linewidth]{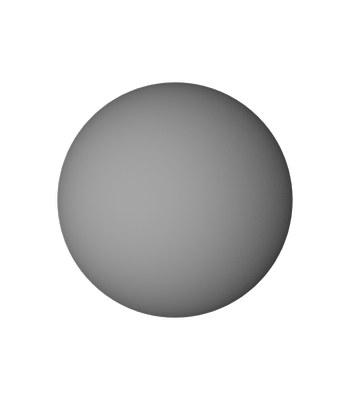}& 
\includegraphics[width=0.1\linewidth]{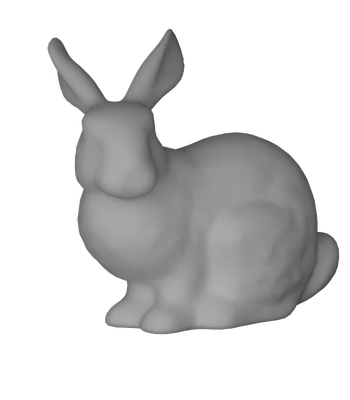} &
\includegraphics[width=0.1\linewidth]{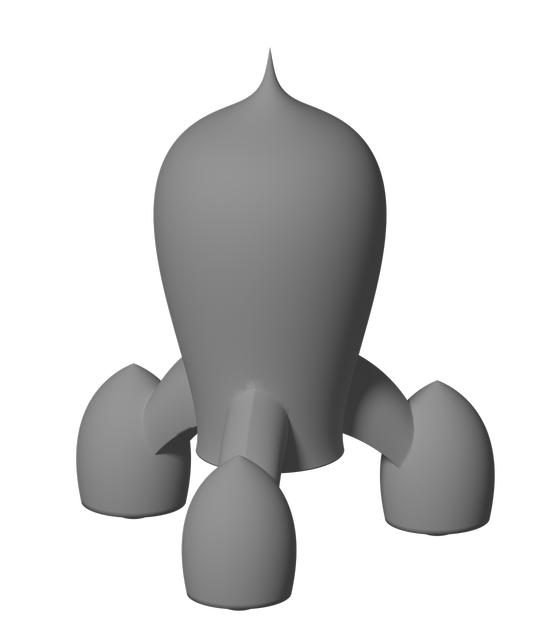} \\
\makecell[c]{(a) Without \\ rotoreflection invariance} & \makecell{(b) With \\ rotoreflection invariance} & \makecell{(c) Spherical symmetry} & \makecell{(d) No proper symmetry} & (e) Finite non trivial \\
\noalign{\rule{\linewidth}{1.5pt}}
\end{tabular}
\end{table*}

\begin{table}
\small\sf\centering
\caption{\label{tab:symmetry_classification_2d}Classification of the potential groups of proper symmetries for a 2D bounded physical object.}
\renewcommand{\arraystretch}{1.3}
\setlength{\extrarowheight}{0pt}
\begin{tabular}{c@{\hskip7pt}c@{\hskip7pt}c} 
\noalign{\hrule height 1.5pt}
\rule{0pt}{1cm}
\begin{tikzpicture}[scale = 0.3] 
\coordinate (o) at (0, 0);		
\coordinate (a) at (-0.8, -0.4);
\coordinate (b) at (0, -0.9);
\coordinate (c) at (0.8, -0.7);
\coordinate (d) at (0.8, 0);
\coordinate (e) at (0.3, 0.5);
\coordinate (f) at (0.3, 1);
\coordinate (g) at (0.9, 1.5);
\coordinate (h) at (0.5, 1.8);
\coordinate (i) at (-0.4, 1);
\coordinate (j) at (-0.7, 0);
\draw plot[smooth cycle, tension=0.5] coordinates {(a) (b) (c) (d) (e) (f) (g) (h) (i)};
\end{tikzpicture}
&
\newcommand{\DrawFlower}[4]{
	\draw [shift={({#1}, {#2})}, rotate = {#3}, #4] plot[smooth cycle, tension=0.5] coordinates {(0, 0) (0.8, -0.2) (1, 0) (0.8, 0.2)};
	\draw [shift={({#1}, {#2})}, rotate = {#3 + 120}] plot[smooth cycle, tension=0.5] coordinates {(0, 0) (0.8, -0.2) (1, 0) (0.8, 0.2)};
	\draw [shift={({#1}, {#2})}, rotate = {#3 - 120}] plot[smooth cycle, tension=0.5] coordinates {(0, 0) (0.8, -0.2) (1, 0) (0.8, 0.2)};
}
\begin{tikzpicture}[scale = 0.5] 
\DrawFlower{0}{0}{30}{};
\end{tikzpicture}
&
\begin{tikzpicture}[scale = 0.3]
\draw (0,0) circle (1.0);
\end{tikzpicture}\\
\makecell{No proper symmetry} & \makecell{Cyclic symmetry \\ (finite non trivial)} & \makecell{Circular symmetry} \\
\noalign{\hrule height 1.5pt}
\end{tabular}
\end{table}

\section{Prior work on metrics over the pose space}
\label{sec:related_work}

We propose in this section a brief review of the recent work on metrics over the pose space of a rigid object. We consider only mathematical distances in our discussion -- \ie symmetric, positive-definite applications from $\mathcal{C} \times \mathcal{C}$ to $\mathbb{R}^{+}$ satisfying triangle inequality.
The existing literature does not take into account potential proper symmetries of the object, and therefore, in this review, the pose space can be identified to the group of rigid transformations $SE(3)$.

\subsection{Objectiveness}
\label{sec:objectiveness}
The identification of the pose space to $SE(3)$ is based on the choice of two arbitrary frames: a frame linked to the object -- to which we will refer to as \emph{object frame} -- and a fixed \emph{inertial frame} such as the object frame coincides with the inertial frame when the object is in the reference pose $\mathcal{P}_0$. For a distance to be well-defined, it should not depend on an arbitrary choice of those frames, a notion that \citet{lin2000} formalize as \emph{objectiveness} or \emph{frame invariance}.

Among possible distances, geodesic distances have focused most interest and have been studied within the framework of Riemannian geometry on the Lie group $SE(3)$. Geo\-desic distances are well-suited for applications dealing with motions as they represent the minimum length of a motion to bring the object from one pose to an other. \citet{park1995} showed that there are no bi-invariant Riemannian metrics on $SE(3)$ -- that is, invariant to any change of inertial frame (left invariance) and of object frame (right invariance). \citet{chirikjian2015} recently studied this question further and showed that while continuous bi-invariant metrics do not exist, there are continuous left-invariant distances that are invariant under right shifts by pure rotations.

\subsection{Hyper-rotation approximation}

Nonetheless, several authors have worked on an ``approximate bi-invariant'' metric \citep{purwar2009} for $SE(3)$ through the mapping of rigid transformations to hyper-rota\-tions of $SO(4)$, and the use of a bi-invariant metric on $SO(4)$. Techniques to perform such mapping have been proposed based on biquaternion representation \citep{etzel1996} and polar decomposition \citep{larochelle2007}. Such transformation unfortunately requires a scaling for the translation part, which has to be set empirically depending on the application \citep{angeles2006}.

\subsection{Decomposition into translation and rotation}

Hopefully, while  inertial frame invariance is necessary for the \emph{objectiveness} of a metric, object frame invariance is not. \citet{lin2000} indeed showed that a distance is objective if and only if it is independent of the choice of inertial frame, and transforms by a right shift in response to a change of object frame. Therefore, a method to define an objective metric consists in defining a left invariant distance considering a given object frame and always using this one, in order to avoid having to transform the distance expression.

For this technique, a frequent approach consists in splitting a pose into a position and an orientation part and to define a distance on $SE(3)$ based on frame invariant metrics on both $\mathbb{R}^3$ and the rotation group $SO(3)$. Those metrics can then be fused in the form of a weighted generalized mean, here written with two strictly positive scaling factors $a$ and $b$ and for an exponent $p \in [1, \infty]$:
\begin{equation}
\label{eq:distance_rotation_translation}
\mydist(\mat{T}_1, \mat{T}_2) = \sqrt[\leftroot{0}\uproot{3}p]{ a \mydistrot(\mat{R}_1, \mat{R}_2)^p + b \mydisttrans(\mat{t}_1, \mat{t}_2)^p }.
\end{equation}

The Euclidean distance is the usual choice for measuring distances between different positions. Considering the usual Riemannian distance over $SO(3)$, a Riemannian distance over $SE(3)$ can be obtained by combining those together into \citep{park1995}:
\begin{equation}
\label{eq:distance_classical_riemannian}
d(\mat{T}_1, \mat{T}_2) = \sqrt{ a \|\log(\mat{R}_1^{-1} \mat{R}_2)\|^2 + b \| \mat{t}_2 - \mat{t}_1 \|^2 }.
\end{equation}
This expression is particularly interesting in that the distance between orientations $\|\log(\mat{R}_1^{-1} \mat{R}_2)\|$ corresponds to the angle $\alpha$ of the relative rotation between the two, which can be evaluated quite easily \eg from the following relations, using matrix or unit quaternion representations and respectively trace or inner product operators:
\begin{equation}
\Tr(\mat{R}_1^{-1} \mat{R}_2) = 2 \cos(\alpha) + 1
\end{equation}
\begin{equation}
1 - \left\langle \mat{q_1} | \mat{q_2} \right\rangle^2 = \cfrac{1}{2} (1 - \cos(\alpha)).
\end{equation}

Without the Riemannian constraint, a large number of inertial-frame-invariant distances can be considered. \citet{gupta1997} notably proposed to consider a Froebenius distance for the rotation part, which also only depends on the angle of the relative rotation between the two poses:
\begin{equation}
\label{eq:distance_froebenius_norm}
\|\mat{R}_2 - \mat{R}_1\|_F = 2 \sqrt{2} |\sin(\alpha/2)|.
\end{equation}
Similar properties are obtained considering the Euclidean distance between representations of antipodal pairs of unit quaternions $\mat{q_1}$ and $\mat{q_2}$:
\begin{equation}
\min \| \mat{q_2} \pm \mat{q_1} \| = 2 |\sin(\alpha/4)|.
\end{equation}

Merging position and orientation distances together requires setting the scaling factors $a$ and $b$. The choice of those factors remains a heuristic issue, and the recent work of~\citet{kendall2015} on camera pose regression using a deep neural network notably showed that this setting may have a great impact on performances.
A reasonable choice in the case of object poses consists in setting the position weight $b$ to  $1$ and the orientation weight $a$ as the square of the maximum radius of the object \citep{digregorio2008} in equation \eqref{eq:distance_classical_riemannian}, assuming an object frame at the center of the object in order to get an upper bound of the displacement of the object's points between two poses.

\subsection{Geometric approaches}
To avoid the need for arbitrary scaling factors, some distances are based only on geometric properties of the object. A particularly interesting possibility is to define a metric based on the distance between corresponding 3D points of instances of the object at these poses, such as depicted figure \ref{fig:corresponding_points}. Let $\mu$ be a density distribution relative to the object and $V=\int \mu(\mat{x}) dv$ its integral over the whole object, we can formulate such a distance the following way considering an $L^p$ norm:
\begin{equation}
\mydist(\mat{T}_1, \mat{T}_2) = \cfrac{1}{V} \left( \int \mu(\mat{x}) \|\mat{T}_2 (\mat{x}) - \mat{T}_1 (\mat{x}) \|^p dv \right)^{\cfrac{1}{p}}
\end{equation}

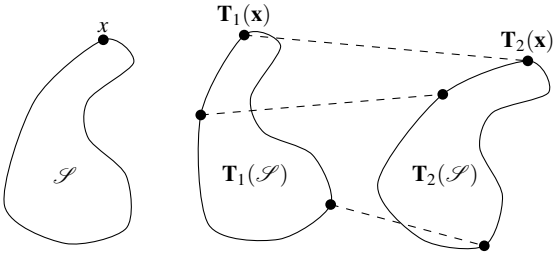
\begin{figure}
\centering	
\begin{tikzpicture}
	\coordinate (o) at (0, 0);		
	\coordinate (a) at (-0.8, -0.4);
	\coordinate (b) at (0, -0.9);
	\coordinate (c) at (0.8, -0.7);
	\coordinate (d) at (0.8, 0);
	\coordinate (e) at (0.3, 0.5);
	\coordinate (f) at (0.3, 1);
	\coordinate (g) at (0.9, 1.5);
	\coordinate (h) at (0.5, 1.8);
	\coordinate (i) at (-0.4, 1);
	\coordinate (j) at (-0.7, 0);
	
	\coordinate (o1) at ([shift={(2.5,0)},rotate=20] o);
	\coordinate (a1) at ([shift={(2.5,0)},rotate=20] a);
	\coordinate (b1) at ([shift={(2.5,0)},rotate=20] b);
	\coordinate (c1) at ([shift={(2.5,0)},rotate=20] c);
	\coordinate (d1) at ([shift={(2.5,0)},rotate=20] d);
	\coordinate (e1) at ([shift={(2.5,0)},rotate=20] e);
	\coordinate (f1) at ([shift={(2.5,0)},rotate=20] f);
	\coordinate (g1) at ([shift={(2.5,0)},rotate=20] g);
	\coordinate (h1) at ([shift={(2.5,0)},rotate=20] h);
	\coordinate (i1) at ([shift={(2.5,0)},rotate=20] i);
	\coordinate (j1) at ([shift={(2.5,0)},rotate=20] j);
	
	\coordinate (o2) at ([shift={(5,0)},rotate=-20] o);
	\coordinate (a2) at ([shift={(5,0)},rotate=-20] a);
	\coordinate (b2) at ([shift={(5,0)},rotate=-20] b);
	\coordinate (c2) at ([shift={(5,0)},rotate=-20] c);
	\coordinate (d2) at ([shift={(5,0)},rotate=-20] d);
	\coordinate (e2) at ([shift={(5,0)},rotate=-20] e);
	\coordinate (f2) at ([shift={(5,0)},rotate=-20] f);
	\coordinate (g2) at ([shift={(5,0)},rotate=-20] g);
	\coordinate (h2) at ([shift={(5,0)},rotate=-20] h);
	\coordinate (i2) at ([shift={(5,0)},rotate=-20] i);
	\coordinate (j2) at ([shift={(5,0)},rotate=-20] j);
	
	\draw plot[smooth cycle, tension=0.5] coordinates {(a) (b) (c) (d) (e) (f) (g) (h) (i)};
    \fill (h) circle[radius=2pt] node[above] {$x$};
    
    \draw plot[smooth cycle, tension=0.5] coordinates {(a1) (b1) (c1) (d1) (e1) (f1) (g1) (h1) (i1)};
    \fill (h1) circle[radius=2pt] node[above] {$\mat{T}_1(\mat{x})$};
    
    \draw plot[smooth cycle, tension=0.5] coordinates {(a2) (b2) (c2) (d2) (e2) (f2) (g2) (h2) (i2)};
    \fill (h2) circle[radius=2pt] node[above] {$\mat{T}_2(\mat{x})$};
	
    \node at (o) {$\mathcal{S}$};
    \node at (o1) {$\mat{T}_1(\mathcal{S})$};
    \node at (o2) {$\mat{T}_2(\mathcal{S})$};
	
    \fill (c1) circle[radius=2pt];
    \fill (c2) circle[radius=2pt];
    \draw[dashed] (c1) -- (c2);

    \fill (i1) circle[radius=2pt];
    \fill (i2) circle[radius=2pt];
    \draw[dashed] (i1) -- (i2);
    
	\draw[dashed] (h1)  -- (h2);
\end{tikzpicture}

\caption{\label{fig:corresponding_points}Representation of corresponding points between instances at different poses of a rigid object without proper symmetries.}
\end{figure}
This expression has a strong physical meaning. It is by construction frame-invariant since its definition does not depend on a particular frame, and takes the shape of the object into account, without the need of arbitrary tuning. \citet{martinez1995} suggest the use of the maximum displacement ($p=\infty$), and \citet{hinterstoisser2012} used the average displacement ($p=1$) for a pose estimation evaluation. For the sake of tractability, those authors suggest to limit consideration to only some vertices of an object model since the integral has to be evaluated explicitly. \citet{kazerounian1992} on the other hand proposed the use of integral of square displacements ($p=2$) on the whole object, and showed that it could be evaluated efficiently given the inertia matrix of the object. \citet{chirikjian1998} later on improved this formulation and extended it for arbitrary affine transformations, showing that the distance could be evaluated as a weighted Froebenius norm.

\citet{zefran1996-a} and \citet{lin2000} independently proposed a Riemannian tensor being linked to the notion of kinetic energy and therefore taking into account the object properties without the need for some arbitrary tuning. Their tensor can be seen as a local equivalent of the distance of \citet{kazerounian1992}. However, to the best of our knowledge there is no known closed-form expression for the resulting geodesic distance in the general case.

\subsection{Local metric}
Various others local parametrization methods exist which, by mapping locally the pose space to a Euclidean space enable to locally define distances. Those parametrizations are \eg based on the representation of orientation with Euler angles, or the local stereographic projection of the pose space, identified to Study's quadric -- an hypersurface embedded in $\mathbb{R}^7$ \citep{eberharter2004}. In-depth discussion of this topic is out of the scope of this article as we are interested in global distances.

\bibliographyInSubfile

\section{Proposed metric}
\label{sec:proposed_metric}

In this section, we propose a distance over the pose space of a 3D rigid bounded object, valid even for symmetric ones. This distance can be considered as an extension of the work of \citet{kazerounian1992} and \citet{chirikjian1998} to arbitrary bounded objects. We also discuss some of its properties.
 
\subsection{Formal definition} 
 
Let $\mathcal{S}$ be the set of points of the object at reference pose $\mathcal{P}_0 \in \mathcal{C}$, and $\mu$ a positive density distribution defined on $\mathcal{S}$. In order to be meaningful, the set of points of the object and its density distribution are assumed to be compatible with the proper symmetry properties of the object and exhibit those symmetry properties. Formally, we assume they verify $\mat{G} (\mathcal{S}) = \mathcal{S}$ and $\mu \circ \mat{G} = \mu$ for any proper symmetry $\mat{G} \in G$.

\myframedminipage{%
\begin{definition}
\label{def:distance}
Let $\mathcal{P}_1, \mathcal{P}_2 \in \mathcal{C}$ be two poses and $\mat{T}_1, \mat{T}_2 \in SE(3)$ two rigid transformations whose equivalence classes are respectively identified to  $\mathcal{P}_1$ and  $\mathcal{P}_2$ given the reference pose, such as defined in \eqref{eq:pose_as_equivalence_class}. We define the distance between $\mathcal{P}_1$ and $\mathcal{P}_2$ as follows~: 
\begin{equation}
\label{eq:distance}
\begin{aligned}
&\mydist(\mathcal{P}_1, \mathcal{P}_2) \triangleq \min_{\mat{G}_1, \mat{G}_2 \in G } \mydistnosym(\mat{T}_1 \circ \mat{G}_1, \mat{T}_2 \circ \mat{G}_2)\text{,} \\
&\text{with }\\
&\mydistnosym(\mat{T}_1,\mat{T}_2) \triangleq  \sqrt{ \cfrac{1}{S} \int_{\mathcal{S}} \mu(\mat{x}) \| \mat{T}_2 (\mat{x}) - \mat{T}_1 (\mat{x})\|^2 ds }\\
&\text{where } S \triangleq \int_\mathcal{S} \mu(\mat{x}) ds.
\end{aligned}
\end{equation}	 
\end{definition}}
This expression is well defined. The minimum in definition \eqref{eq:distance} is reached because of the compactness of the proper symmetry group $G$ -- as a closed subgroup of $SO(3)$ which itself is compact -- and of the continuity of $\mydistnosym$.
Moreover, this definition is by construction independent of the choice of the rigid transformations $\mat{T}_1, \mat{T}_2$ identified to the poses considered.
We verify easily that it satisfies the conditions of a distance definition: $\mydist$ is symmetric, positive-definite, and triangle inequality derives from the triangle inequality satisfied by $\mydistnosym$, which is a direct consequence of Minkowski inequality.
An equivalent formulation of this distance, involving a single minimization over $G$, is introduced in proposition~\ref{prop:distance_as_relative_displacement}.

In typical applications, one is particularly interested in the positioning of the surface of the object. Therefore in our experiments, we consider the surface of the object as set of points  $\mathcal{S}$. The density function $\mu$ can be used to modulate the importance of the positioning of specific areas, but without additional information it is natural to consider an uniform weight $\mu=1$.
	
\subsection{Objectiveness}
\label{subsec:objectiveness}
The proposed distance is independent by construction of the choice of some arbitrary frames as it admits a purely geometric interpretation, a point we discuss in subsection~\ref{sec:geometric_interpretation}.

Definition~\ref{def:distance} makes no assumption on the choice of object frame, and the use of a reference pose  -- \ie an inertial frame -- in our formulation is only here for the sake of writability. Indeed, the Euclidean distance between 3D points is invariant to isometries by definition of these, and in particular to any rigid transformation $\mat{T}_3^{-1} \in SE(3)$:
\begin{equation}
\forall \mat{x}, \mat{y} \in \mathbb{R}^3, \| \mat{x} - \mat{y}\| = \| \mat{T}_3^{-1} (\mat{x}) - \mat{T}_3^{-1} (\mat{y})\|
\end{equation}
Therefore, an arbitrary new reference pose $\mathcal{P}_3$ could be considered without any effect on the metric properties. Denoting $\mat{T}_3$ a rigid transformation identified to $\mathcal{P}_3$ relatively to the old reference pose $\mathcal{P}_0$, we verify the independence of $\mydistnosym$ from the choice of reference pose
\begin{equation}
\mydistnosym(\mat{T}_1, \mat{T}_2) = \mydistnosym(\mat{T}_3^{-1} \mat{T}_1, \mat{T}_3^{-1} \mat{T}_2),
\end{equation}
and hence the independence of the general distance:
\begin{equation}
\mydist([\mat{T}_1], [\mat{T}_2]) = \mydist([\mat{T}_3^{-1} \mat{T}_1], [\mat{T}_3^{-1} \mat{T}_2]).
\end{equation}


\subsection{Geometric interpretation}
\label{sec:geometric_interpretation} 
A picture being worth a thousand words, the reasoning we develop in this section is illustrated in figure \ref{fig:distance_illustration} for the case of a 2D object with a rotation symmetry of $2 \pi / 3$: a flower with three petals.
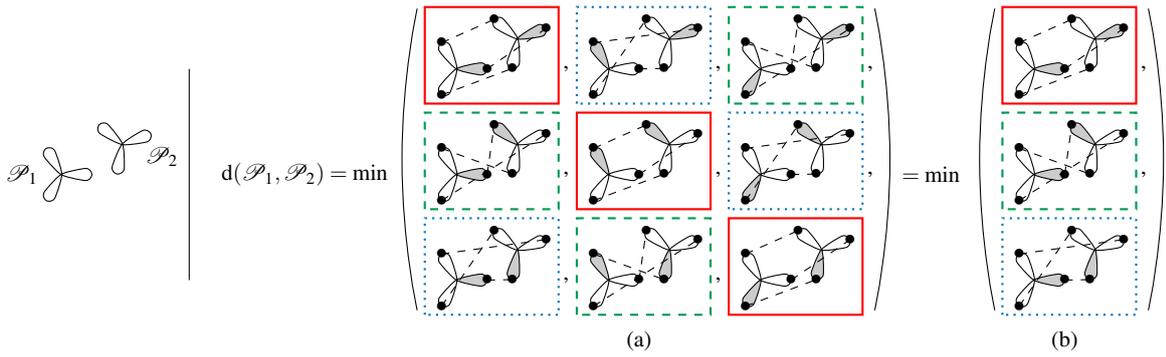
\begin{figure*}
\centering

\newcommand{\DrawFlower}[4]{
	\draw [shift={({#1}, {#2})}, rotate = {#3}, #4] plot[smooth cycle, tension=0.5] coordinates {(0, 0) (0.8, -0.2) (1, 0) (0.8, 0.2)};
	\draw [shift={({#1}, {#2})}, rotate = {#3 + 120}] plot[smooth cycle, tension=0.5] coordinates {(0, 0) (0.8, -0.2) (1, 0) (0.8, 0.2)};
	\draw [shift={({#1}, {#2})}, rotate = {#3 - 120}] plot[smooth cycle, tension=0.5] coordinates {(0, 0) (0.8, -0.2) (1, 0) (0.8, 0.2)};
}

\newcommand{\setCoordinates}[4]{
	\coordinate (a#4) at ({#1 + cos(#3)}, {#2 + sin(#3)});
	\coordinate (b#4) at ({#1 + cos(#3 + 120)}, {#2 + sin(#3 + 120)});
	\coordinate (c#4) at ({#1 + cos(#3 - 120)}, {#2 + sin(#3 - 120)});
}

\newcommand{\DrawFlowerCouple}[7]{
	\DrawFlower{#1}{#2}{#3}{#7};
	\DrawFlower{#4}{#5}{#6}{#7};
}

\newcommand{\DrawFlowerCoupleWithCorrespondences}[6]{
	
	\DrawFlowerCouple{#1}{#2}{#3}{#4}{#5}{#6}{fill=gray!40}
	
	\setCoordinates{#1}{#2}{#3}{_1};
	\setCoordinates{#4}{#5}{#6}{_2};
	\draw [dashed] (a_1) node {$\bullet$} -- (a_2) node {$\bullet$};
	\draw [dashed] (b_1) node {$\bullet$} -- (b_2) node {$\bullet$};
	\draw [dashed] (c_1) node {$\bullet$} -- (c_2) node {$\bullet$};
}

\newcommand{\DrawFlowerCoupleSimpleWithCorrespondences}[4]{
	\DrawFlowerCoupleWithCorrespondences{#1}{#2}{0 + #3}{#1 + 2}{#2 + 1}{20 + #4}
}
	
\newcommand{\DrawFlowerCoupleSimple}[2]{
	\def\x2{#1 + 2}
	\def\y2{#2 + 1}
	\DrawFlowerCouple{#1}{#2}{0}{\x2}{\y2}{20}{};
	\draw ({#1 - 0.5}, {#2}) node [left] {$\mathcal{P}_1$};
	\draw ({\x2 + 0.5}, {\y2 - 0.4}) node [right] {$\mathcal{P}_2$};	
}

\newcommand{\DrawRectangle}[3]{
	\draw [#3] ({#1 + -1.05}, {#2 + -1.15}) rectangle ({#1 + 3.35}, {#2 + 2.05});
}

\newcommand{\DrawRectangleTypeOne}[2]{
	\DrawRectangle{#1}{#2}{thick, solid, red};
}

\newcommand{\DrawRectangleTypeTwo}[2]{
	\DrawRectangle{#1}{#2}{thick, dashed, ForestGreen};
}

\newcommand{\DrawRectangleTypeThree}[2]{
	\DrawRectangle{#1}{#2}{thick, dotted, RoyalBlue};
}

\newcommand{\DrawParenthesisStart}[1]{
\draw plot[smooth, tension=1, thick] coordinates {(#1, \YParenthesisStart) ({#1 - 0.5}, {(\YParenthesisStart + \YParenthesisEnd)/2}) (#1,\YParenthesisEnd)};
}

\newcommand{\DrawParenthesisEnd}[1]{
\draw plot[smooth, tension=1, thick] coordinates {(#1, \YParenthesisStart) ({#1 + 0.5}, {(\YParenthesisStart + \YParenthesisEnd)/2}) (#1,\YParenthesisEnd)};
}

\begin{tikzpicture}[scale = 0.4]

\def\xoffset{5}
\def\yoffset{-3.5}
\def\YParenthesisStart{2}
\def\YParenthesisEnd{-8}

\DrawFlowerCoupleSimple{-13}{\yoffset};

\draw (-8.8, 0) -- (-8.8, {\yoffset * 2});

\draw (-2, \yoffset) node [left] {$\text{d}(\mathcal{P}_1, \mathcal{P}_2) = \min$};

\DrawParenthesisStart{-1.3}

\foreach \i in {0,...,2}
{
	\foreach \j in {0, ..., 2}
	{
	\DrawFlowerCoupleSimpleWithCorrespondences{\xoffset * \i}{\yoffset * \j}{120 * \i}{120 * \j};
	}
}

\def\xvirguleoffset{3.6}
\foreach \i in {0,...,2}
{
	\foreach \j in {0, ..., 1}
	{
	\draw ({\xoffset * \i + \xvirguleoffset}, {\yoffset * \j}) node {,};
	}
}
\foreach \i in {0, ..., 1}
{
\draw ({\xoffset * \i + \xvirguleoffset}, {\yoffset * 2}) node {,};
}

\foreach \i in {0,...,2}
{
\DrawRectangleTypeOne{\i * \xoffset}{\i  * \yoffset}
}

\DrawRectangleTypeThree{2 * \xoffset}{1 * \yoffset}
\DrawRectangleTypeThree{1 * \xoffset}{0 * \yoffset}
\DrawRectangleTypeThree{0 * \xoffset}{2 * \yoffset}

\DrawRectangleTypeTwo{1 * \xoffset}{2 * \yoffset}
\DrawRectangleTypeTwo{0 * \xoffset}{1 * \yoffset}
\DrawRectangleTypeTwo{2 * \xoffset}{0 * \yoffset}

\DrawParenthesisEnd{2 * \xoffset + 3.75}

\draw (\xoffset + 1, -9) node {(a)};

\draw (15.6, \yoffset) node {$=\min$};

\DrawParenthesisStart{19 - 1.3}

\foreach \j in {0, ..., 2}
{
	\DrawFlowerCoupleSimpleWithCorrespondences{19}{\yoffset * \j}{0}{120 * \j};
}
\DrawRectangleTypeOne{19}{0 * \yoffset}
\DrawRectangleTypeTwo{19}{1 * \yoffset}
\DrawRectangleTypeThree{19}{2 * \yoffset}

\foreach \j in {0, ..., 1}
{
\draw ({19 + \xvirguleoffset}, {\yoffset * \j}) node {,};
}
\DrawParenthesisEnd{19 + 3.75}
\draw (19 + 1, -9) node {(b)};

\end{tikzpicture}

\caption{\label{fig:distance_illustration} Illustration of our proposed distance for a 2D object with a rotation symmetry of $2\pi/3$. (a) The distance between two poses consists in the minimum distance between two poses of an equivalent object without proper symmetry -- here there are 3 possible poses of the equivalent object for each pose of the original object. The distance between poses of an object without proper symmetry corresponds to the RMS distance between corresponding object points (dashed segments). (b) Equivalently, the proposed distance can be considered as a measure of the smallest displacement from one pose to an other -- here there are actually only 3 different displacements between those two poses (solid, dotted and dashed boxes).}
\end{figure*}

As we discussed in section \ref{sec:pose_definition}, a pose $\mathcal{P}_i \in  \mathcal{C}$ can be identified to a set of rigid transformations $\left\lbrace \mat{T}_i \circ \mat{G}, \mat{G} \in G \right\rbrace$. Each of these transformations can themselves be identified with the pose of an object with identical characteristics to the considered one but with no proper symmetries -- to which we will refer to as the equivalent object and which we depict on figure~\ref{fig:distance_illustration} with a grey petal (in order to break the symmetry of the initial object). A pose of the object can therefore be considered as a set of poses of the equivalent object --  3 in our example. 
Points of the equivalent object can be unambiguously put in correspondence between different poses -- correspondences we represent by dashed segments on the figure. Therefore it is legitimate to define a distance between poses of the equivalent object based on the distance between such corresponding points. In this paper, we consider the RMS distance $\mydistnosym$ as it enables efficient computations (see section~\ref{sec:pose_representation}):
\begin{equation}
\mydistnosymsquare(\mat{T}_1,\mat{T}_2) = \cfrac{1}{S} \int_{\mathcal{S}} \mu(\mat{x}) \| \mat{T}_2 (\mat{x}) - \mat{T}_1 (\mat{x})\|^2 ds.
\end{equation}

The proposed distance between two poses of the object can then be defined as the minimum distance between each potential pair of poses for the equivalent object ($3 \times 3$ combinations in our example):
\begin{equation}
\mydist(\mathcal{P}_1, \mathcal{P}_2) = \min_{\mat{G}_1, \mat{G}_2 \in G } \mydistnosym(\mat{T}_1 \circ \mat{G}_1, \mat{T}_2 \circ \mat{G}_2).
\end{equation}

An other and more intuitive interpretation is to consider our distance as a measure of the smallest displacement from one pose to the other. A displacement from a pose $\mathcal{P}_1$ to an other $\mathcal{P}_2$ is a relative transformation from a pose of the equivalent object corresponding to $\mathcal{P}_1$ to a pose of the equi\-valent object corresponding to $\mathcal{P}_2$, and the length of a displacement is measured via $\mydistnosym$. Different pairs of poses of the equivalent object are actually linked by the same displacement -- as can be observed on figure~\ref{fig:distance_illustration} where pairs of poses of the equivalent object being linked by the same transformation are highlighted by identical boxes. All displacements from one pose $\mathcal{P}_1$ to an other $\mathcal{P}_2$ are in fact considered when choosing an arbitrary pose of the equivalent object $\mat{T}_1$ for $\mathcal{P}_1$ and considering rigid transformations from $\mat{T}_1$ to the poses of the equivalent object corresponding to $\mathcal{P}_2$. Thanks to this, the distance between two poses can actually be computed considering the proper symmetries only for one pose:
\myframedminipage{%
\begin{proposition}
\label{prop:distance_as_relative_displacement}
For any poses $\mathcal{P}_1, \mathcal{P}_2 \in \mathcal{C}$ and $\mat{T}_1, \mat{T}_2 \in SE(3)$ two rigid transformations whose equivalency classes are respectively identified to  $\mathcal{P}_1$ and  $\mathcal{P}_2$ given the reference pose,
\begin{equation}
\label{eq:distance_as_relative_displacement}
\mydist(\mathcal{P}_1, \mathcal{P}_2) = \min_{\mat{G} \in G} \mydistnosym(\mat{T}_1, \mat{T}_2 \circ \mat{G})
\end{equation}
\end{proposition}}
This formulation is simpler than the definition~\ref{def:distance}, however it breaks the symmetry of the roles of the two poses.

\begin{proof}
Formally, expression \eqref{eq:distance_as_relative_displacement} can be deduced from the distance definition \eqref{eq:distance} as follows. Given two proper symmetries $\mat{G}_1, \mat{G}_2 \in G$, one can perform the change of variables $x \leftarrow \mat{G}_1 (x)$ and $\mat{G} \leftarrow \mat{G}_2 \circ \mat{G}_1^{-1}$ to write the following equality:
\begin{equation}
\begin{aligned}
&\mydistnosymsquare(\mat{T}_1 \circ \mat{G}_1, \mat{T}_2 \circ \mat{G}_2) \\
&=\cfrac{1}{S} \int_{\mathcal{S}} \mu(\mat{x}) \| \mat{T}_2 \circ \mat{G}_2 (\mat{x}) - \mat{T}_1  \circ \mat{G}_1(\mat{x})\|^2 ds \\
&=\cfrac{1}{S} \int_{\mat{G}_1(\mathcal{S})} \mu(\mat{G}_1^{-1} (\mat{x})) \| \mat{T}_2 \circ \mat{G} (\mat{x}) - \mat{T}_1 (\mat{x})\|^2 ds \\
\end{aligned}
\end{equation}
The symmetry of the object pointset and of its density ensures that $\mat{G}_1(\mathcal{S}) = \mathcal{S}$ and $\mu \circ \mat{G}_1^{-1} = \mu$, leading to the following result from which the conclusion is straightforward: 
\begin{equation}
\label{eq:dist_no_sym_symmetry_invariance}
\begin{aligned}
&\mydistnosymsquare(\mat{T}_1 \circ \mat{G}_1, \mat{T}_2 \circ \mat{G}_2) \\
&=\cfrac{1}{S} \int_{\mathcal{S}} \mu(\mat{x}) \| \mat{T}_2 \circ \mat{G} (\mat{x}) - \mat{T}_1 (\mat{x})\|^2 ds \\
&= \mydistnosymsquare(\mat{T}_1, \mat{T}_2 \circ \mat{G}).
\end{aligned}
\end{equation}
\qed
\end{proof}


\subsection{Rotation anisotropy}
\label{subsec:rotation_anisotropy}
In case of a pure rotational displacement around the center of mass of a non symmetric object, usual metrics are solely dependent on the angle of the relative rotation between the two poses. $\mydistnosym$ on the other hand -- and the proposed distance as well -- accounts for the object's geometry and as such also depends on the considered axis.
More precisely, the distance between two poses linked by such a displacement depends on the angle $\theta$ and on the inertia moment $I_{\mat{k}}$ along the axis $\mat{k}$ of the relative rotation between the two poses, as follows:
\begin{equation}
\begin{aligned}
\label{eq:distance_and_inertia_moment}
& \mydistnosym(\mat{T}_1, \mat{T_2}) = 2 \sqrt{I_{\mat{k}}}  \sin \left( \cfrac{\theta}{2} \right) \\
&\text{where } I_{\mat{k}} = \cfrac{1}{S} \int \mu(\mat{x}) \| \mat{k} \times \mat{x} \|^2 ds.
\end{aligned}
\end{equation}
This result can be easily obtained by injecting Rodrigues' rotation formula in the expression of the proposed distance.
We illustrate this property on figure~\ref{fig:example_rotation_anisotropy} with an object consisting in a model of the Eiffel tower scaled to its actual dimensions, for two couples of poses linked by a smallest displacement consisting of a rotation of 15\textdegree{} around different axes. While the angle of the relative rotation is identical in both cases, displacements of surface points are quite different and we visually tend to consider poses in case~(b) as being farther away one another than in case~(a). Our framework formalizes this intuition, resulting in a distance between the poses in configuration~(b) approximatively $2.1$ times greater than the one between the poses in configuration~(a).

\begin{figure}
\centering
\begin{tikzpicture}
\node[inner sep=0pt] (eiffel) at (0,0)
    {\includegraphics[height=6cm]{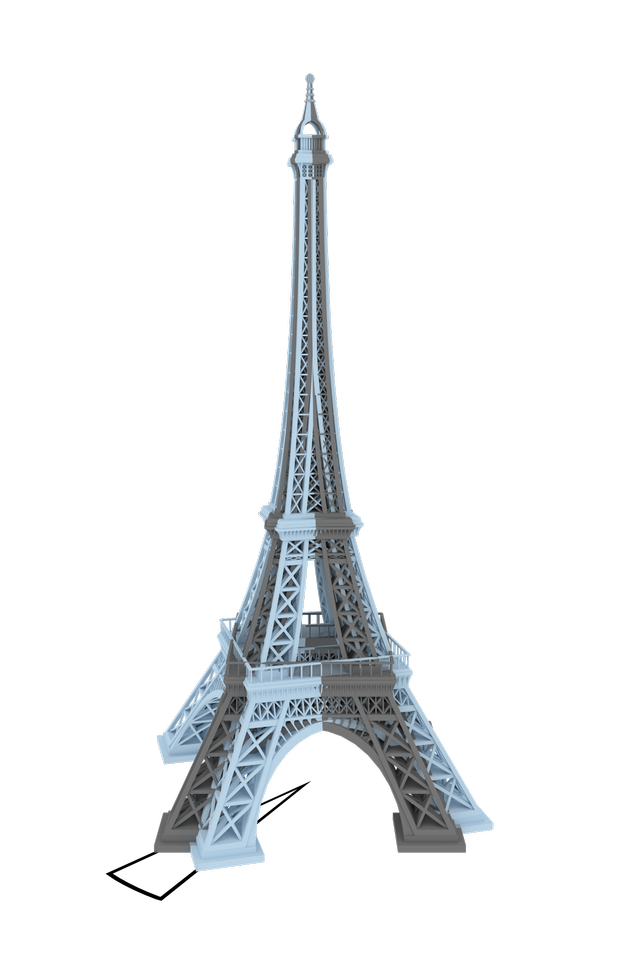}};
\node at (-1.5cm, -2.7cm) {15\textdegree{}};
\node at (0cm, -2.8) {\textbf{(a)} d=9.2m};
\node[inner sep=0pt] (eiffel2) at (4cm,0)
    {\includegraphics[height=6cm]{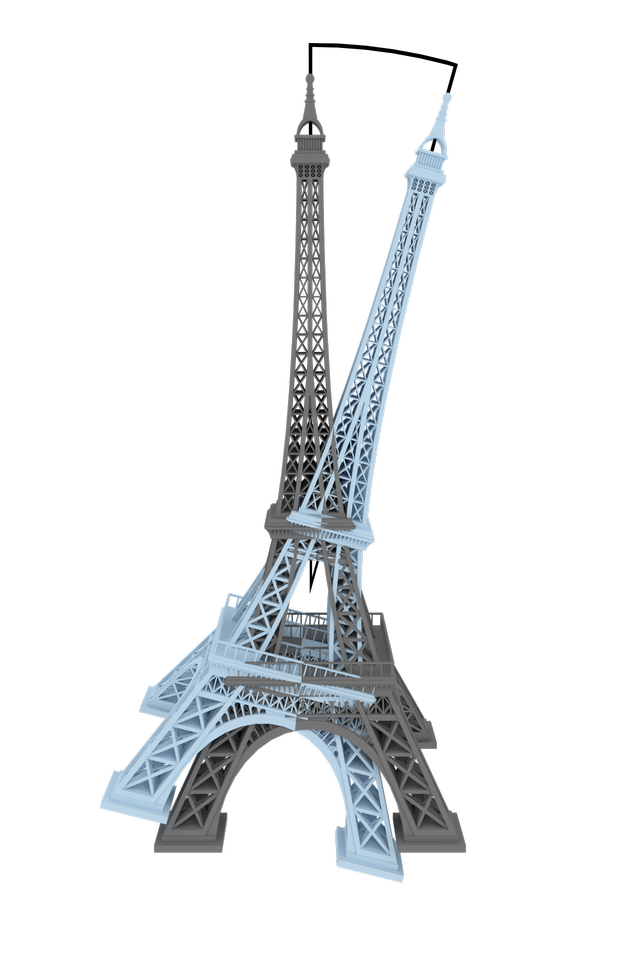}};
\node at (4.5cm, 2.9cm) {15\textdegree{}};
\node at (4cm, -2.8cm) {\textbf{(b)} d=18.9m};
\end{tikzpicture}
\caption{Usual metrics would consider the distances between the two poses in cases (a) and in (b) equal -- as in both cases the two poses are linked by a rotation of 15\textdegree{} around the center of mass of the object. Our distance accounts for the object geometry and discriminates these two configurations.\label{fig:example_rotation_anisotropy}}
\end{figure}
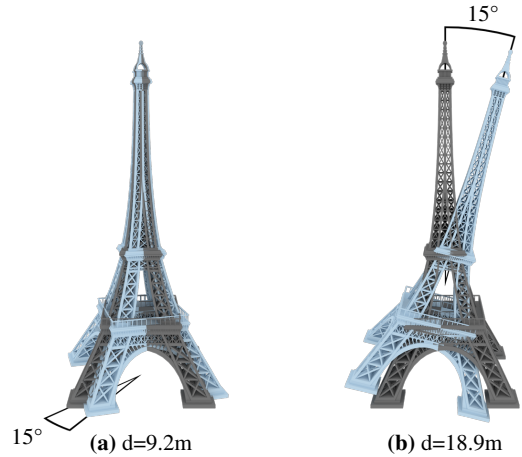

\bibliographyInSubfile

\section{Efficient distance computation}
\label{sec:pose_representation}

The distance definition~\ref{def:distance} and the simpler expression of pro\-position~\ref{prop:distance_as_relative_displacement} are of little practical use for actual applications as they contain a summation term over the set of points of the object and a minimization over its proper symmetry group, both sets being potentially infinite.
 
In this section, we show how our proposed distance can be evaluated efficiently.
To this aim, we propose a representation of a pose $\mathcal{P}$ as a finite set of points $\mathcal{R}(\mathcal{P})$  of a Euclidean space $\mathbb{R}^N$ of at most 12 dimensions, depending on the object's symmetries.  We refer to an element of $\mathcal{R}(\mathcal{P})$ as a \emph{representative} of $\mathcal{P}$, since a representative completely defines a pose (see section \ref{sec:projection}). 

Within this representation framework, the distance between a pair of poses $\mathcal{P}_1, \mathcal{P}_2$ can be expressed as the minimum of Euclidean distance between their respective representatives,
\begin{equation}
\boxed{
\mydist(\mathcal{P}_1, \mathcal{P}_2) = \min_{\mat{p}_1 \in \mathcal{R}(\mathcal{P}_1), \mat{p}_2  \in \mathcal{R}(\mathcal{P}_2)} \|\mat{p}_2 - \mat{p}_1 \|
}
\end{equation}
or equivalently, as the minimum Euclidean distance between a given representative for one pose and the representatives of the other, thanks to a reasoning similar to the one developed in the proof of proposition~\ref{prop:distance_as_relative_displacement}:
\begin{equation}
\label{distance_with_a_given_representative}
\boxed{
\forall \mat{p}_1 \in \mathcal{R}(\mathcal{P}_1), \mydist{}(\mathcal{P}_1, \mathcal{P}_2) = \min_{\mat{p}_2  \in \mathcal{R}(\mathcal{P}_2)} \|\mat{p}_2 - \mat{p}_1 \|.
}
\end{equation}

The cardinal of $\mathcal{R}(\mathcal{P})$ is independent of the pose considered, depending solely from the class of proper symmetries of the considered object. We therefore denote it $|\mathcal{R}(\bullet)|$.
For most classes -- objects with no proper symmetry, revolution objects without rotoreflection invariance and spherical objects -- a pose admits a single representative and in that case we will refer to it as $\mathcal{R}(\mathcal{P})$ by an abuse of notation. $\mathcal{R}$ can be  considered in such case as an isometric embedding of $\mathcal{C}$ into the Euclidean space $\mathbb{R}^N$, and the distance between two poses simply corresponds to the Euclidean distance between their respective representatives.

The expression of pose representatives will be derived later in this section for the different classes of objects, and a synthesis is proposed in table~\ref{tab:representations_synthesis} for 3D objects, and in table \ref{tab:representations_synthesis_2D} for 2D ones.

\begin{table*}
\small\sf\centering
\caption{Proposed representatives for a pose $\mathcal{P} = [(\mat{R} \in SO(3), \mat{t} \in \mathbb{R}^3)]$ of a 3D object depending on its proper symmetries.\label{tab:representations_synthesis}}
\rule{\linewidth}{1.5pt}
\textbf{Assumptions} \\
Center of mass of the object as origin of the object frame. For revolution objects, revolution axis as $\mat{e}_z$ axis of the object frame.\\
\textbf{Notations}\\
$\matLambda \triangleq \left( \cfrac{1}{S} \int_\mathcal{S} \mu(\mat{x}) \mat{x} \mat{x}^\top ds \right)^{1/2}$ \\
and $\lambda \triangleq \sqrt{\lambda_r^2 + \lambda_z^2}$ for revolution objects where $\matLambda = \text{diag}(\lambda_r, \lambda_r, \lambda_z)$.
\renewcommand{\arraystretch}{1.2}
\rule{\linewidth}{0.5pt}
\begin{tabular}{ccc}
\textbf{Proper symmetry class} & \textbf{Proper symmetry group $G$} & \textbf{Pose representatives $\mathcal{R}(\mathcal{P})$} \\
Spherical symmetry & $SO(3)$ & $\mat{t} \in \mathbb{R}^3$ \\
Revolution symmetry without rotoreflection invariance & $\left\lbrace \mat{R}_z^\alpha \middle\vert \alpha \in \mathbb{R} \right\rbrace$ & $ (\lambda (\mat{R} \mat{e}_z)^\top, \mat{t}^\top)^\top \in \mathbb{R}^{6}$ \\
Revolution symmetry with rotoreflection invariance & $\left\lbrace \mat{R}_x^\delta \mat{R}_z^\alpha \middle\vert \delta \in \left\lbrace 0, \pi \right\rbrace, \alpha \in \mathbb{R} \right\rbrace$ & $\left\lbrace (\pm \lambda (\mat{R} \mat{e}_z)^\top, \mat{t}^\top)^\top\right\rbrace \subset \mathbb{R}^{6}$ \\
No proper symmetry & $\left\lbrace \mat{I} \right\rbrace$ & $  (\vect(\mat{R} \matLambda)^\top, \mat{t}^\top)^\top \in \mathbb{R}^{12}$  \\
Finite nontrivial & Finite & $\left\lbrace (\vect(\mat{R} \mat{G} \matLambda)^\top, \mat{t}^\top)^\top | \mat{G} \in G \right\rbrace \subset \mathbb{R}^{12}$\\
\end{tabular}\\
\rule{\linewidth}{1.5pt}
\end{table*}

\begin{table*}
\small\sf\centering
\caption{Proposed representatives for a pose $\mathcal{P} = [(\theta \in \mathbb{R}, \mat{t} \in \mathbb{R}^2)]$ of a 2D object depending on its proper symmetries. \label{tab:representations_synthesis_2D}}
\rule{\linewidth}{1.5pt}
\textbf{Assumptions} \\
Center of mass of the object as origin of the object frame.\\
\textbf{Notations}\\
$\forall \alpha \in \mathbb{R}, e^{i \alpha} = \left( \cos(\alpha), \sin(\alpha) \right)$,
 and $\lambda \triangleq \left( \cfrac{1}{S} \int_\mathcal{S} \mu(\mat{x}) \| \mat{x} \|^2 ds \right)^{1/2}$.
\renewcommand{\arraystretch}{1.2}
\rule{\linewidth}{0.5pt}
\begin{tabular}{ccc}
\textbf{Proper symmetry class} & \textbf{Proper symmetry group $G$} & \textbf{Pose representatives $\mathcal{R}(\mathcal{P})$} \\
Circular symmetry & $SO(2)$ & $\mat{t} \in \mathbb{R}^2$ \\
No proper symmetry & $\left\lbrace \mat{I} \right\rbrace$ & $(\lambda e^{i \theta}, \mat{t}^\top)^\top \in \mathbb{R}^{4}$  \\
Cyclic symmetry (order $n \in \mathbb{N}^*$)& $\left\lbrace \mat{R}^{2 k \pi/n} \middle\vert k \in \llbracket 0, n \llbracket \right\rbrace$ & $\left\lbrace (\lambda e^{i (\theta + 2 k \pi / n)}, \mat{t}^\top)^\top \middle\vert k \in \llbracket 0, n \llbracket \right\rbrace \subset \mathbb{R}^{4}$ \\
\end{tabular}
\rule{\linewidth}{1.5pt}
\end{table*}


\subsection{Neighborhood query}
\label{sec:proximity_query}

The distance formulation \eqref{distance_with_a_given_representative} is of great practical value as it enables to perform efficient \emph{radius search} and exact or approximate \emph{k-nearest neighbors} queries within a large set of poses through the use of any off-the-shelf neighborhood query algorithms designed for Euclidean spaces. Neighborhood queries are useful for numerous problems, and we provide an example in section \ref{sec:application_example} where radius search is heavily used. Existing methods for neighborhood queries enable fast neighborhood retrieval within a set of points of a vector space compared to a brute-force approach consisting in computing the distance to every points of the set. They make use of a specific search structure -- such as a grid or a kD-tree for example -- adapted to the specific properties of the considered metric space. A review of those algorithms is out of the scope of this work, and we will only refer the interested reader to the well-known FLANN library \citep{muja2009} as a starting point.

Let $S$ be a finite set of poses. We consider the pointset $R$ consisting of the aggregation of all representatives of the poses of $S$:
\begin{equation}
R = \bigcup\limits_{\mathcal{P} \in S} \mathcal{R}(\mathcal{P}).
\end{equation}

From \eqref{distance_with_a_given_representative}, given a query pose $\mathcal{Q}$ and one of its representatives $\mat{q} \in \mathcal{R}(\mathcal{Q})$, the poses of $S$ that are closer to $\mathcal{Q}$ than a given distance $\delta$ are the poses that have a representative closer to $\mat{q}$ than $\delta$ using the Euclidean distance.

Such representatives can be retrieved through a standard \emph{radius search} operation around $\mat{q}$ in $R$.
The search for nearest neighbors can be performed in a similar fashion.

One should nevertheless be careful of potential duplicates with those operations, as a pose may have several representatives depending on the proper symmetries of the object. Nonetheless, the absence of duplicates is locally guaranteed around the query point in an open ball of radius $T/2$, 
where $T$ is a constant defined as follows:
\myframedminipage{%
\begin{definition}[Min.\ distance between representatives]
\label{def:definition_T}
We note $T$ the minimum distance between different representatives of the same pose, with the convention $T=+\infty$ if a pose admits a single representative:
\begin{equation}
\forall \mathcal{P} \in \mathcal{C}, \forall \mat{p} \in \mathcal{R}(\mathcal{P}), T \triangleq \min_{\mat{q} \in \mathcal{R}(\mathcal{P}), \mat{q}\neq \mat{p}} \| \mat{q} - \mat{p} \|.
\end{equation}
\end{definition}}
$T$ can be computed considering an arbitrary pose $\mathcal{P}$ because of the invariance of our underlying metric to the choice of a reference pose (see section~\ref{subsec:objectiveness}), and considering an arbitrary representative $\mat{p} \in \mathcal{R}(\mathcal{P})$ because of the symmetry properties of representatives described  in section \ref{sec:representatives_symmetry}.


\subsection{Decomposition into translation and rotation parts}	
\label{sec:pose_decomposition}

From this point on, we consider a direct orthonormal coordinates system $(\mat{O}, \mat{e}_x, \mat{e}_y, \mat{e}_z)$. As in section~\ref{sec:group_of_proper_symmetry} on the proper symmetry group, we assume that the origin $\mat{O}$ of the object frame is an invariant point of the object with respect to its proper symmetries. It is for example chosen at the center of the object for a spherical object, and on the revolution axis for a revolution object. Doing so, the proper symmetry group of the object can be considered as a group of rotations around the origin, and we therefore assimilate proper symmetries to rotation matrices.
We exploit this property to develop the inner term of the expression \eqref{eq:distance} of the square distance into:
\begin{equation}
\label{eq:distance_inner_term_decomposition}
\begin{aligned}
&\|  (\mat{T}_2 \circ \mat{G}_2) (\mat{x}) - (\mat{T}_1 \circ \mat{G}_1) (\mat{x})\|^2 \\
&= \| \mat{R}_2 \mat{G}_2 \mat{x} + \mat{t}_2 - (\mat{R}_1 \mat{G}_1 \mat{x} + \mat{t}_1)\|^2 \\
& 
\begin{aligned}
= \| \mat{R}_2 \mat{G}_2 \mat{x} - \mat{R}_1 \mat{G}_1 \mat{x} \|^2 + \|\mat{t}_2 - \mat{t}_1\|^2 \\
+ 2 (\mat{t}_2 - \mat{t}_1)^\top (\mat{R}_2 \mat{G}_2 - \mat{R}_1 \mat{G}_1) \mat{x}.
\end{aligned}
\end{aligned}
\end{equation}	
	
We add the further constraint that the origin of the object frame is the center of mass of the object's surface, \ie $\int_\mathcal{S} \mu(\mat{x}) \mat{x} ds = \mat{0}$. This constraint is compatible with the previous one because the center of mass is unique, and therefore has to be left unchanged by the proper symmetries of the object.
Thanks to this choice, the last term of \eqref{eq:distance_inner_term_decomposition} disappears during the integration, and the squared distance \eqref{eq:distance} can therefore be decomposed into a translation and a rotation part:
\begin{multline}
\label{eq:distance_decomposition}
\mydist^2(\mathcal{P}_1, \mathcal{P}_2) = \|\mat{t}_2 - \mat{t}_1\|^2 \\
+ \underbrace{\min_{\mat{G}_1, \mat{G}_2  \in G } \cfrac{1}{S} \int_{\mathcal{S}} \mu(\mat{x}) \| \mat{R}_2 \mat{G}_2 \mat{x} - \mat{R}_1 \mat{G}_1 \mat{x} \|^2 ds}_{\mydistrotsquare(\mat{R}_1, \mat{R}_2)}.
\end{multline}	
		
In the following subsections, we show how this rotation part can be simplified, and how it leads to the notion of pose representatives.

	
\subsection{Object with no proper symmetries}
\label{sec:no_invariance_object}

Let us consider the case of an object showing no proper symmetries. The proper symmetry group of this object is reduced to the identity rotation: $G=\left\lbrace \mat{I} \right\rbrace$, therefore the rotation part of the square distance \eqref{eq:distance_decomposition} can be expressed as follows:
\begin{equation}
\label{eq:rotation_distance_no_invariance}
\mydistrotsquare(\mathcal{P}_1, \mathcal{P}_2) = \int_{\mathcal{S}} \mu(\mat{x}) \| \mat{R}_2 \mat{x} - \mat{R}_1 \mat{x} \|^2 ds
\end{equation}

Let us show how, for any given pose $\mathcal{P} \in \mathcal{C}$, one can define a representative $\mathcal{R}(\mathcal{P}) \in  \mathbb{R}^{12}$ such that the distance between two poses in $\mathcal{C}$ is the same as the distance between their representatives in $\mathbb{R}^{12}$. Our approach is inspired by the work of \citet{kazerounian1992} and \citet{chirikjian1998}.

Let $\matLambda$ be the symmetric positive semi-definite square root matrix of the covariance matrix of the object's weighted surface:
\begin{equation}
\label{eq:lambda_definition}
\boxed{
\matLambda \triangleq \left( \cfrac{1}{S} \int_\mathcal{S} \mu(\mat{x}) \mat{x} \mat{x}^\top ds \right)^{1/2}.
}
\end{equation}
$\matLambda$ does not depend on the considered pose and can therefore be estimated once and for all for a given rigid object in a preprocessing step. We provide in appendix~\ref{sec:triangular_mesh} formulas to compute $\matLambda$ when $\mathcal{S}$ is the surface of a triangular mesh. 

Rewriting the inner part of \eqref{eq:rotation_distance_no_invariance} with a trace operator,
$$\| \mat{R}_2 \mat{x} - \mat{R}_1 \mat{x} \|^2 = \Tr \left( (\mat{R}_2 - \mat{R}_1 ) \mat{x} \mat{x}^\top (\mat{R}_2 - \mat{R}_1)^\top \right),$$
one can express the rotation part of the squared distance in a closed form as a weighted Frobenius square distance between the rotation matrices:
\begin{equation}
\begin{aligned}
\label{eq:distance_rotation_frobenius_norm}
\mydistrotsquare(\mathcal{P}_1, \mathcal{P}_2) &= \Tr \left( (\mat{R}_2 - \mat{R}_1 ) \matLambda^2 (\mat{R}_2 - \mat{R}_1)^\top \right) \\
&= \| \mat{R}_2 \matLambda - \mat{R}_1 \matLambda \|_F^2.
\end{aligned}
\end{equation}

Therefore, denoting $\vect$ the operator vectorizing columnwise a matrix into a column vector, we can define an isometry $\mathcal{R}$ from the pose space into the 12-dimensional Euclidian space, such as
\begin{equation}
\label{eq:distance_no_rotational_invariance_representation}
\boxed
{
\begin{aligned}
&\text{Object without proper symmetry:} \\
&\begin{aligned}
\mydist^2(\mathcal{P}_1, \mathcal{P}_2) &= \| \mat{R}_2 \matLambda - \mat{R}_1 \matLambda \|_F^2 + \| \mat{t}_2 - \mat{t}_1\|^2 \\
&= \| \mathcal{R}(\mathcal{P}_2) - \mathcal{R}(\mathcal{P}_1) \|^2
\end{aligned} \\
&\text{with } \mathcal{R}(\mathcal{P}) \triangleq \left( \vect(\mat{R} \matLambda)^\top, \mat{t}^\top \right)^\top \in \mathbb{R}^{12}.\\
\end{aligned}
}
\end{equation}

The conversion from a pose represented in term of a rotation matrix $\mat{R}$ and a translation vector $\mat{t}$ to its representative in $\mathbb{R}^{12}$ is direct, since it consists in a simple linear operation.
If the object frame is moreover chosen aligned with the principal axes of the object, $\matLambda$ is diagonal, making the computation of the pose representative even cheaper.


\subsection{Revolution object without rotoreflection invariance}
\label{sec:revolution_object}

We now consider the case of a revolution object without rotoreflection invariance. As stated in section \ref{sec:pose_decomposition}, we assume that the origin of the object frame corresponds to the center of mass of the object. Without loss of generality, we moreover assume that the axis $\mat{e}_z$ of the object frame is aligned with the revolution axis. A pose $\mathcal{P}$ is thus defined up to a rotation $\mat{R}_z^\phi$ along the $\mat{e}_z$ axis, where $\phi$ is the angle of the considered rotation, and the proper symmetry group of the object consists in $G = \left\lbrace \mat{R}_z^\phi \middle\vert \phi \in \mathbb{R} \right\rbrace$.

The simplification to get rid of the integral within the distance expression thanks to the introduction of the matrix $\matLambda$ in section \ref{sec:no_invariance_object} is also valid here. Moreover, because $(\mat{O}, \mat{e_z})$ is the revolution axis of the object, $\matLambda$ is necessarily diagonal and of the form
\begin{equation}
\label{eq:lambda_expression_revolution}
\matLambda = \left( \begin{matrix}
\lambda_r & 0 & 0 \\
0 & \lambda_r & 0 \\
0 & 0 & \lambda_z \\
\end{matrix}
\right)
\end{equation}
with $\lambda_r, \lambda_z \in \mathbb{R}^+$.
This enables us to express the rotation part of the distance as a simple scaled distance between the revolution axes seen as 3D vectors:
\begin{equation}
\mydistrotsquare(\mathcal{P}_1, \mathcal{P}_2) = \lambda^2 \| \mat{R}_2 \mat{e}_z - \mat{R}_1 \mat{e}_z \|^2
\end{equation}
with $\lambda \triangleq \sqrt{\lambda_r^2 + \lambda_z^2}$. The reader is referred to appendix~\ref{app:proof_distance_simplification_revolution_object} for a proof of this result.

Therefore, similarly to what we proposed for an object without proper symmetry, we can consider a simple isometry  $\mathcal{R}$ which associates to a pose of a revolution object without rotoreflection invariance, a 6D vector, consisting of the concatenation of the coordinates of its scaled revolution axis and of its position, in order to efficiently evaluate distances:
\begin{equation}
\label{eq:distance_revolution_object}
\boxed
{
\begin{aligned}
&\text{Revolution object without rotoreflection invariance:} \\
&\begin{aligned}
\mydist^2(\mathcal{P}_1, \mathcal{P}_2) &= \|\mat{t}_2 - \mat{t}_1 \|^2 + \lambda^2 \|\mat{R}_2 \mat{e}_z - \mat{R}_1 \mat{e}_z \|^2 \\
&= \|\mathcal{R}(\mathcal{P}_2) - \mathcal{R}(\mathcal{P}_1) \|^2 \\
\end{aligned} \\
&\text{with } \mathcal{R}(\mathcal{P}) \triangleq \left( \lambda (\mat{R} \mat{e}_z)^\top, \mat{t}^\top \right)^\top \in \mathbb{R}^{6} \\
&\text{where } \lambda = \sqrt{\lambda_r^2 + \lambda_z^2}.
\end{aligned}
}
\end{equation}

	
\subsection{Spherical object}
\label{sec:sec:spherical_object}
We now consider the simpler case of an object with spherical symmetry. Choosing the center of the object as origin of the object frame, the proper symmetry group of the object is the whole rotation group $SO(3)$. The rotation part of the distance \eqref{eq:distance_decomposition} can thus be rewritten as follows:
\begin{equation}
\mydistrotsquare(\mathcal{P}_1, \mathcal{P}_2) = \min_{\mat{R}_1, \mat{R}_2} \left( \cfrac{1}{S} \int_{\mathcal{S}} \mu(\mat{x}) \| \mat{R}_2 \mat{x} - \mat{R}_1  \mat{x} \|^2 ds \right).
\end{equation}
This term is null, the minimum being reached for $\mat{R}_1 = \mat{R}_2$. Therefore, the pose space of a spherical object can be also isometrically embedded into a $\mathbb{R}^3$ by representing a pose by the position of its center:
\begin{equation}
\boxed
{
\begin{aligned}
&\text{Spherical object:} \\
&\begin{aligned}
\mydist^2(\mathcal{P}_1, \mathcal{P}_2) &= \| \mat{t}_2 - \mat{t}_1 \|^2 \\
&= \| \mathcal{R}(\mathcal{P}_2) - \mathcal{R}(\mathcal{P}_1) \|^2
\end{aligned} \\
&\text{With } \mathcal{R}(\mathcal{P}) = \mat{t} \in \mathbb{R}^3.
\end{aligned} 
}
\end{equation}


\subsection{Revolution object with rotoreflection invariance}
\label{sec:revolution_object_with_rotoreflection_invariance}
Let us consider the case of a revolution object with rotoreflection invariance, \ie having a reflection symmetry with respect to a plane orthogonal to the revolution axis.
With the same constraints on the choice of the object frame as for a revolution object without rotoreflection invariance, the proper symmetry group of such object can be written as follows:
\begin{equation}
G = \left\lbrace \mat{R}_x^\delta \mat{R}_z^\alpha \;|\; \alpha \in \mathbb{R}, \delta \in \left\lbrace 0, \pi \right\rbrace \right\rbrace.
\end{equation}

Therefore, the distance between two poses $\mathcal{P}_1, \mathcal{P}_2$ can be expressed as:
\begin{equation}
\min_{\delta_1, \delta_2, \phi_1, \phi_2} \mydistnosym\left((\mat{R}_1 \mat{R}_x^{\delta_1} \mat{R}_z^{\phi_1}, \mat{t}_1), (\mat{R}_2 \mat{R}_x^{\delta_2} \mat{R}_z^{\phi_2}, \mat{t}_2)\right).
\end{equation}

We discussed in section \ref{sec:revolution_object} how to compute such an expression relatively to the symmetries along the revolution axis. Therefore using result \eqref{eq:distance_revolution_object}, our distance can be rewritten as the minimum Euclidean distance between 6D points, two being assigned to each pose:
\begin{equation}
\mydist(\mathcal{P}_1, \mathcal{P}_2) = \min_{\delta_1, \delta_2 \in \left\lbrace 0, \pi \right\rbrace} \|\mat{p}_2^{\delta_2}  - \mat{p}_1^{\delta_1}\|
\end{equation}
with  $\mat{p}_i^\delta = \left( \lambda (\mat{R}_i \mat{R}_x^{\delta} \mat{e}_z)^\top, \mat{t}^\top \right)^\top \in \mathbb{R}^6$ the representatives of pose $\mathcal{P}_i$, for $\delta=0, \pi$ and $i=1,2$.

Simplifying the representative expression a little given that $\mat{R}_x^{0} \mat{e}_z = \mat{e}_z$ and $\mat{R}_x^{\pi} \mat{e}_z = -\mat{e}_z$, we see that a pose of a revolution object with rotoreflection invariance can be represented by two 6D vectors consisting of the concatenation of the coordinates of its scaled revolution axis and of its position, each potential orientation of the axis being taken into account by one representative:
\begin{equation}
\boxed{
\begin{aligned}
&\text{Revolution object with rotoreflection invariance:} \\
&\mydist(\mathcal{P}_1, \mathcal{P}_2) = \min_{\mat{p}_1 \in \mathcal{R}(\mathcal{P}_1), \mat{p}_2 \in \mathcal{R}(\mathcal{P}_2)} \| \mat{p}_2 - \mat{p}_1 \| \\
&\text{With } \mathcal{R}(\mathcal{P}) \triangleq \left\lbrace \left( \pm \lambda (\mat{R} \mat{e}_z)^\top, \mat{t}^\top \right)^\top \right\rbrace \subset \mathbb{R}^6. \\
\end{aligned}
}
\end{equation}


\subsection{Object with a nontrivial finite proper symmetry group}

The last type of 3D object to deal with is the case of an object with a finite proper symmetry group $G$ different from the identity, such as the object depicted in table~\ref{tab:symmetries_classes}e.
The proposed distance between two poses of such an object can be written as:
\begin{equation}
\min_{\mat{G}_1, \mat{G}_2 \in G} \mydistnosym((\mat{R}_1 \mat{G}_1, \mat{t}_1), (\mat{R}_2 \mat{G}_2, \mat{t}_2)).
\end{equation}

We showed in section \ref{sec:no_invariance_object} that the pose of an object without proper symmetry can be represented as a 12D point, such that the distance between two poses of such object corresponds to the Euclidean distance between their respective representatives.
Therefore, it is straightforward to conclude that the pose of an object with a finite proper symmetry group can be represented by a finite set of 12D representative points, such that the distance between two poses corresponds to the minimum Euclidean distance between their respective representatives:
\begin{equation}
\label{eq:distance_finite_proper_symmetry_group}
\boxed
{
\begin{aligned}
&\text{Object with a nontrivial finite proper symmetry group:} \\
&\mydist(\mathcal{P}_1, \mathcal{P}_2) = \min_{\mat{p}_1 \in \mathcal{R}(\mathcal{P}_1), \mat{p}_2 \in \mathcal{R}(\mathcal{P}_2)} \| \mat{p}_2 - \mat{p}_1 \| \\
&\text{With } \mathcal{R}(\mathcal{P}) \triangleq \left\lbrace \left( \vect(\mat{R} \mat{G} \matLambda)^\top, \mat{t}^\top \right)^\top \middle\vert \mat{G} \in G \right\rbrace \subset \mathbb{R}^{12}. \\
\end{aligned}
}
\end{equation}

\subsection{2D object}
The notion of pose representative can be applied to 2D objects as well. For the sake of conciseness, we will only discuss the case of a 2D object with no proper symmetry, as the reasoning is very similar to the one performed for 3D objects. The full list of proposed representatives is given in table \ref{tab:representations_synthesis_2D}.

The decomposition of the square distance between two poses in a translation and rotation terms \eqref{eq:distance_decomposition} and the expression of the rotation part as a Frobenius norm \eqref{eq:distance_rotation_frobenius_norm} are still valid in the 2D case, but they can be even further simplified. 
Indeed, a 2D rotation matrix can be parametrized by an angle $\theta$ as follows: 
\begin{equation}
\mat{R}^\theta = \left( \begin{matrix}
\cos(\theta) & -\sin(\theta) \\
\sin(\theta) & \cos(\theta) \\
\end{matrix} \right).
\end{equation}
Introducing the elements of the covariance matrix
\begin{equation}
\matLambda^2 = \left( \begin{matrix}
\lambda_{xx}^2 & \lambda_{xy}^2 \\
\lambda_{xy}^2 & \lambda_{yy}^2 \\
\end{matrix} \right),
\end{equation}
the rotation part can be simplified into
\begin{equation}
\begin{aligned}
\mydistrotsquare(\mathcal{P}_1, \mathcal{P}_2) &= \Tr \left( (\mat{R}^{\theta_2} - \mat{R}^{\theta_1} ) \matLambda^2 (\mat{R}^{\theta_2} - \mat{R}^{\theta_1} )^\top \right) \\
&= (\lambda_{xx}^2 + \lambda_{yy}^2) \| e^{i \theta_2} - e^{i \theta_1} \|^2
\end{aligned}
\end{equation}
where $e^{i \theta} \triangleq (\cos(\theta), \sin(\theta))$. Therefore, we can include in our framework a 2D object without proper symmetry, and represent a pose of such object by a 4D vector, consisting of the concatenation of the coordinates of its scaled complex orientation and of its position:
\begin{equation}
\boxed{
\begin{aligned}
&\text{2D object without proper symmetry:} \\
&\mydist(\mathcal{P}_1, \mathcal{P}_2) = \min_{\mat{p}_1 \in \mathcal{R}(\mathcal{P}_1), \mat{p}_2 \in \mathcal{R}(\mathcal{P}_2)} \| \mat{p}_2 - \mat{p}_1 \| \\
&\text{with } \mathcal{R}(\mathcal{P}) \triangleq  \left( \lambda e^{i \theta}, \mat{t}^\top \right)^\top \in \mathbb{R}^4 \\
& \text{where } \lambda = \left( \cfrac{1}{S} \int_\mathcal{S} \mu(\mat{x}) \| \mat{x} \|^2 ds \right)^{1/2}.\\
\end{aligned}
}
\end{equation}

\bibliographyInSubfile

\section{Symmetry within representatives}
\label{sec:representatives_symmetry}

Objects with finite non trivial symmetry groups and revolution objects with rotoreflection invariance admit several representatives per pose. This multiplicity of representatives expresses the proper symmetries of the object that are not accounted for in the expression of a representative, and leads to some symmetry properties within the set of representatives itself.
Formally, for a given object, we define a finite group of symmetry operations $G_\mathcal{R}$ on the ambient space $\mathbb{R}^N$. This group consists in the identity singleton in the case of objects admitting a single representative per pose, and is defined in table~\ref{tab:symmetries_within_representatives} for the other objects classes.
In this section we discuss some properties of this group. Those will be used in order to propose a method to properly average poses, in section~\ref{sec:averaging_multiple_representatives}.
\begin{table*}
\small\sf\centering
\caption{\label{tab:symmetries_within_representatives} Proposed symmetry operations on the ambient space for objects with multiple representatives per pose.}
\rule{\linewidth}{1.5pt}
\textbf{Notations} \\
We decompose a point of the ambient space $\mathbb{R}^N$ into two parts as follows, depending on the dimension $N$ of the space: \\
\begin{varwidth}{\textwidth}
\begin{itemize}
\item $(\vect(\mat{M})^\top, \mat{t}^\top)^\top$ for a 12D space, with $\mat{M} \in \mathcal{M}_{3,3}(\mathbb{R})$, and $\mat{t} \in \mathbb{R}^3$.
\item $(\mat{a}^\top, \mat{t}^\top)^\top$ for a 6D space, with $\mat{a}, \mat{t} \in \mathbb{R}^3$.
\item $(\mat{a}^\top, \mat{t}^\top)^\top$ for a 4D space, with $\mat{a}, \mat{t} \in \mathbb{R}^2$. $\mat{a}$.  
\end{itemize}
\end{varwidth} \\
In the 4D case, we use the complex multiplication notation, assimilating $\mat{a}$ to a complex number.\\
\renewcommand{\arraystretch}{2.5}
\rule{\linewidth}{0.5pt}
\begin{tabular}{ccccc}
\textbf{Object type} & \textbf{Proper symmetry class} & \textbf{Symmetry group $G_\mathcal{R}$}  & \textbf{Symmetry definition} \\
\hline
\multirow{2}{*}{3D}
& Finite
& $\left\lbrace s_{\mat{G} } \middle\vert \mat{G} \in G \right\rbrace$ 
& $\begin{aligned}
s_{\mat{G}} : \mathbb{R}^{12} &\to \mathbb{R}^{12} \\
(\vect(\mat{M})^\top, \mat{t}^\top)^\top  &\mapsto (\vect(\mat{M} \mat{G})^\top, \mat{t}^\top)^\top\end{aligned}$
 \\
\cline{2-4}
& Revolution with rotoreflection invariance
& $\left\lbrace s_{\text{rev}, \delta} \middle\vert \delta = \pm 1 \right\rbrace$
& $\begin{aligned}
s_{\text{rev}, \delta} : \mathbb{R}^{6} &\to \mathbb{R}^{6} \\
(\mat{a}^\top, \mat{t}^\top)^\top &\mapsto (\delta \mat{a}^\top, \mat{t}^\top)^\top
\end{aligned}$ \\
\hline
2D & Cyclic (order $n \in \mathbb{N}^*$)
& $\left\lbrace s_{\text{2D}, n, k} \middle\vert k \in \llbracket 0, n \llbracket \right\rbrace$
& $\begin{aligned}
s_{\text{2D}, n, k} : \mathbb{R}^{4} &\to \mathbb{R}^{4} \\
(\mat{a}^\top, \mat{t}^\top)^\top &\mapsto (e^{i 2 k \pi / n} \cdot \mat{a}, \mat{t}^\top)^\top.
\end{aligned}$ \\
\end{tabular}
\rule{\linewidth}{1.5pt}
\end{table*}

First, we ensure that the proposed group is well defined:
\begin{proposition}
$G_\mathcal{R}$ is a group for the composition operation.
\end{proposition}
\begin{proof} 
This property derives directly from the group properties of $\mat{G}$, $\lbrace 1, -1 \rbrace$ and $\lbrace e^{i 2 k \pi / n} \vert k \in \llbracket 0, n \rrbracket \rbrace$ for multiplication operations.  
\qed
\end{proof}


Then, we introduce the following lemma, which somehow expresses the fact that the geometry of the object is consistent with the object's symmetries:
\begin{lemma}
\label{lem:commutation_g_lambda}
For any proper symmetry $\mat{G} \in G$, $\mat{G}$ and $\matLambda$ commute, \ie $\mat{G} \matLambda = \matLambda \mat{G}$.
\end{lemma}
\begin{proof} 
Let $\mat{G}$ be a proper symmetry in $G$. By definition of $\matLambda^2$,
\begin{equation}
\mat{G} \matLambda^2 =  \cfrac{1}{S} \int_\mathcal{S} \mu(\mat{x}) \mat{G} \mat{x} \mat{x}^\top ds.
\end{equation}
Performing the change of variable $\mat{x} \leftarrow \mat{G} \mat{x}$ enables to rewrite this expression into:
\begin{equation}
\cfrac{1}{S} \int_{\mat{G}(\mathcal{S})} \mu(\mat{G}^{-1} \mat{x}) \mat{x} (\mat{G}^{-1} \mat{x})^\top ds.
\end{equation}
Thanks to the invariance of $\mathcal{S}$ and $\mu$ to the proper symmetries of the object, we exhibit back $\matLambda^2$ as follows:
\begin{equation}
\begin{aligned}
\mat{G} \matLambda^2  &= \cfrac{1}{S} \int_{\mathcal{S}} \mu(\mat{x}) \mat{x} \mat{x}^\top \mat{G}^{-\top} ds \\
&= \matLambda^2 \mat{G}^{-\top}.
\end{aligned}
\end{equation}
$\mat{G}$ being a rotation, $\mat{G}^{-\top} = \mat{G}$, and therefore $\mat{G}$ and $\matLambda^2$ commute, \ie
\begin{equation}
\label{eq:commutation_g_lambda2}
\mat{G} \matLambda^2 = \matLambda^2 \mat{G}.
\end{equation}

Moreover, as a positive semi-definite symmetric matrix, $\matLambda^2$ admits an eigenvalue decomposition $\matLambda^2 = \mat{U} \mat{D} \mat{U}^\top$, 
where $\mat{U} \in SO(3)$ and $\mat{D}$ is a positive semi-definite diagonal matrix. Injecting this decomposition into the right hand side of equation~\ref{eq:commutation_g_lambda2}, we observe that $\mat{G}^\top \mat{U}$ is also an eigenbasis of $\matLambda^2$:
\begin{equation}
\matLambda^2 = (\mat{G}^\top \mat{U}) \mat{D} (\mat{G}^\top \mat{U})^\top.
\end{equation}

$\matLambda$ being the principal square root of $\matLambda^2$, both share the same eigenspaces, thus:
\begin{equation}
\left\lbrace
\begin{matrix}
\matLambda = \mat{U} \mat{D}^{1/2} \mat{U}^\top \\
\matLambda = (\mat{G}^\top \mat{U}) \mat{D}^{1/2} (\mat{G}^\top \mat{U})^\top.
\end{matrix}
\right.
\end{equation}
Therefore, injecting the first equality into the second one, we proved that
\begin{equation}
\matLambda = \mat{G}^\top \matLambda \mat{G}
\end{equation}
\ie, that $\mat{G}$ and $\matLambda$ commute: $\mat{G} \matLambda = \matLambda \mat{G}$.
\qed
\end{proof} 

Thanks to this lemma, it is now possible to exhibit the three following properties of those symmetries within the ambient space:
\begin{proposition}
\label{prop:symmetric_equal_representatives}
$G_\mathcal{R}$  contains $|\mathcal{R}(\bullet)|$ elements, and given a pose $\mathcal{P}$ and one of its representative $\mat{p} \in \mathcal{R}(\mathcal{P})$, the set of elements symmetric to $\mat{p}$ (including itself) is the whole set of representatives of the pose, \ie
\begin{equation}
\left\lbrace s(\mat{p}) \middle\vert s \in G_\mathcal{R} \right\rbrace  = \mathcal{R}(\mathcal{P}).
\end{equation}
\end{proposition}

\begin{proof}
This proposition is easily verified thanks to the expression of pose representatives tables~\ref{tab:representations_synthesis} and~\ref{tab:representations_synthesis_2D}. 
The only subtlety is the case of a 3D object with a finite proper symmetry group. In this case for any $\mat{G} \in G$, the symmetric by $s_{\mat{G}}$ of a pose representative $(\vect(\mat{R} \matLambda)^\top, \mat{t}^\top)^\top$, where $\mat{R} \in \mathcal{M}_{3, 3}(\mathbb{R})$ and $\mat{t} \in \mathbb{R}^3$, can be expressed as
\begin{equation}
s_{\mat{G}} \left( (\vect(\mat{R} \matLambda)^\top, \mat{t}^\top)^\top  \right) =  (\vect(\mat{R}  \matLambda \mat{G})^\top, \mat{t}^\top)^\top,
\end{equation}
which according to lemma~\ref{lem:commutation_g_lambda} is equal to $(\vect(\mat{R} \mat{G} \matLambda)^\top, \mat{t}^\top)^\top$. By definition of representatives for such object, it is therefore a representative of the same pose $\mathcal{P}$.
\qed
\end{proof}

\begin{proposition}
\label{prop:linearity_representative_symmetry}
Elements of $G_\mathcal{R}$ are linear transformations of the ambient space, \ie for  $s \in G_\mathcal{R}$, and for any $\mat{x}_1, \mat{x}_2 \in \mathbb{R}^N$, and $\alpha \in \mathbb{R}$,
\begin{equation}
s(\mat{x}_1 + \alpha \mat{x}_2) = s(\mat{x}_1) + \alpha s(\mat{x}_2).
\end{equation}
\end{proposition}
This proposition is a direct consequence of the definition of symmetries described table~\ref{tab:symmetries_within_representatives}.

\begin{proposition}\label{prop:automorphism}Elements of $G_\mathcal{R}$ are automorphisms of the ambient space: for any $s \in G_\mathcal{R}$, $s$ is bijective, and for any $\mat{x}_1, \mat{x}_2 \in \mathbb{R}^N$,
\begin{equation}
\| s(\mat{x}_2) - s(\mat{x}_1) \| = \| \mat{x}_2 - \mat{x}_1 \|.
\end{equation}
\end{proposition}

\begin{proof}
Bijectivity is straightforward since
\begin{itemize}
\item $(s_{\mat{G}})^{-1} = s_{\mat{G}^{-1}}$ for any $\mat{G} \in G$.
\item $(s_{\text{rev}, \delta})^{-1} = s_{\text{rev}, \delta}$ for any $\delta \in \left\lbrace -1, 1 \right\rbrace$.
\item $(s_{\text{2D}, n, k})^{-1} = s_{\text{2D}, n, -k}$ for any $k \in \mathbb{N}$.
\end{itemize}
The morphism property comes from the linearity of those symmetry operations (proposition~\ref{prop:linearity_representative_symmetry}) and the fact they preserve the norm, since elements of $\mat{G}$, $\lbrace 1, -1 \rbrace$ and $\lbrace e^{i 2 k \pi / n} \vert k \in \llbracket 0, n \llbracket \rbrace$ are themselves of unit norm.
\qed
\end{proof}


\section{Projection onto the pose space}
\label{sec:projection}

In section \ref{sec:pose_representation}, we discussed how a pose $\mathcal{P}$ can be identified to a finite pointset $\mathcal{R}(\mathcal{P})$ of an Euclidean space of finite dimension $\mathbb{R}^N$ and how elements of $\mathcal{R}(\mathcal{P})$ can be computed easily from any rigid transformation $(\mat{R}, \mat{t})$ associated to the pose.
The backward mapping is possible, and for any element of $\mathcal{R}(\mathcal{P})$ we can compute a rigid transformation fully describing the pose $\mathcal{P}$. Hence we consider an element of $\mathcal{R}(\mathcal{P})$ as a \emph{representative} of $\mathcal{P}$.
This computation is actually straightforward given the expressions of poses representatives (see tables~\ref{tab:representations_synthesis} and~\ref{tab:representations_synthesis_2D}), thus we choose to discuss this assertion in the more general framework of projection onto the pose space: given an arbitrary $N$-D vector $\mat{x}$, find out what pose has the most similar representative to $\mat{x}$.
The results of this section will be useful in section~\ref{sec:averaging_poses} to propose a method for pose averaging.

\myframedminipage{%
\begin{definition}
We define as projections of $\mat{x} \in \mathbb{R}^N$ the poses:
\begin{equation}
\label{eq:projection}
\proj(\mat{x}) \triangleq \argmin_{\mathcal{P}} \min_{\mat{p} \in  \mathcal{R}(\mathcal{P})} \| \mat{p} -\mat{x} \|^2 
\end{equation}
\end{definition}
}
In nondegenerate cases, the projection is unique, and we propose in the next subsections its expression for the different classes of bounded objects, based on the computation of the closest pose representative to the query point.


\subsection{Spherical object}

Projection is trivial in the case of a spherical object, since all points of $\mathbb{R}^3$ are valid representatives of poses. A point $\mat{x} \in \mathbb{R}^3$ therefore projects onto the pose having $\mat{x}$ for representative, namely the pose in which the center of the object admits $\mat{x}$ for 3D coordinates.


\subsection{Object of revolution}
\label{sec:projection_revolution_object}
In the case of a revolution object without rotoreflection invariance, the position of the center of mass and the oriented revolution axis of the object are well defined at any given pose. 
Reciprocally, a pose can be defined by the position of its center of mass $\mat{t}$ and its oriented revolution axis, that we represent by a normalized vector $\mat{a} \in \mathbb{R}^3$. The unique representative of such a pose is $(\lambda \mat{a}^\top, \mat{t}^\top)^\top$ as we defined in \ref{sec:revolution_object}. 

Let $\mat{x} \in \mathbb{R}^6$ be a point to project onto the pose space. Without loss of generality, $\mat{x}$ can be split into two parts: $\mat{x} = (\mat{x}_r^\top, \mat{x}_t^\top)^\top$ with $\mat{x}_r, \mat{x}_t \in \mathbb{R}^3$. The projection problem can therefore be reformulated into:
\begin{equation}
\begin{aligned}
\proj(\mat{x}) &= \argmin_{\mathcal{P}} \| \mat{x} - \mathcal{R}(\mathcal{P}) \|^2 \\ &= \argmin_{\mat{a}, \mat{t} \in \mathbb{R}^3 / \|\mat{a}\| = 1} \left( \| \mat{x}_r - \lambda \mat{a}\|^2 + \| \mat{x}_t - \mat{t}\|^2 \right) \\
\end{aligned}
\end{equation}

This problem admits an unique solution as long as $ \mat{x}_r \neq \mat{0}$, and in that case the  projection of $\mat{x}$ is the pose of center of mass $\mat{\hat{t}} = \mat{x}_t$ and of axis $\mat{\hat{a}}=\mat{x}_r / \| \mat{x}_r \|$.
This result holds true in the case of an object of revolution with rotoreflection invariance, since $(\lambda \mat{\hat{a}}^\top, \mat{\hat{t}}^\top)^\top$ is the closest pose representative to $\mat{x}$.


\subsection{Object with a finite proper symmetry group}
\label{sec:projection_finite_proper_symmetry_group}
The representative of a pose of an object without proper symmetry is a 12D vector, the first 9 dimensions representing the orientation in the form of a vectorized matrix and the 3 others the position of the object.
Therefore, and without loss of generality, we can split a point $\mat{x} \in \mathbb{R}^{12}$ to project in a similar fashion: $\mat{x} = (\vect(\mat{X}_r)^\top, \mat{x}_t^\top)^\top$ with $\mat{x}_t \in \mathbb{R}^3$ and $\mat{X}_r \in \mathcal{M}_{3,3}(\mathbb{R})$.
The projection problem for such a point $\mat{x}$ -- in the case of an object without proper symmetry -- thus consists in:
\begin{equation}
\begin{aligned}
\proj(\mat{x}) &= \argmin_{\mathcal{P}} \| \mat{x} - \mathcal{R}(\mathcal{P}) \|^2 \\ &= \argmin_{\mat{R}, \mat{t}} \left( \| \mat{X}_r - \mat{R} \matLambda \|_F^2 + \| \mat{x}_t - \mat{t}\|^2 \right) \\
\end{aligned}
\end{equation}
 The two terms being independent, we conclude again that the position of the center of mass of the object for a projection of $\mat{x}$ is $\mat{\hat{t}} = \mat{x}_t$.
The minimization problem regarding the orientation part is in the form of the so-called constrained orthogonal Procrustes problem \citep{schoenemann1966,umeyama1991} and admits the solution $\mat{\hat{R}} = \mat{U} \mat{S} \mat{V}^\top$, where $\mat{U} \mat{D} \mat{V}^\top$ is a singular value decomposition of $\mat{X}_r \matLambda$ such as
\begin{equation}
\mat{D} = \diag(\alpha_1, \alpha_2, \alpha_3),
\end{equation}
with  $\alpha_1 \geq \alpha_2 \geq \alpha_3 \geq 0$ and
\begin{equation}
\mat{S} = \left\lbrace 
\begin{matrix} 
\mat{I} \text{ if } \det(\mat{U}) \det(\mat{V}) > 0 \\
\diag(1, 1, -1) \text{ otherwise.}
\end{matrix}
\right.
\end{equation}

The projection is unique if $\rank(\mat{X}_r \matLambda^\top) \geq 2$ \citep{umeyama1991}, a condition that is fulfilled in most practical cases.
This result also holds true in the general case of an object with a finite proper symmetry group, as $(\vect(\mat{\hat{R}} \matLambda)^\top, \mat{\hat{t}}^\top)^\top$ is the closest pose representative to $\mat{x}$.

\subsection{2D object}

The projection problem for a 2D object is similar to the 3D case.

In the case of a circular object, any point of $\mat{x} \in \mathbb{R}^2$ is a valid representative of a pose and therefore projects onto the pose of center $\mat{x}$.

Regarding an object with cyclic symmetry, we conclude by the same reasoning as in the case of a 3D revolution object that a 4D vector $\mat{x} = (\mat{a}^\top, \mat{t}^\top)^\top$, where $\mat{a}, \mat{t} \in \mathbb{R}^2$, admits an unique projection as long as  $\| \mat{a} \| \neq 0$. The projection admits a representative  $(\lambda / \| \mat{a} \| \cdot \mat{a}^\top, \mat{t}^\top)^\top$,  and consists of the pose defined by a translation $\mat{t}$ and a rotation of angle $\arg(\mat{a})$, where $\arg(\mat{a})$ is the argument of $\mat{a}$ seen as a complex number.


\section{Averaging poses}
\label{sec:averaging_poses}

Pose averaging is of great use for applications such as denoising, modes detection or interpolation.
Definition of the average is not obvious in non-vector spaces such as ours, and we therefore consider a generalization of the average to arbitrary metric spaces, known as the Fréchet mean.

Let us consider a finite set of poses $S = \left\lbrace \mathcal{P}_i \right\rbrace_{i=1..n}$ and a set of strictly positive weights $\left\lbrace w_i \right\rbrace_{i=1..n}$ assigned to each of those. The weighted mean of poses $S$ is by definition the pose which minimizes the corresponding Fréchet variance:\begin{equation}
\label{eq:mean_definition}
\mean(S) \triangleq \argmin_{\mathcal{P} \in \mathcal{C}} \frechetvariance(\mathcal{P}),
\end{equation}
the Fréchet variance at a pose $\mathcal{P} \in \mathcal{C}$ being expressed as follows:
\begin{equation}
\frechetvariance(\mathcal{P}) \triangleq \sum_{i=1}^n w_i \mydist^2(\mathcal{P}_i, \mathcal{P}).
\end{equation}

This mean is not necessarily well defined, since the minimum of Fréchet variance is not necessarily reached at a unique pose.
However, such cases typically occur in configurations where the average would actually be meaningless, \eg when averaging two poses of opposite axes for a revolution object without rotoreflection invariance.

The problem of pose averaging has already been studied for objects without proper symmetry with various metrics. \citet{sharf2010} notably compare different averaging techniques for the rotation part of a pose, using common metrics.
While there is no known closed-form solution for the Riemannian metric \eqref{eq:distance_classical_riemannian}, it can be computed iteratively, and admits closed form approximations which are ``good enough'' for practical applications \citep{gramkow2001}.
A good approximation, when dealing with more than two poses, is based on computing the average of rotation matrices, and actually corresponds to the exact average when considering the distance \eqref{eq:distance_froebenius_norm} \citep{curtis1993}.

In the case of our proposed distance, the expression of the Fréchet variance can be developed into:
\begin{equation}
\label{eq:frechet_variance_developped}
\frechetvariance(\mathcal{P}) = \sum_{i=1}^{n} w_i \min_{\mat{p}_i \in \mathcal{R}(\mathcal{P}_i),\, \mat{p} \in \mathcal{R}(\mathcal{P})} \| \mat{p}_i - \mat{p} \|^2.
\end{equation}

Considering a given tuple  $P = (\mat{p}_i)_{i=1 \dots n} \in \prod_i \mathcal{R}(\mathcal{P}_i)$ of representatives of the poses to average, the weighted sum of square distances from a pose representative $\mat{p}$ to elements of $P$ can be split into two terms, through the introduction of the arithmetic mean $\mat{m}_{P}$ of the elements of $P$:
\begin{equation}
\label{eq:frechet_variance_split}
\begin{split}
\sum_{i=1}^{n} w_i \| \mat{p}_i - \mat{p} \|^2 &= \sum_i w_i \| \mat{p}_i - \mat{m}_{P} \|^2 \\
&\quad + \left( \sum_i w_i \right) \| \mat{p} - \mat{m}_{P} \|^2,
\end{split}
\end{equation}
with the arithmetic mean
\begin{equation}
\mat{m}_{P} \triangleq \cfrac{\sum_i w_i \mat{p}_i}{\sum_i w_i}.
\end{equation}

The first term of \eqref{eq:frechet_variance_split} is independent of $\mat{p}$. Therefore, the problem of minimizing \eqref{eq:frechet_variance_split} for a given tuple $P$ is reduced to the problem of finding a pose $\mathcal{P}$ which minimizes $\min_{\mat{p} \in \mathcal{R}(\mathcal{P})} \| \mat{p} - \mat{m}_{P} \|^2$.
This corresponds to the projection problem we discussed and solved in section \ref{sec:projection}.
The average pose, if well defined, is thus the projection of the arithmetic average of a combination of representatives of the poses to average, or more formally:
\begin{gather}
\mean(S) = \argmin_{\mathcal{P} \in \mathcal{A}} \frechetvariance(\mathcal{P}),\\
\intertext{where}
\mathcal{A} \triangleq \left\lbrace \proj\left( \mat{m}_P \right)\, \middle\vert\, P \in \prod_i \mathcal{R}(\mathcal{P}_i) \right\rbrace.
\end{gather}


\subsection{Objects with a single representative per pose}
\label{sec:averaging_single_representation}

Because the projection is unique in nondegenerate cases, the conclusion is straightforward for spherical objects, objects of revolution without rotoreflection invariance, and objects without proper symmetry. Since poses of such objects admit only one representative, a single tuple of representatives has to be considered.
For those objects, the average pose exists and corresponds simply to the projection of the arithmetic average of their representatives:
\begin{equation}
\boxed{
\mean(S) = \proj\left( \cfrac{\sum_i w_i \mathcal{R}(\mathcal{P}_i)}{\sum_i w_i} \right)
}
\end{equation}


\subsection{Objects with multiple representatives per pose}
\label{sec:averaging_multiple_representatives}

In the other cases, a pose admits several representatives and one should consider the different combinations of representatives to find the exact average -- assuming its existence and uniqueness in nonpathological cases.
This problem is not specific to our method and is similar to the issue encountered when averaging orientations of an object without proper symmetries through the arithmetic mean of their quaternions representatives, each orientation admitting two antipodal quaternions as representatives.
Because the number of combinations of representatives is exponential in the number of considered poses, the exact computation of the average might easily become expensive.

A common practice to circumvent this issue when averaging orientations based on a quaternion representation is to compute the arithmetic mean of a ``consistent'' combination of representatives and consider its projection as the average.
Such a combination is usually built by choosing a representative for an arbitrary initial pose, and then picking for each pose the nearest representative to the initial one \citep{gramkow2001}.
This approach is simple, but is in the general case ill-defined. Indeed, the chosen combination -- and hence the estimated average -- are in general dependent of the initial choice.
We depict an illustration on figure~\ref{fig:averaging_poses} of such a case, where three different choices for the initial pose lead to three different choices of ``consistent'' representatives combinations and therefore three different estimations of the average.

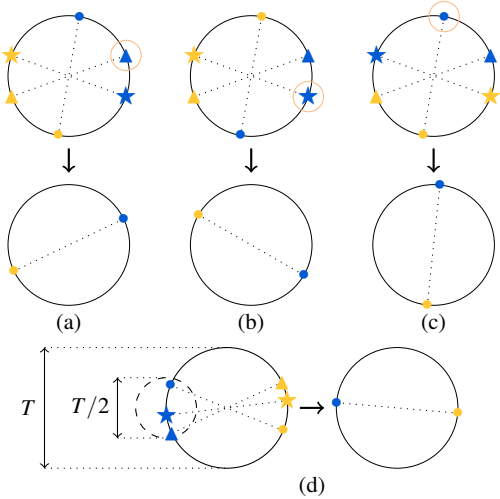
\begin{figure}
\centering
\colorlet{color11}{color1}
\colorlet{color12}{color2}
\colorlet{color21}{color1}
\colorlet{color22}{color2}
\colorlet{color31}{color1}
\colorlet{color32}{color2}

\definecolor{refpointcirclecolor}{HTML}{FDBB84}
\newcommand{\highlightpoint}[1]
{
\draw[refpointcirclecolor] {#1} circle (0.25);
}

\begin{tikzpicture}[scale=0.8]
\def\angleone{80}
\def\angletwo{20}
\def\anglethree{-20}

\begin{scope}[shift={(0,0)}];
\draw[very thin] (0,0) circle (1);
\highlightpoint{(\angletwo:1)};
\draw[dotted] (\angleone:1) node {\textcolor{color11}{$\bullet$}} -- (\angleone:-1) node{\textcolor{color12}{$\bullet$}};
\draw[dotted] (\angletwo:1) node{\textcolor{color11}{$\blacktriangle$}} -- (\angletwo:-1) node{\textcolor{color12}{$\blacktriangle$}};
\draw[dotted] (\anglethree:1) node{\textcolor{color11}{$\bigstar$}} -- (\anglethree:-1) node{\textcolor{color12}{$\bigstar$}};
\draw[thick, arrows=->] (0, -1.2) -- (0, -1.6);
\def\aravx{{(cos(\angleone) + cos(\angletwo) + cos(\anglethree))/3}}
\def\aravy{{(sin(\angleone) + sin(\angletwo) + sin(\anglethree))/3}}
\def\avangle{{atan2(\aravy, \aravx)}}
\begin{scope}[shift={(0, -2.8)}];
\draw[very thin] (0,0) circle (1);
\draw[dotted] (\avangle:1) node {\textcolor{color11}{$\bullet$}} -- (\avangle:-1) node {\textcolor{color12}{$\bullet$}};
\draw (0, -1) node [below] {(a)};
\end{scope}
\end{scope}

\begin{scope}[shift={(3,0)}];
\draw[very thin] (0,0) circle (1);
\highlightpoint{(\anglethree:1)};
\draw[dotted] (\angleone:1) node {\textcolor{color22}{$\bullet$}} -- (\angleone:-1) node{\textcolor{color21}{$\bullet$}};
\draw[dotted] (\angletwo:1) node{\textcolor{color21}{$\blacktriangle$}} -- (\angletwo:-1) node{\textcolor{color22}{$\blacktriangle$}};
\draw[dotted] (\anglethree:1) node{\textcolor{color21}{$\bigstar$}} -- (\anglethree:-1) node{\textcolor{color22}{$\bigstar$}};
\draw[thick, arrows=->] (0, -1.2) -- (0, -1.6);
\def\aravx{{(cos(\angleone + 180) + cos(\angletwo) + cos(\anglethree))/3}}
\def\aravy{{(sin(\angleone + 180) + sin(\angletwo) + sin(\anglethree))/3}}
\def\avangle{{atan2(\aravy, \aravx)}}
\begin{scope}[shift={(0, -2.8)}];
\draw[very thin] (0,0) circle (1);
\draw[dotted] (\avangle:1) node {\textcolor{color21}{$\bullet$}} -- (\avangle:-1) node {\textcolor{color22}{$\bullet$}};
\draw (0, -1) node [below] {(b)};
\end{scope}
\end{scope}

\begin{scope}[shift={(6,0)}];
\draw[very thin] (0,0) circle (1);
\highlightpoint{(\angleone:1)};
\draw[dotted] (\angleone:1) node {\textcolor{color31}{$\bullet$}} -- (\angleone:-1) node{\textcolor{color32}{$\bullet$}};
\draw[dotted] (\angletwo:1) node{\textcolor{color31}{$\blacktriangle$}} -- (\angletwo:-1) node{\textcolor{color32}{$\blacktriangle$}};
\draw[dotted] (\anglethree:1) node{\textcolor{color32}{$\bigstar$}} -- (\anglethree:-1) node{\textcolor{color31}{$\bigstar$}};
\draw[thick, arrows=->] (0, -1.2) -- (0, -1.6);
\def\aravx{{(cos(\angleone) + cos(\angletwo) + cos(\anglethree + 180))/3}}
\def\aravy{{(sin(\angleone) + sin(\angletwo) + sin(\anglethree + 180))/3}}
\def\avangle{{atan2(\aravy, \aravx)}}
\begin{scope}[shift={(0, -2.8)}];
\draw[very thin] (0,0) circle (1);
\draw[dotted] (\avangle:1) node {\textcolor{color31}{$\bullet$}} -- (\avangle:-1) node {\textcolor{color32}{$\bullet$}};
\draw (0, -1) node [below] {(c)};
\end{scope}
\end{scope}

\def\angleone{180-22}
\def\angletwo{180+25}
\def\anglethree{180+7}
\begin{scope}[shift={(4,-5.5)}, scale=1];
\begin{scope}[shift={(-1.4, 0)}];
\draw (0,0) circle (1);
\draw [dashed] (-1, 0) circle (0.5);
\draw [arrows=<->] (-1.8, -0.5) -- node[left] {$T/2$} (-1.8, 0.5);
\draw [dotted] (-1.8, -0.5) -- (-1, -0.5);
\draw [dotted] (-1.8, 0.5) -- (-1, 0.5);
\draw [arrows=<->] (-3, -1) -- node[left] {$T$} (-3, 1);
\draw [dotted] (-3, -1) -- (0, -1);
\draw [dotted] (-3, 1) -- (0, 1);
\draw[dotted] (\angleone:1) node {\textcolor{color31}{$\bullet$}} -- (\angleone:-1) node{\textcolor{color32}{$\bullet$}};
\draw[dotted] (\angletwo:1) node{\textcolor{color31}{$\blacktriangle$}} -- (\angletwo:-1) node{\textcolor{color32}{$\blacktriangle$}};
\draw[dotted] (\anglethree:1) node{\textcolor{color31}{$\bigstar$}} -- (\anglethree:-1) node{\textcolor{color32}{$\bigstar$}};
\end{scope};
\draw[thick, arrows=->] (-0.2, 0) -- (0.2, 0);
\def\aravx{{(cos(\angleone) + cos(\angletwo) + cos(\anglethree + 180))/3}}
\def\aravy{{(sin(\angleone) + sin(\angletwo) + sin(\anglethree + 180))/3}}
\def\avangle{{atan2(\aravy, \aravx)}}
\begin{scope}[shift={(1.4, 0)}];
\draw[very thin] (0,0) circle (1);
\draw[dotted] (\avangle:1) node {\textcolor{color31}{$\bullet$}} -- (\avangle:-1) node {\textcolor{color32}{$\bullet$}};
\end{scope}

\draw (0, -1) node [below] {(d)};
\end{scope}
\end{tikzpicture}
\caption{Estimating the average of poses with multiple representatives: illustration with the orientation of a 2D object with a 180\textdegree{} rotation symmetry, that can be represented by a point on a circle, or the antipodal point. We consider three poses to average (triangle, star and disk shapes).
(a, b, c) The choice of a ``consistent'' combination of poses representatives (blue clusters) in the sense of \citep{gramkow2001} is dependent of the initial choice (circled, first row), resulting potentially in different estimations of the average pose (second row). 
(d) We propose a definition of consistency -- which is in particular satisfied when the considered representatives are close enough one an other -- ensuring an unambiguous estimation of the average pose.}
\label{fig:averaging_poses}
\end{figure}

\bibliographyInSubfile

In this subsection, we propose an stricter definition of the \emph{consistency} of a combination of representatives and prove that it enables an unambiguous estimation of the mean.

\begin{definition}[Consistency] A tuple $(\mat{p}_i)_{i=1 \dots n} \in \prod_i \mathcal{R}(\mathcal{P}_i)$ is said \emph{consistent} if and only if
\begin{multline}
\forall (i,j) \in \llbracket 1, n \rrbracket^2, \forall \mat{q}_j \in \mathcal{R}(\mathcal{P}_j) \setminus \left\lbrace \mat{p_j} \right\rbrace, \\
\| \mat{p}_j - \mat{p}_i \| < \| \mat{q}_j - \mat{p}_i \|.
\end{multline}
\end{definition}

In other words, a consistent tuple is a set of pose representatives closer one another than to any other representatives.

\begin{proposition}[Uniqueness of a consistent tuple, up to symmetry]
\label{prop:uniqueness_up_to_symmetry_consistent_tuple}
If $(\mat{p}_i)_{i=1 \dots n} \in \prod_i \mathcal{R}(\mathcal{P}_i)$ is consistent, then the set of consistent tuples of $\prod_i \mathcal{R}(\mathcal{P}_i)$ is the set composed of $(\mat{p}_i)_{i=1 \dots n}$ and its symmetric tuples
\begin{equation}
\left\lbrace (s(\mat{p}_i))_{i=1 \dots n} \middle\vert s \in G_\mathcal{R} \right\rbrace.
\end{equation}
\end{proposition}

\begin{proof}
Let $(\mat{p}_i)_{i=1 \dots n}, (\mat{q}_i)_{i=1 \dots n} \in \prod_i \mathcal{R}(\mathcal{P}_i)$ be two different consistent tuples. There exists $j \in \llbracket 1, n \rrbracket$ such as $\mat{p}_j \neq \mat{q}_j$, and we know from the consistency definition of $(\mat{p}_i)_{i=1 \dots n}$ and $(\mat{q}_i)_{i=1 \dots n}$ that
\begin{equation}
\forall i \in \llbracket 1, n \rrbracket, 
\left\lbrace
\begin{matrix}
\| \mat{p}_j - \mat{p}_i \| < \| \mat{q}_j - \mat{p}_i \| \\
\| \mat{q}_j - \mat{q}_i \| < \| \mat{p}_j - \mat{q}_i \|.
\end{matrix} \right.
\end{equation}
If there existed $i \in \llbracket 1, n \rrbracket$ such as $\mat{p}_i = \mat{q}_i$, it would lead to the inequality $\| \mat{p}_j - \mat{p}_i \| < \| \mat{p}_j - \mat{p}_i \|$ which is a contradiction. The tuples $(\mat{p}_i)_{i=1 \dots n}$ and $(\mat{q}_i)_{i=1 \dots n}$ are thus pairwise disjoint, and there exists therefore at most $|\mathcal{R}(\bullet)|$ consistent tuples.

There are moreover exactly $|\mathcal{R}(\bullet)|$ different representative combinations symmetric to $(\mat{p}_i)_{i=1 \dots n}$ -- including itself (proposition~\ref{prop:symmetric_equal_representatives}):
\begin{equation}
\left\lbrace (s(\mat{p}_i))_{i=1 \dots n} \middle\vert s \in G_{\mathcal{R}} \right\rbrace.
\end{equation}
Those combinations are consistent as much as $(\mat{p}_i)_{i=1 \dots n}$, since symmetry operations are morphisms (proposition~\ref{prop:automorphism}).
Hence the uniqueness up to symmetry of a consistent tuple of representatives.
\qed
\end{proof}

\begin{proposition}[Invariance of the projection under symmetry of representatives]
\label{prop:invariance_of_projection_to_rep_sym}
Let $\mat{x} \in \mathbb{R}^N$ be a point of the ambient space, and $s \in G_\mathcal{R}$. The projection of $\mat{x}$ and its symmetric $s(\mat{x})$ correspond to the same pose:
\begin{equation}
\proj(s(\mat{x})) = \proj(\mat{x}).
\end{equation}
\end{proposition}

\begin{proof}
This result can be easily verified in the case of a revolution object with rotoreflection symmetry or a 2D cyclic object. Therefore, we only discuss the case of an object with a finite proper symmetry group.

Let $\mat{x} \in \mathbb{R}^{12}$ be a point of the ambient space, and $s_{\mat{G}} \in G_\mathcal{R}$, where $\mat{G} \in G$. We split $\mat{x}$ into two parts $\mat{M} \in \mathcal{M}_{3, 3}(\mathbb{R})$ and $\mat{t} \in \mathbb{R}^3$ such as
\begin{equation}
\mat{x} = (\vect(\mat{M})^\top, \mat{t}^\top)^\top.
\end{equation}
The symmetric of $\mat{x}$ can thus by definition be written as
\begin{equation}
s_{\mat{G}}(\mat{x}) = (\vect(\mat{M} \mat{G})^\top, \mat{t}^\top)^\top.
\end{equation}

The projection of $\mat{x}$ onto the pose space consists in the pose $[ \hat{\mat{R}}, \mat{t} ]$, with $\hat{\mat{R}} = \mat{U} \mat{S} \mat{V}^\top$, considering a SVD decomposition $\mat{M} \matLambda = \mat{U} \mat{D} \mat{V}^\top$ and using the same conventions for $\mat{U}, \mat{V}, \mat{S}$ and $\mat{D}$ than in subsection~\ref{sec:projection_finite_proper_symmetry_group} in which we detailed this result.

Similarly, the projection of $s_{\mat{G}}(\mat{x})$ can be deduced from a SVD decomposition of $\mat{M} \mat{G} \matLambda$.
We know from lemma~\ref{lem:commutation_g_lambda} that this latter term can be rearranged into
\begin{equation}
\mat{M} \mat{G} \matLambda = \mat{M} \matLambda \mat{G}.
\end{equation}
Thus, injecting the previous decomposition into this expression enables us to exhibit a SVD decomposition of $\mat{M} \mat{G} \matLambda$:
\begin{equation}
\begin{aligned}
\mat{M} \mat{G} \matLambda &= \mat{U} \mat{D} \mat{V}^\top \mat{G} \\
&=\mat{U} \mat{D} \tilde{\mat{V}}^\top
\end{aligned}
\end{equation}
where $\tilde{\mat{V}} = \mat{G}^\top \mat{V}$.
Because $G$ is a rotation matrix,
\begin{equation}
\begin{aligned}
\det(\tilde{\mat{V}}) &= \det(\mat{G})  \det(\mat{V}) \\
&= \det(\mat{V})
\end{aligned}
\end{equation}
and the projection of $s_{\mat{G}}(\mat{x})$ is therefore
\begin{equation}
\begin{aligned}
\proj(s_{\mat{G}}(\mat{x})) &= [\mat{U} \mat{S} \tilde{\mat{V}}^\top, \mat{t}] \\
&=[\hat{\mat{R}} \mat{G}, \mat{t}].
\end{aligned}
\end{equation}
Since $\mat{G}$ is a proper symmetry of the object,
\begin{equation}
[\hat{\mat{R}} \mat{G}, \mat{t}] = [\hat{\mat{R}}, \mat{t}]
\end{equation}
which concludes the proof.
\qed
\end{proof}

Based on those properties, it is possible to propose an unambiguous estimation of the mean, as follows:
\myframedminipage{%
\begin{definition}[Mean estimation]
Given a consistent tuple $(\mat{p}_i)_{i=1 \dots n} \in \prod_i \mathcal{R}(\mathcal{P}_i)$ of representatives of the poses $S = \left\lbrace \mathcal{P}_i \right\rbrace_{i=1 \dots n}$, we define as estimation of the mean of those poses
\begin{equation}
\widehat{\mean}(S) \triangleq \proj \left( \cfrac{\sum_i w_i \mat{p}_i}{\sum_i w_i} \right).
\end{equation}
\end{definition}}

\begin{proof}
We show here that this expression is well-defined, \ie that it does depends on the consistent tuple of representatives considered. Let $(\mat{p}_i)_{i=1 \dots n} \in \prod_i \mathcal{R}(\mathcal{P}_i)$ be a consistent tuple of representatives. The consistent tuples are the tuples symmetric to this one (proposition~\ref{prop:uniqueness_up_to_symmetry_consistent_tuple}):
\begin{equation}
\left\lbrace (s(\mat{p}_i))_{i=1 \dots n} \middle\vert s \in G_\mathcal{R} \right\rbrace. 
\end{equation}

Let us therefore consider an arbitrary consistent tuple $(s(\mat{p}_i))_{i=1 \dots n}$, with $s \in  G_\mathcal{R}$ and show that it leads to the same estimation $\mathcal{M}$ of the average pose than an estimation performed with $(\mat{p}_i)_{i=1 \dots n}$.

By definition,
\begin{equation}
\mathcal{M} = \proj \left( \cfrac{\sum_i w_i s(\mat{p}_i)}{\sum_i w_i} \right).
\end{equation}
Because of the linearity of symmetries (proposition~\ref{prop:linearity_representative_symmetry}), the arithmetic mean of $(s(\mat{p}_i))_{i=1 \dots n}$ corresponds to the symmetric of the arithmetic mean of $(\mat{p}_i)_{i=1 \dots n}$:
\begin{equation}
\cfrac{\sum_i w_i s(\mat{p}_i)}{\sum_i w_i} = s \left( \cfrac{\sum_i w_i \mat{p}_i}{\sum_i w_i} \right),
\end{equation}
hence this expression of $\mathcal{M}$:
\begin{equation}
\mathcal{M} = \proj \left( s \left( \cfrac{\sum_i w_i \mat{p}_i}{\sum_i w_i} \right) \right).
\end{equation}
Invariance of the projection under symmetry of representatives (proposition~\ref{prop:invariance_of_projection_to_rep_sym}) enables to conclude this proof, since
\begin{equation}
\mathcal{M} = \proj \left( \cfrac{\sum_i w_i \mat{p}_i}{\sum_i w_i} \right).
\end{equation}
\qed
\end{proof}

Such estimation corresponds most likely to the actual mean~\eqref{eq:mean_definition}, unfortunately we do not have a proof of this conjecture.


\subsection{Sufficient conditions of consistency}

While the average of poses can be easily estimated given a consistent combination of representatives, there are however cases where no such combination exist, \eg when trying to average poses spread out over the set of orientations such as illustrated on figure~\ref{fig:averaging_poses}abc. In such a case, one might have to perform an exhaustive evaluation of the Fréchet variance for the different combinations in order to pick the one corresponding to the actual mean. Fortunately, this case is of limited practical interest, as the mean makes little sense.

Consistency is nonetheless not trivial to establish in the general case, and therefore in this section, we provide simple sufficient conditions for a combination of representatives to be consistent.
Consistency is in particular satisfied when the considered representatives are \emph{close enough} one an other, relatively to the distance between their representatives:
\myframedminipage{%
\begin{proposition}[Close-enough representatives]
\label{prop:condition_representative_combination_consistency}
Let $(\mat{p}_i)_{i=1 \dots n} \in \prod_i \mathcal{R}(\mathcal{P}_i)$ be a tuple of pose representatives. If its elements are closer one another than $T/2$, \ie if
\begin{equation}
\label{eq:condition_representative_combination_consistency}
\forall (i,j) \in \llbracket 1, n \rrbracket^2, \| \mat{p}_i - \mat{p}_j \| < T/2,
\end{equation}
then this tuple is consistent.
\end{proposition}}

\begin{proof}
Let us consider a tuple $(\mat{p}_i)_{i=1 \dots n} \in \prod_i \mathcal{R}(\mathcal{P}_i)$ that satisfies the condition~\eqref{eq:condition_representative_combination_consistency}. For any $(i,j) \in \llbracket 1, n \rrbracket^2$ and $\mat{q}_j \in \mathcal{R}(\mathcal{P}_j) \setminus \left\lbrace \mat{p}_j \right\rbrace$, the below properties hold true:
\begin{equation}
\!\begin{cases}
\|\mat{q}_j - \mat{p}_j \| \!\leq \!\| \mat{p}_j - \mat{p}_i \| \!+ \!\|\mat{q}_j - \mat{p}_i \| & \!\!\!\!\!\text{(triangle inequality)}\\
\|\mat{q}_j-\mat{p}_j \| \!\geq \!T & \!\!\!\!\!\text{(definition~\ref{def:definition_T})}\\
\| \mat{p}_j - \mat{p}_i \| \!< \!T/2.  & \!\!\!\!\!\text{(condition~\eqref{eq:condition_representative_combination_consistency})}\\
\end{cases}
\end{equation}
From those inequalities, we deduce that
\begin{equation}
\label{eq:consistency_proof}
\|\mat{p}_j-\mat{p}_i\| < \|\mat{q}_j-\mat{p}_i\|,
\end{equation}
hence the consistency of $(\mat{p}_i)_{i=1 \dots n}$. \qed
\end{proof}

A special case of practical value of this criterion is obtained when the considered representatives are included in a ball small enough. It is illustrated on figure~\ref{fig:averaging_poses}d, and is exploited in our application example section~\ref{sec:application_example}.
\myframedminipage{%
\begin{proposition}[Representatives within a ball]
\label{prop:condition_representative_combination_consistency_ball}
Let $(\mat{p}_i)_{i=1 \dots n} \in \prod_i \mathcal{R}(\mathcal{P}_i)$ be a tuple of pose representatives. If those representatives lie in a ball of radius $T/4$, \ie if
\begin{equation}
\exists \mat{c} \in \mathbb{R}^N / \forall i \in \llbracket 1, n \rrbracket^2, \| \mat{p}_i - \mat{c} \| < T/4,
\end{equation}
then this tuple is consistent.
\end{proposition}
}

\begin{proof}
A tuple satisfying this condition also satisfies the one of proposition~\ref{prop:condition_representative_combination_consistency} because of the triangle inequality, since for any $(i,j) \in \llbracket 1, n \rrbracket^2$,
\begin{equation}
\begin{aligned}
\| \mat{p}_i - \mat{p}_j \| &\leq \| \mat{p}_i - \mat{c} \| + \| \mat{p}_j - \mat{c} \| \\
&< T/4 + T/4.
\end{aligned}
\end{equation}
\qed
\end{proof}

These sufficient conditions can easily be generalized by considering that only the orientation parts of poses representatives have actually to be close enough one an other. This is a direct consequence to the fact that the pose space can be decomposed into a Cartesian product of a position and an orientation space, and that symmetry considerations only affect orientation for a bounded object.
 
\bibliographyInSubfile

\section{Local properties}
\label{sec:local_properties}

While we focus in this article on global metric properties, the proposed distance can be shown locally equivalent to a Riemannian metric over the pose space manifold. We therefore briefly discuss in this section those aspects.

\paragraph{Object with finite proper symmetry group}
In the case of a 3D object of finite proper symmetry group, the pose space can be seen as a manifold of dimension 6 (3 for translation, and 3 for rotation).
Let us consider two poses of such an object, and $\mat{T}_1$ and $\mat{T}_2 \in SE(3)$ two associated rigid transformations, such as
\begin{equation}
\mydist([\mat{T}_1], [\mat{T}_2]) = \mydistnosym(\mat{T}_1, \mat{T}_2),
\end{equation}
\ie such as $\mat{T}_1^{-1} \circ \mat{T}_2$ to be a shortest displacement from pose $[\mat{T}_1]$ to pose $[\mat{T}_2]$. 
If the angle of the relative rotation between $\mat{T}_1$ and $\mat{T}_2$ is small, the corresponding displacement of a point $\mat{x} \in \mathbb{R}^3$ of the object between those two poses can be approximated by introducing the displacement vector $\mat{v} \in \mathbb{R}^3$ and the rotation vector $\mat{\omega} \in \mathbb{R}^3$ between $\mat{T}_1$ and $\mat{T}_2$  as follows:
\begin{equation}
(\mat{T}_1^{-1} \circ \mat{T}_2) (\mat{x}) \underset{\theta \rightarrow 0}{\sim} \mat{x} + \omega \times \mat{x} + \mat{v}.
\end{equation}
The distance between the two poses can then be approximated by
\begin{equation}
\mydist([\mat{T}_1], [\mat{T}_2]) \underset{\theta \rightarrow 0}{\sim} \cfrac{1}{S} \int_{\mathcal{S}} \mu(\mat{x}) \| \omega \times \mat{x} + \mat{v} \|^2 ds.
\end{equation}
When considering an infinitesimal displacement between $\mat{T}_1$ and $\mat{T}_2$, and assimilating $\mat{v}$ and $\mat{\omega}$ to translational and angular velocities, this expression corresponds to the notion of kinetic energy (up to a factor $2/S$).

It is also a quadratic form, and therefore the proposed distance is locally equivalent to a Riemannian distance associated with the following metric tensor $g$, defined for any two tangent vectors $(\mat{v}_1^\top, \mat{\omega}_1^\top)^\top$, $(\mat{v}_2^\top, \mat{\omega}_2^\top)^\top \in \mathbb{R}^3 \times \mathbb{R}^3$ by
\begin{multline}
g \left( (\mat{v}_1^\top, \mat{\omega}_1^\top)^\top, (\mat{v}_2^\top, \mat{\omega}_2^\top)^\top \right) \\
\triangleq \cfrac{1}{S} \int_{\mathcal{S}} \mu(\mat{x}) ( \mat{\omega}_1 \times \mat{x} + \mat{v}_1 )^\top (\mat{\omega}_2 \times \mat{x} + \mat{v}_2) ds.
\end{multline}
This Riemannian metric was already described in the litterature~\citep{zefran1996-a, lin2000}, and \citet{belta2002} notably suggested its use for interpolation on SE(3).
If the covariance matrix considered for the object is isotropic  -- \ie $\matLambda = \lambda \mat{I}$, with $\lambda \in \mathbb{R}^{+*}$, which is the case \eg for an object of spherical or cubic shape --, the proposed distance is moreover locally equivalent to the usual Riemannian distance~\eqref{eq:distance_classical_riemannian} over $SE(3)$.

\paragraph{Other objects classes}
Similarly, the pose space of a 3D revolution object can be seen as a 5D manifold. The proposed distance for such an object is indeed locally equivalent to a Riemannian distance over $\mathbb{S}^2 \times \mathbb{R}^3$ induced by embedding $\mathbb{S}^2$ as a sphere in a 3D Euclidean space, and considering the usual Euclidean distance regarding the translation part.

Poses of a 2D object with a finite proper symmetry group lie likewise on a 3D manifold, and the proposed distance is locally equivalent to a Riemannian distance over $\mathbb{S}^1 \times \mathbb{R}^2$  induced by embedding $\mathbb{S}^1$ as a circle in a 2D Euclidean space, \ie by considering as distance between two infinitesimally close orientations the angle of the relative rotation between those two (up to a scaling factor).

Finally, pose spaces of spherical 3D objects and circular 2D ones are respectively equivalent to $\mathbb{R}^3$ and $\mathbb{R}^2$, associated with the Euclidean distance.

\paragraph{Remarks} Despite this local equivalence, the proposed pose spaces can be topologically different from the manifolds evoked above, because of the discrete symmetries of the object. As an example, the proposed pose space for a revolution object with rotoreflection invariance is actually homeomorphic with $\mathbb{RP}^2 \times \mathbb{R}^3$, where $\mathbb{RP}^2$ is the real projective plane -- \ie a sphere with antipodal points associated -- instead of $\mathbb{S}^2 \times \mathbb{R}^3$.

Moreover, except for 3D spherical or 2D circular objects, the  distance~\eqref{eq:distance} is globally different from these Riemannian metrics. Compared to the proposed distance, those latter have several drawbacks.

They are indeed more expensive to estimate since they include trigonometric computations (\eg for distance~\eqref{eq:distance_classical_riemannian}). There is even no known closed-form expression for such distance in the case of an arbitrary 3D object of finite proper symmetry group. Moreover, they do not benefit from the same nice computational properties than the proposed distance regarding the problem of pose averaging, for which they may require iterative approaches~\citep{pennec1998}.
But more fundamentally, and as discussed in the introduction, our distance was proposed to quantify the \emph{similarity} between poses at a global scale. The notion of \emph{motion} between two poses, expressed in a Riemannian distance, is therefore irrelevant to the kind of applications we are interested in.

\bibliographyInSubfile

\section{Application example}
\label{sec:application_example}

In this section, we illustrate the use of our metric on the problem of rigid object instances detection and pose estimation. Given an input depth image of a scene containing potentially multiple instances of a rigid object, our goal consists in recovering the poses of these instances. We perform experiments with three different objects of different symmetry classes among those shown in table~\ref{tab:symmetries_classes} to illustrate the versatility of our approach:
\begin{itemize}
\item the Stanford bunny -- an object without proper symmetry.
\item a candlestick --  considered as a revolution object without rotoreflection invariance.
\item a cartoon-like space rocket -- which is invariant by rotation of 120\textdegree{} along its axis.
\end{itemize}
For practical reasons, our example is based on synthetic 3D data, depicted in figure~\ref{fig:pose_recovery}b. It was produced using an off-the-shelf stereo matching algorithm, on a pair of images of a virtual scene lit by a pseudo-random pattern projector (figure~\ref{fig:pose_recovery}a). Those were synthesized with the rendering engine Blender Cycles~\citep{blender}.
The reader is referred to the work of \citet{Bregier_2017_ICCV_Workshops} for a more quantitative analysis of the interest of the proposed distance for pose estimation.

\subsection{Mean Shift for pose recovery}
\label{sec:mean_shift}

\begin{figure*}
\centering
\newlength{\myfigwidth}
\setlength{\myfigwidth}{4cm}

\newlength{\myfigheight}
\setlength{\myfigheight}{3cm}

\begin{minipage}[t]{8cm}
\centering
\includegraphics[height=\myfigheight]{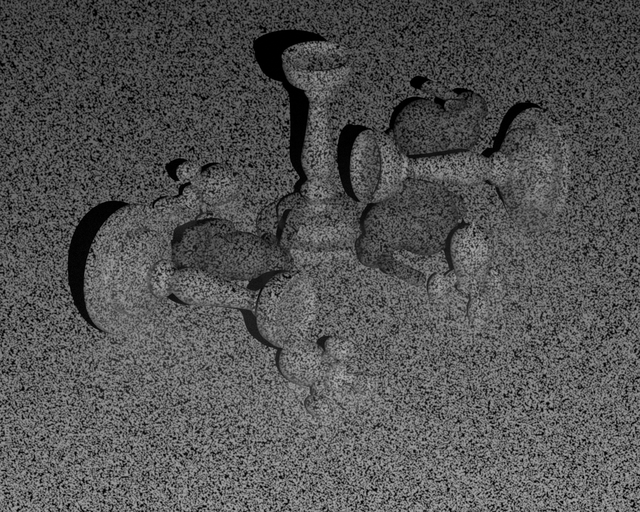}
\includegraphics[height=\myfigheight]{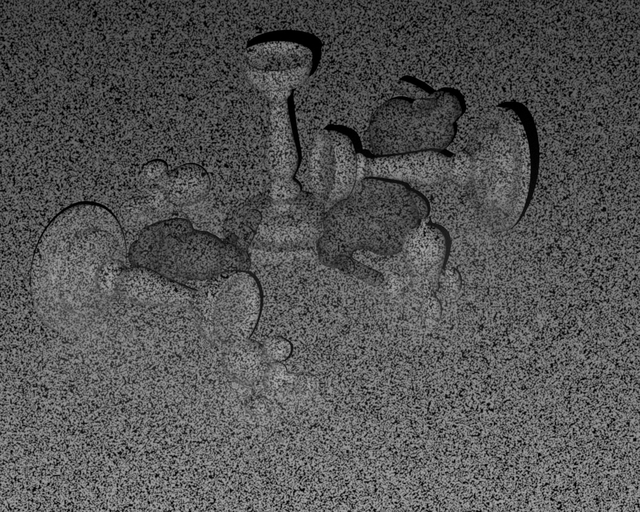}\\
(a) Stereo pair used for 3D reconstruction.
\end{minipage}
\begin{minipage}[t]{8cm}
\centering
\includegraphics[height=\myfigheight]{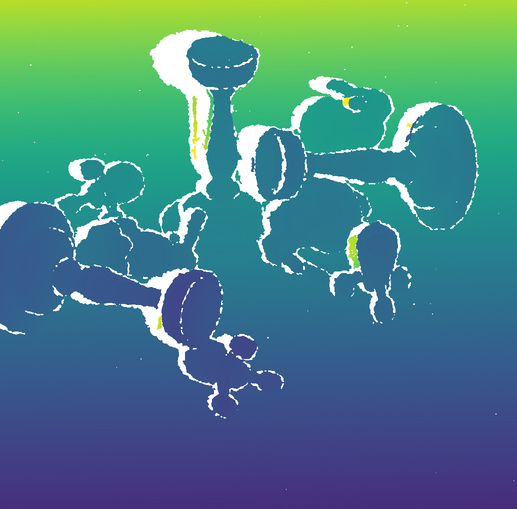}
\includegraphics[height=\myfigheight]{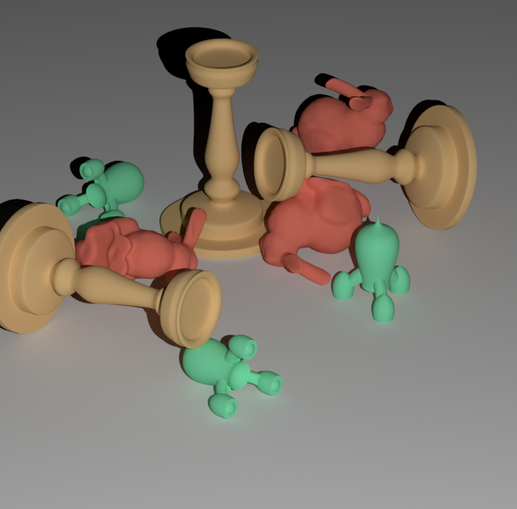}\\
(b) Reconstructed 3D range data (RGB channel solely for visualization purposes).
\end{minipage}

\includegraphics[width=\myfigwidth]{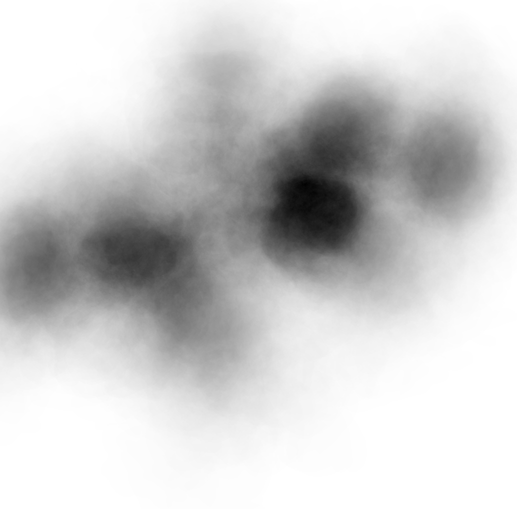}
\includegraphics[width=\myfigwidth]{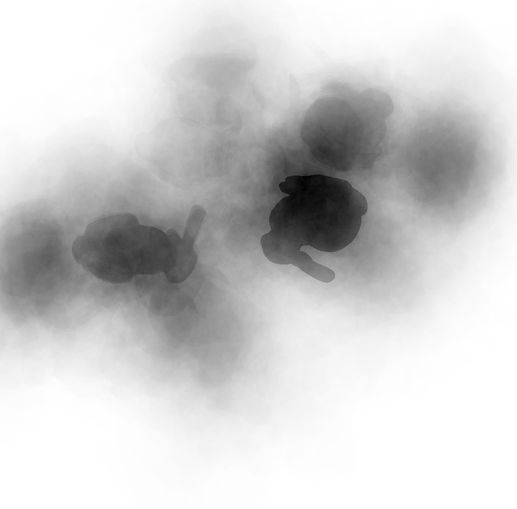}
\includegraphics[width=\myfigwidth]{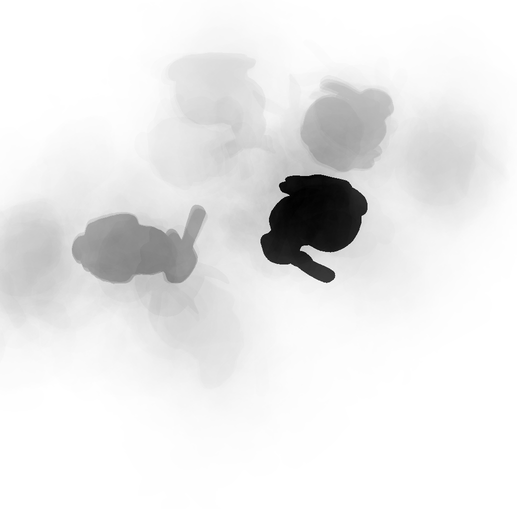}
\includegraphics[width=\myfigwidth]{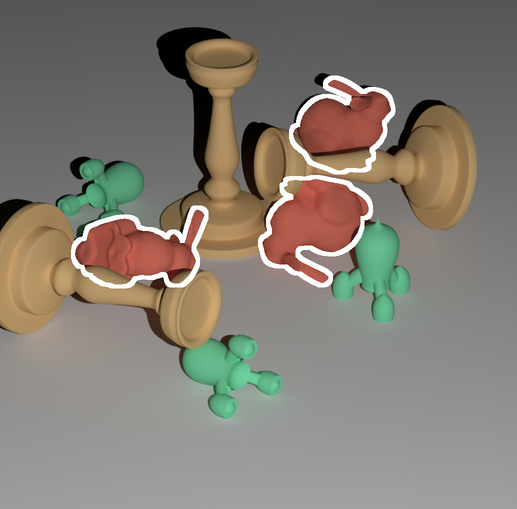}

\includegraphics[width=\myfigwidth]{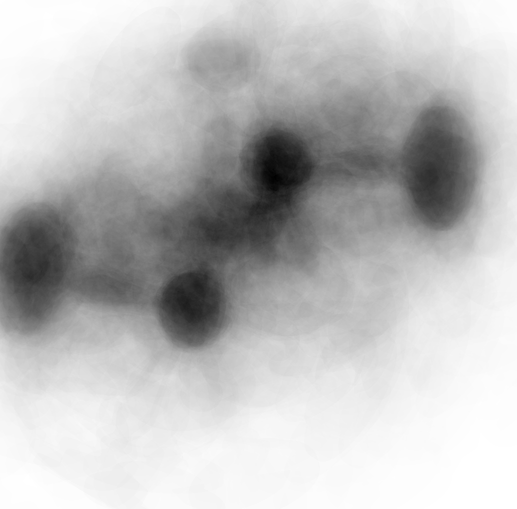}
\includegraphics[width=\myfigwidth]{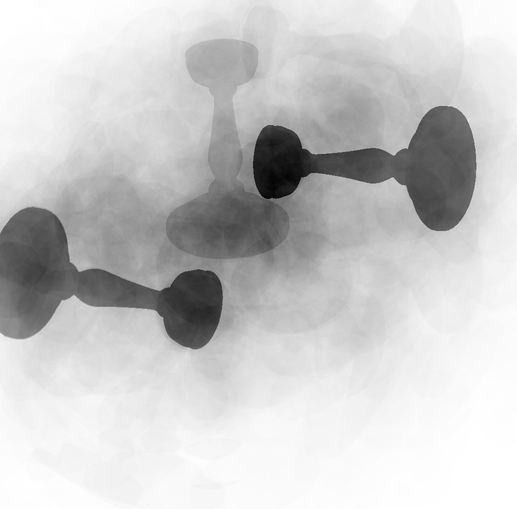}
\includegraphics[width=\myfigwidth]{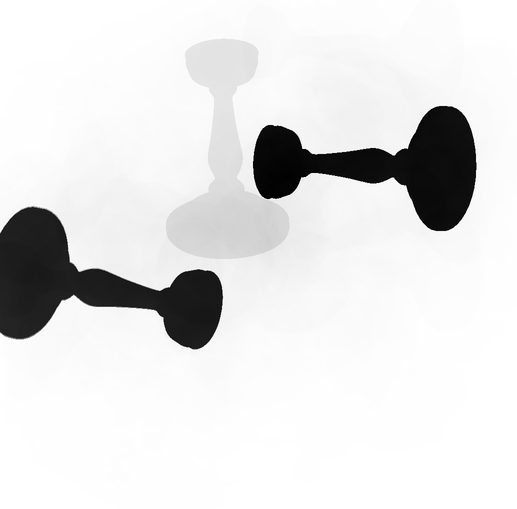}
\includegraphics[width=\myfigwidth]{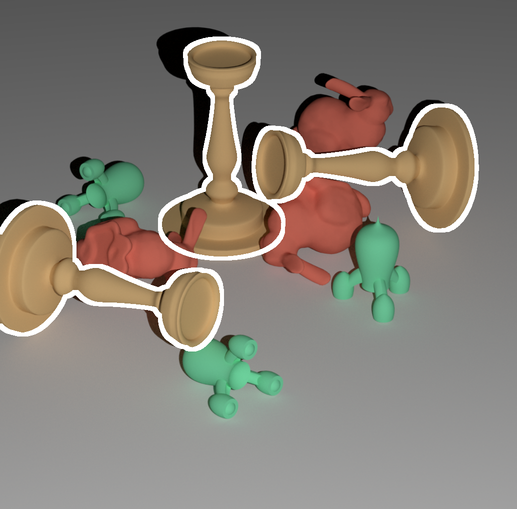}

\includegraphics[width=\myfigwidth]{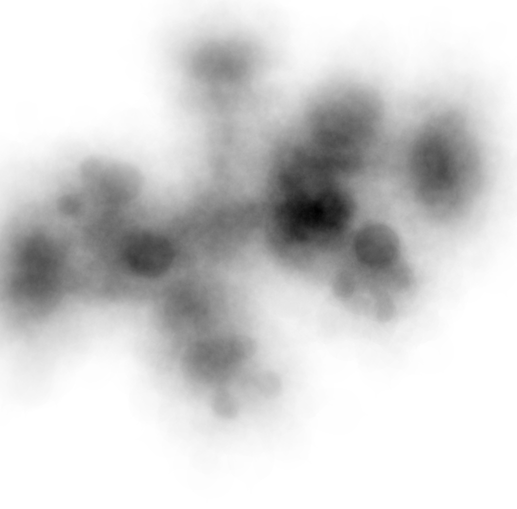}
\includegraphics[width=\myfigwidth]{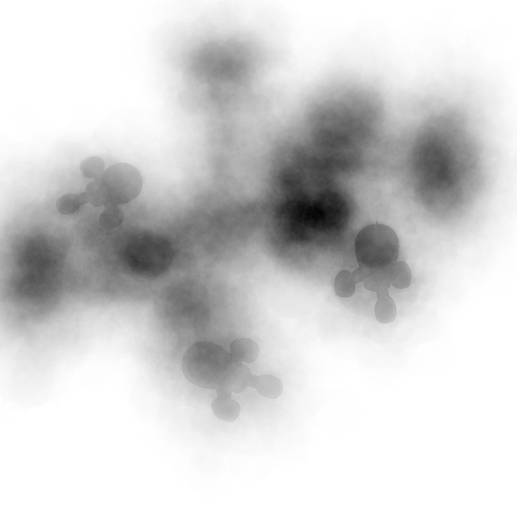}
\includegraphics[width=\myfigwidth]{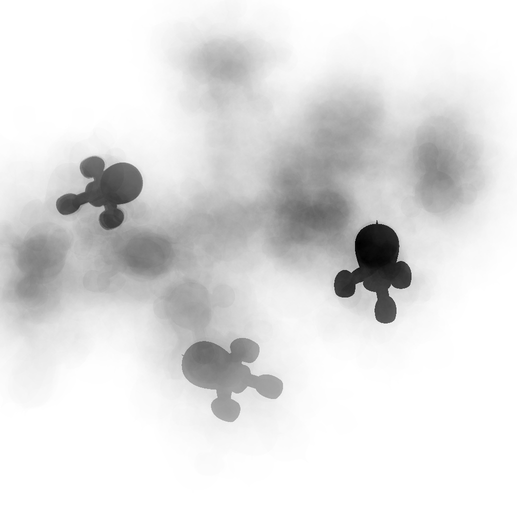}
\includegraphics[width=\myfigwidth]{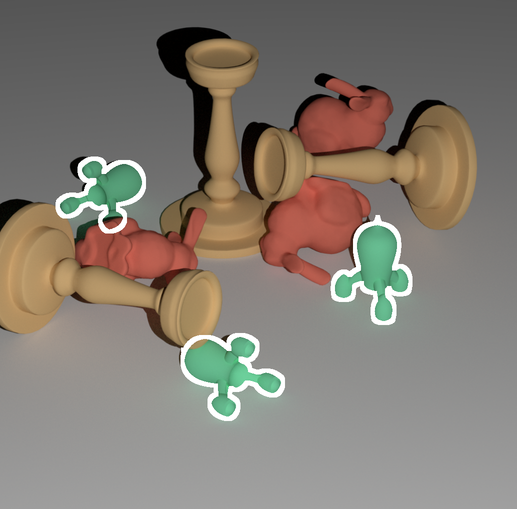}

(c) From left to right:
Pose distribution generated from the input 3D data using the method of~\citet{drost2010}.
Shifted poses using Mean Shift.
Shifted poses weighted by the density of the initial distribution.
Recovered poses of object instances. A pose distribution is represented by accumulating the 2D silhouettes on the image plane of the object at its different poses (the more silhouettes a pixel belongs to, the darker it is). The process is performed independently for the three different objects.

\caption{\label{fig:pose_recovery}Application example: object instances detection and pose estimation via a Mean Shift procedure to extract the main modes of an initial pose distribution. Illustration with three different objects of various symmetry properties.}
\end{figure*}

Among the existing work adapted for a depth image input, some popular approaches \citep{drost2010,fanelli2011,tejani2014} process in a bottom-up approach, generating votes for poses candidates in a Hough-like manner and, identifying those votes to the sampling of a pose distribution, look for its main modes which hopefully correspond to the actual poses of object instances. We place ourselves within such a modes-seeking framework.

Modes detection in a distribution on the pose space is not an easy problem. Grid-based accumulation techniques traditionally used in Hough-like methods are unpractical due to the high dimension of the pose space, except through the use of a sparse structure \citep{rodrigues2012} or solely as a preprocessing technique used on a few dimensions \citep{drost2010,rodrigues2012,tejani2014}. A popular approach for modes detection more adapted to high dimensional problems is Mean-Shift, a local and non parametric iterative method based on a kernel density estimation of the probability distribution. Unfortunately, this method is designed for vector spaces, which the pose space is not.
\citet{fanelli2011} and \citet{tejani2014} used nonetheless Mean Shift with a global parametrization of the pose space, but such approach suffers from the intrinsic drawbacks we evoked in the introduction.
\citet{tuzel2005} and later \citet{subbarao2006} proposed versions of Mean Shift for Lie groups and Riemannian manifolds that might circumvent those issues, but their approach is computationally expensive as it requires at each iteration to map the samples to the local tangent plane of the point to shift, compute the shift vector through the classical Mean Shift procedure and map it back to get the updated point. 

In this example, we show how standard Mean Shift algorithm can be adapted to perform modes detection on the pose space quite efficiently through the use of our distance, even for objects showing proper symmetries properties.

Given an input depth map, we generate a set of votes for object poses $\left\lbrace \mathcal{P}_i \right\rbrace_{i=1,\ldots,n}$ thanks to our own implementation of the method of \citet{drost2010}. It is based on a local aggregation scheme, and is performed by matching geometric features extracted from the input data with those extracted from a model of the object.
We use a sampling rate of $\tau_d = 0.025$ and consider every samples as reference points -- the interested reader is referred to the original description of~\citet{drost2010} regarding the meaning of those parameters.
This initial pose distribution is quite spread, as can be observed from the \emph{blurred} effect of the representation figure~\ref{fig:pose_recovery}c, column~1.

We then consider each of those votes as a starting pose for Mean Shift. A usual practice to speed up drastically computations when seeking modes is to consider only a subset taken from the votes as starting points \eg by random sampling, but we do not use such approach here to avoid the introduction of additional parameters.
For each of the poses to shift, we process iteratively following the usual Mean Shift procedure. Considering a flat Mean Shift kernel of radius $r$, we find the poses within the set of votes that are within a radius $r$ of the current pose, compute their mean, shift the current pose to this mean and repeat until convergence.

We choose arbitrarily the radius $r$ of our kernel to correspond to 1.5 times the smallest eigenvalue of the matrix $\matLambda$ for the bunny and the candlestick -- that is roughly 75\% of the smallest typical dimension of the object. 

For the rocket, we use a smaller radius of $\sqrt{3}/2$ times this eigenvalue, which is the greatest value that satisfies the condition of proposition~\ref{prop:condition_representative_combination_consistency_ball} (see appendix~\ref{ap:T_value_application_example}). A bigger radius value may also experimentally give good results, but does not provide the same theoretical guarantees.

Indeed, such choices of radii satisfy to the condition $r < T/4$, where $T$ is the minimum distance between representatives of the same pose (see definition~\ref{def:definition_T}). We know from section~\ref{sec:proximity_query} that given a representative $\mat{p}$ of the pose $\mathcal{P}$ to shift, poses closer than $r$ from $\mathcal{P}$ are the poses who have one representative  within a ball of radius $r$ around $\mat{p}$. These representatives can be retrieved efficiently through an off-the-shelf \emph{radius search} method. Moreover, because $r$ is chosen strictly smaller than $T/2$, the representatives retrieved by such query necessarily correspond to different poses and there are therefore no duplicates.
Furthermore, these representatives lie in a ball of radius $T/4$, therefore we have the insurance from proposition~\ref{prop:condition_representative_combination_consistency_ball} that we can unambiguously estimate the average of the corresponding poses as the projection on the pose space of the arithmetic mean of these representatives.

As a consequence, adapting the Mean Shift procedure to our pose space given the chosen radius only requires an additional step compared to the usual procedure in a vector space. This step consists in projecting the arithmetic mean of the retrieved representatives on the pose space, which can be performed as described in section~\ref{sec:projection}. Pseudo-code of the adapted Mean Shift algorithm is proposed in algorithm~\ref{al:mean_shift_algorithm}.
\begin{algorithm}
\small\sf\centering
\caption{Mean Shift algorithm within the pose space\label{al:mean_shift_algorithm}}
\begin{algorithmic}[1]
\REQUIRE{$\lbrace \mathcal{P}_i \rbrace_{i=1,\dots,n}$ a set of poses,  \\
$\mat{p}_{\text{in}}$ a representative of the pose to shift, 
\\ $r$ the Mean Shift radius.}
\ENSURE{A representative of the shifted pose.}

\STATE {$R \gets GetRepresentatives( \lbrace \mathcal{P}_i \rbrace_{i=1,\dots,n})$}

\textcolor{blue}{
\# Preprocessing step independent of $\mat{p}_{\text{in}}$. \\
\# $R$ contains all poses representatives: \\ 
\# $\forall i \in  \llbracket 1, n \rrbracket, \left\lbrace R[i,j] \right\rbrace_{j=1, \dots, |\mathcal{R}(\bullet)|} = \mathcal{R}(\mathcal{P}_i)$.
}
\STATE{$\mat{p} \gets \mat{p}_{in}$} 
\REPEAT 
\STATE{$\mat{p}_{\text{old}} \gets \mat{p}$}
\STATE{$\mathcal{N} \gets \text{RadiusSearch}(R, \mat{p}, r)$}
\IF {$\mathcal{N} \neq \emptyset$}
	\STATE {$\mat{m} \gets \left( \sum_{(i,j) \in \mathcal{N}} \mat{R}[i,j] \right) /  |\mathcal{N}|$}
	\STATE {$\mat{p} \gets \text{ClosestRepresentative}(\proj(\mat{m}))$}
\ENDIF
\UNTIL{$\mat{p} \neq \mat{p}_{\text{old}}$}
\RETURN {$\mat{p}$}
\end{algorithmic}
\end{algorithm}

The projection of the mean at each iteration is actually not required for the poses to shift towards meaningful modes in practice, and therefore we perform it only once after convergence. The pose distribution obtained after Mean Shift is \emph{sharper} than the original one, as can be seen on figure~\ref{fig:pose_recovery}c column 2 where the silhouettes of object instances emerge. 

We then estimate the probability density (up to a scaling factor) at a mode $\mathcal{M}$  by kernel density estimation:
\begin{equation}
s(\mathcal{M}) = \sum_i H \left( \cfrac{\mydist(\mathcal{M}, \mathcal{P}_i)}{r} \right)
\end{equation}
with $H$ the Epanechnikov kernel to which is associated the flat Mean Shift kernel \citep{fukunaga1975}:
\begin{equation}
H(d) = \left\lbrace 
\begin{array}{lc}
\cfrac{3}{4}(1-d^2) & \text{if } |d| \leq 1\\
0 & \text{otherwise.}
\end{array} \right.
\end{equation}

The most significant modes based on this estimate can then be extracted. Those poses are assumed to be good pose hypotheses for the object instances of the scene, and typically stand up from the weighted distribution (figure~\ref{fig:pose_recovery}c column 3).
We refine them further through \eg the ICP procedure \citep{besl1992}, and filter them by checking their consistency with the actual data in order to avoid false postives, to hopefully retrieve the poses of object instances in the 3D scene (figure~\ref{fig:pose_recovery}c column 4).


\paragraph{Theoretical limitation}
The probabilistic interpretation used here is abusive and should only be considered as a way to give the intuition of the Mean Shift approach. Kernel density estimation over a Riemannian manifold has been mathematically studied by \citet{pelletier2005}, but our approach does not enter into such framework as we do not consider a Riemannian distance. 
Some theoretical results might nonetheless be obtained, since our metric is equivalent to a Riemannian metric for small Mean Shift radii (see section~\ref{sec:local_properties}).
Such considerations however, are out of the scope of this work, and $s(\mathcal{M})$ can simply be considered as a score for the pose $\mathcal{M}$.


\subsection{Comparison with a $SE(3)$ metric}
\label{subsec:experimental_comparison_se3}

We compare these experimental results with those obtained using a more usual distance adapted to $SE(3)$
\begin{equation}
\label{eq:distance_se3_experimental_comparison}
\mydist(\mat{T}_1, \mat{T}_2) = \sqrt{\| \mat{t}_2 - \mat{t}_1 \|^2 + r^2 \| \mat{R}_2 - \mat{R}_1 \|^2}.
\end{equation}
We chose this particular distance because the Mean Shift approach depends on the ability to average multiple poses, and a Frobenius norm over the rotation space is quite suited for this task~\citep{curtis1993}.
To limit the comparison bias, we choose the scaling factor $r$ between the rotation and translation parts to be
\begin{equation}
r=\sqrt{\cfrac{\lambda_1^2+\lambda_2^2+\lambda_3^2}{3}},
\end{equation}
where $\lambda_1 \leq \lambda_2 \leq \lambda_3$ are the eigenvalues of $\matLambda$.
This choice is indeed consistent with our proposed metric,  in that the rotational part of the distance between two poses of an object without proper symmetry corresponds respectively to $2 \sqrt{2} r \sin(\theta/2)$ for distance~\eqref{eq:distance_se3_experimental_comparison}, and $2 \sqrt{I_{\mat{k}}} \sin(\theta/2)$ for the proposed one, where $\theta$ is the angle of the relative rotation between the two poses, and $I_{\mat{k}}$ is the inertia moment of the corresponding axis $\mat{k}$ (see section~\ref{subsec:rotation_anisotropy}). Considering a typical value of  $2/3(\lambda_1^2+\lambda_2^2+\lambda_3^2)$ for $I_{\mat{k}}$ enables to identify those two terms.

As illustrated on figure~\ref{fig:comparison_our_distance_with_se3}, we do not observe much differences between the two approaches for the \emph{bunny} object. This is actually not surprising because the two distances are in this case very similar, since the \emph{bunny} is not symmetric, and has a limited anisotropy.

However, the benefit of our metric appears for the \emph{candlestick} and the \emph{rocket}, that both are symmetric.
Initial votes for poses are indeed spread out over the space of rigid transformations $SE(3)$, and considering the SE(3) distance~\eqref{eq:distance_se3_experimental_comparison} therefore leads to the detection of multiple modes corresponding to the same instance, because of the symmetries. One would have to filter out these duplicated poses hypotheses prior to any practical application, and because of these, it is required to check numerous modes to find every instance of the scene. In our example (figure~\ref{fig:comparison_our_distance_with_se3}b) we had to test up to respectively the 4th and 8th mode to recover the poses of the 3 rockets and candlesticks present in the scene.

On the other hand, the proposed distance enables to account for the proper symmetries of the object, and thus to better exploit the information contained in the initial set of votes than the $SE(3)$ distance. In our example, the 3 first modes extracted indeed correspond to the 3 actual instances for each object, without any duplicates. Moreover, these modes have more support from the initial set of votes and therefore stand out more clearly from the noise, which is important for the robustness of the method. The pose distribution obtained after Mean Shift is indeed visually sharper (figure~\ref{fig:comparison_our_distance_with_se3}, left), and spurious modes  for the \emph{rocket} and the \emph{candlestick} have a score below respectively 59\% and 66\% of the ones of the modes corresponding to actual object instances, compared to 89\% and 98\% when using the distance~\eqref{eq:distance_se3_experimental_comparison}.

\begin{figure*}
\centering

\colorlet{color_good}{color1}

\colorlet{color_bad}{color2}


\colorlet{color_duplicate}{color3}

\newcommand{\drawbar}[5]{ 
\draw [#1, fill=#1] ({#2 * \xmodeoffset + \baroffset}, 0) rectangle ({(#2 + 1) * \xmodeoffset - \baroffset}, {#4 * \barscale});
\begin{scope}[shift={({(#2 + 0.5) * \xmodeoffset}, {#4 * \barscale})}];
\draw (0, \ymodelabeloffset) node {$\pgfmathprintnumber[fixed,precision=2]{#4}$};
\draw {(0, \ymodeimgoffset)} node [above, draw=black, very thin, inner sep=1pt] {\includegraphics[width=\modewidth]{#5}};
\end{scope};
\draw ({(#2 + 0.5) * \xmodeoffset},0) node [below] {#3};
}

\newcommand{\drawbargood}[4]{
\drawbar{color_good}{#1}{#2}{#3}{#4};
}

\newcommand{\drawbarduplicate}[4]{
\drawbar{color_duplicate}{#1}{#2}{#3}{#4};
\begin{scope}[shift={({(#1 + 0.5) * \xmodeoffset}, {#3 * \barscale + \ymodeimgoffset})}];
\draw [color=red, thick] (-0.4\modewidth, -4pt) -- (0.4\modewidth, {\modeheight + 7pt});
\end{scope};
}

\newcommand{\drawbarbad}[4]{
\drawbar{color_bad}{#1}{#2}{#3}{#4};
\begin{scope}[shift={({(#1 + 0.5) * \xmodeoffset}, {#3 * \barscale + \ymodeimgoffset})}];
\draw [color=red, thick] (-0.4\modewidth, -4pt) -- (0.4\modewidth, {\modeheight + 7pt});
\draw [color=red, thick] (0.4\modewidth, -4pt) -- (-0.4\modewidth, {\modeheight + 7pt});

\end{scope};
}


\begin{tikzpicture}
\newlength{\imgwidth};
\setlength{\imgwidth}{2.8cm};
\def\xmodesaggregate{-10cm};

\def\yoffset{3.2cm};

\def\xmodeoffset{2.2cm};
\newlength{\modewidth};
\setlength{\modewidth}{1.7cm};
\def\ymodeimgoffset{0.4cm};
\def\ymodelabeloffset{0.2cm};
\pgfmathsetmacro\modeheight{{1018/1034*\modewidth}};

\def\barwidth{0.5cm};
\def\barscale{0.6cm};
\pgfmathsetmacro\baroffset{{(\xmodeoffset - \barwidth)/2}};


\begin{scope}[shift={(0, {-0 * \yoffset})}];
\begin{scope}[shift={(0, {-0 * \yoffset})}];
\draw {(\xmodesaggregate, 0)} node [anchor=south west,inner sep=0] {\includegraphics[width=\imgwidth]{z_img_example_bunny_clusters.png}};
\begin{scope}[shift={(-3 * \xmodeoffset,0)}];
\def\wzero{662.538};
\def\wone{343.504};
\def\wtwo{177.622};
\def\wthree{162.079};
\def\wfour{148.288};
\def\wfive{136.345};
\pgfmathsetmacro\sone{{\wone / \wzero}};
\pgfmathsetmacro\stwo{{\wtwo / \wzero}};
\pgfmathsetmacro\sthree{{\wthree / \wzero}};
\pgfmathsetmacro\sfour{{\wfour / \wzero}};
\pgfmathsetmacro\sfive{{\wfive/ \wzero}};
\draw (0, {0.6 * \barscale}) node [rotate=90, align = center] {relative\\score};
\drawbargood{0}{1}{1}{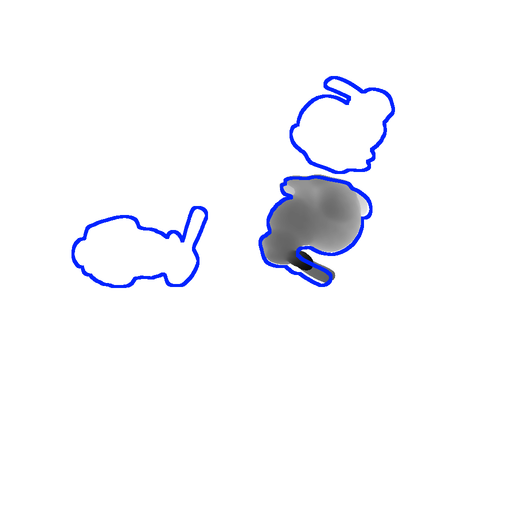};
\drawbargood{1}{2}{\sone}{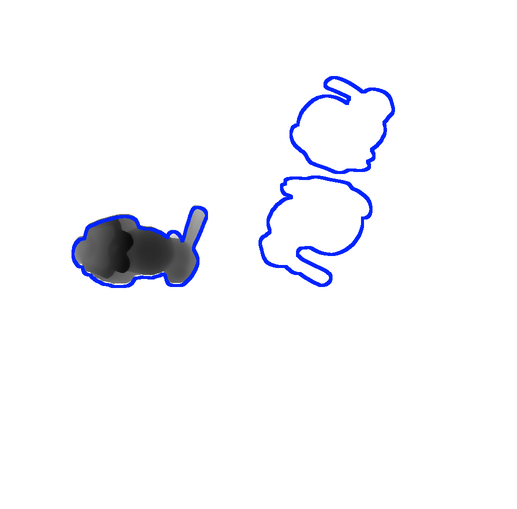};
\drawbargood{2}{3}{\stwo}{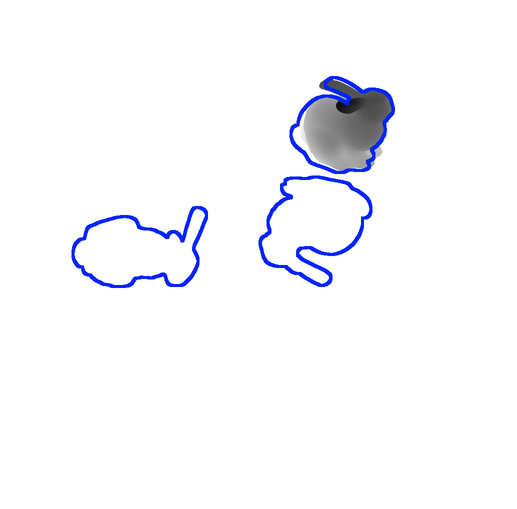};
\drawbarbad{3}{4}{\sthree}{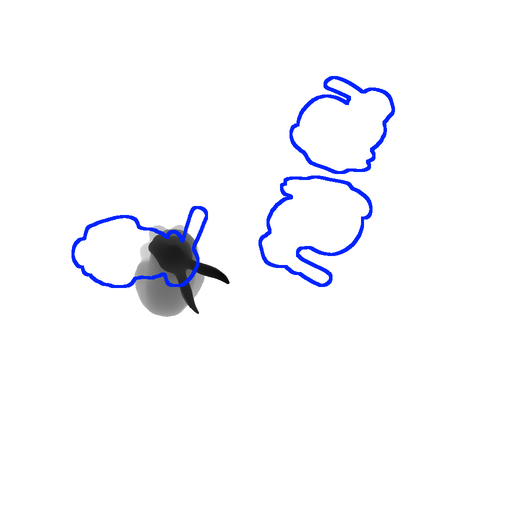};
\drawbarbad{4}{5}{\sfour}{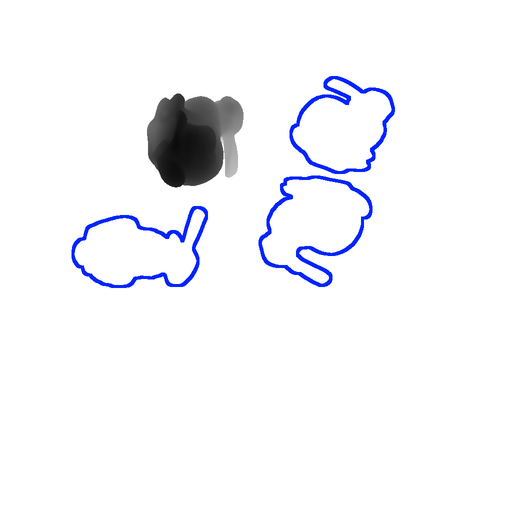};
\drawbarbad{5}{6}{\sfive}{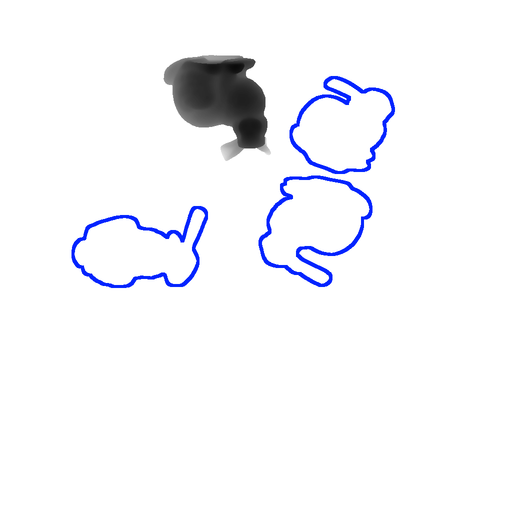};
\draw (0, 0) -- ({6 * \xmodeoffset}, 0);
\end{scope};
\end{scope};

\begin{scope}[shift={(0, {-1 * \yoffset})}];
\draw {(\xmodesaggregate, 0)} node [anchor=south west,inner sep=0] {\includegraphics[width=\imgwidth]{z_img_example_rocket_clusters.png}};
\begin{scope}[shift={(-3 * \xmodeoffset,0)}];
\def\wzero{361.606};
\def\wone{222.617};
\def\wtwo{213.099};
\def\wthree{126.49};
\def\wfour{105.275};
\def\wfive{91.9979};
\pgfmathsetmacro\sone{{\wone / \wzero}};
\pgfmathsetmacro\stwo{{\wtwo / \wzero}};
\pgfmathsetmacro\sthree{{\wthree / \wzero}};
\pgfmathsetmacro\sfour{{\wfour / \wzero}};
\pgfmathsetmacro\sfive{{\wfive/ \wzero}};
\draw (0, {0.6 * \barscale}) node [rotate=90, align = center] {relative\\score};
\drawbargood{0}{1}{1}{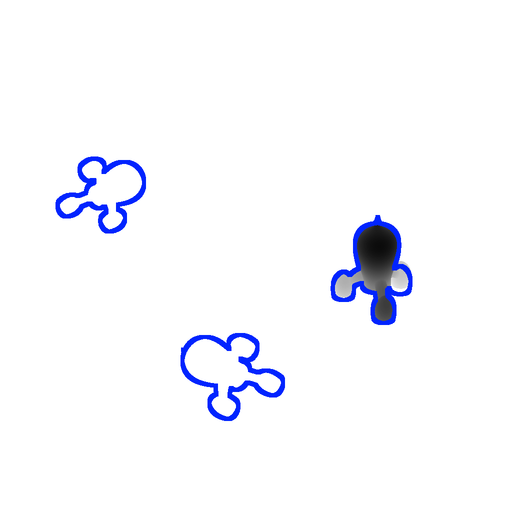};
\drawbargood{1}{2}{\sone}{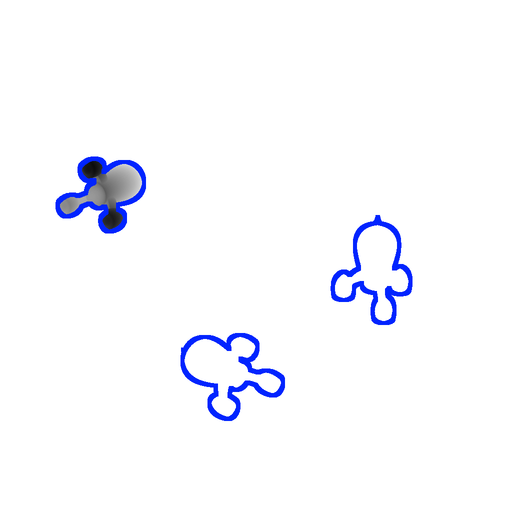};
\drawbargood{2}{3}{\stwo}{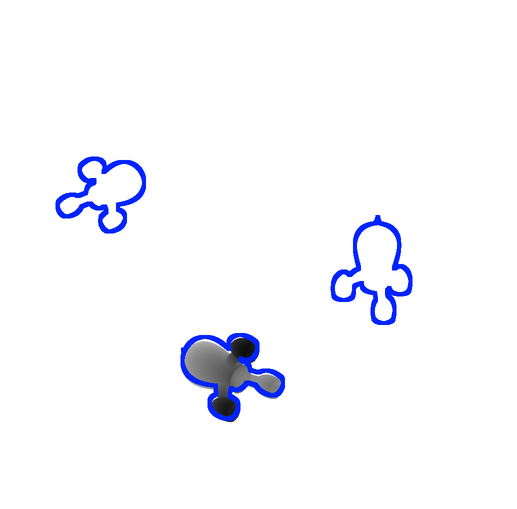};
\drawbarbad{3}{4}{\sthree}{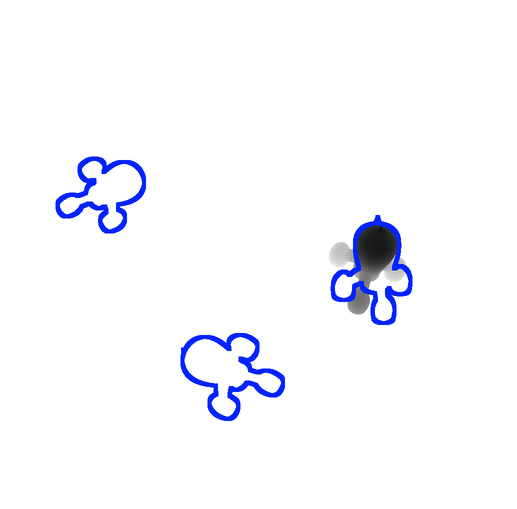};
\drawbarbad{4}{5}{\sfour}{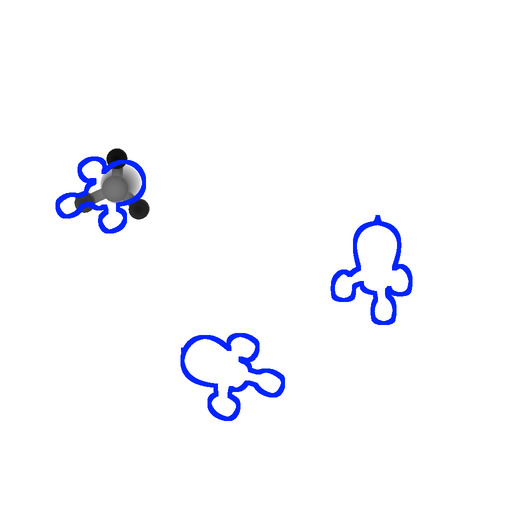};
\drawbarbad{5}{6}{\sfive}{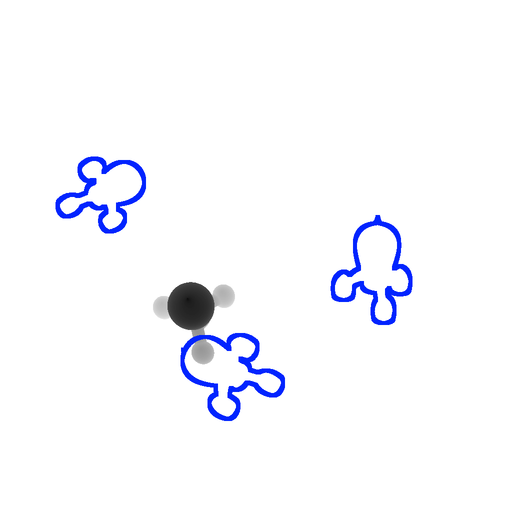};
\draw (0, 0) -- ({6 * \xmodeoffset}, 0);
\end{scope};
\end{scope};

\begin{scope}[shift={(0, {-2 * \yoffset})}];
\draw (0, -8pt) node [below, align = center] {(a) First modes of the pose distribution retrieved using the proposed metric.};
\draw {(\xmodesaggregate, 0)} node [anchor=south west,inner sep=0] {\includegraphics[width=\imgwidth]{z_img_example_candlestick_clusters.png}};
\begin{scope}[shift={(-3 * \xmodeoffset,0)}];
\def\wzero{1988.07};
\def\wone{1788.58};
\def\wtwo{566.989};
\def\wthree{368.171};
\def\wfour{306.798};
\def\wfive{302.319};
\pgfmathsetmacro\sone{{\wone / \wzero}};
\pgfmathsetmacro\stwo{{\wtwo / \wzero}};
\pgfmathsetmacro\sthree{{\wthree / \wzero}};
\pgfmathsetmacro\sfour{{\wfour / \wzero}};
\pgfmathsetmacro\sfive{{\wfive/ \wzero}};
\draw (0, {0.6 * \barscale}) node [rotate=90, align = center] {relative\\score};
\drawbargood{0}{1}{1}{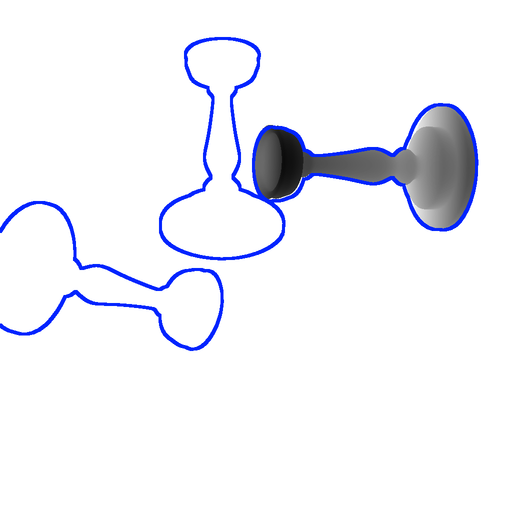};
\drawbargood{1}{2}{\sone}{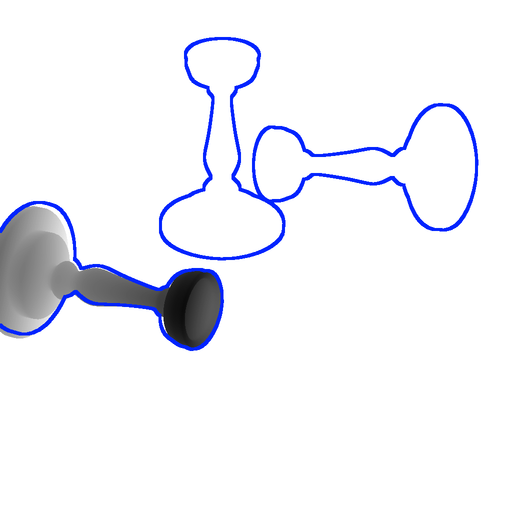};
\drawbargood{2}{3}{\stwo}{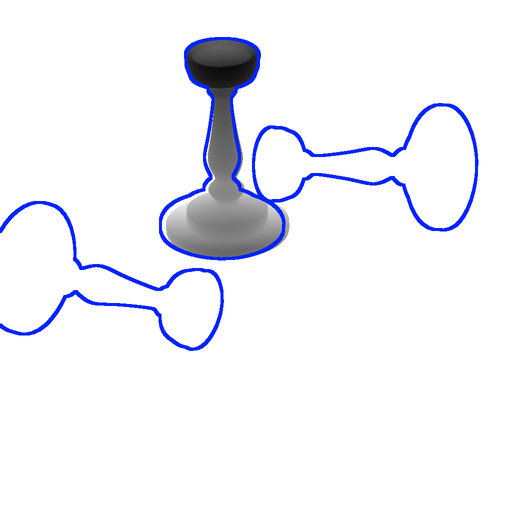};
\drawbarbad{3}{4}{\sthree}{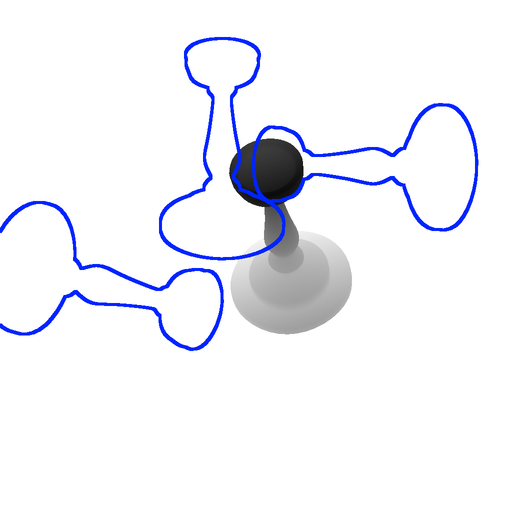};
\drawbarbad{4}{5}{\sfour}{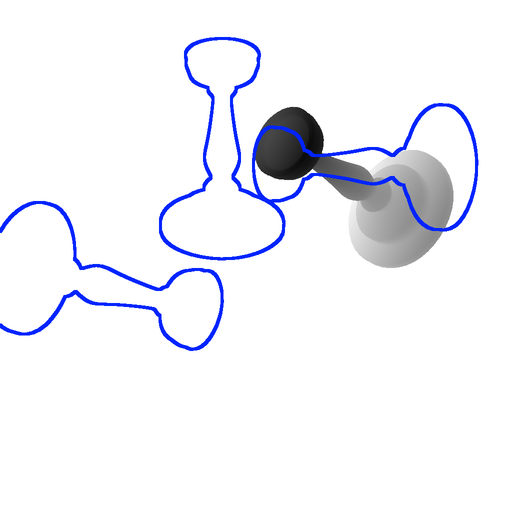};
\drawbarbad{5}{6}{\sfive}{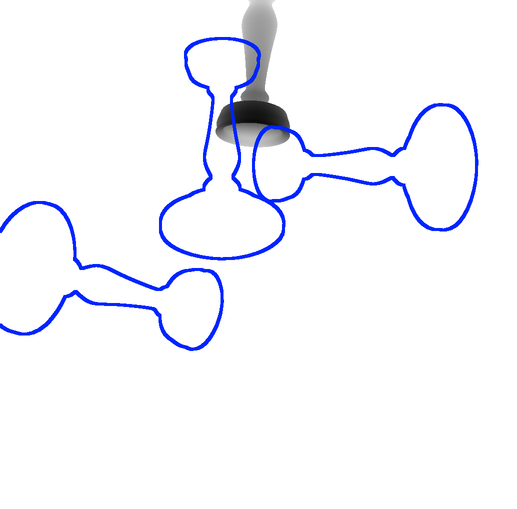};
\draw (0, 0) -- ({6 * \xmodeoffset}, 0);
\end{scope};
\end{scope};

\end{scope};


\begin{scope}[shift={(0, {-3 * \yoffset - 0.5cm})}];
\begin{scope}[shift={(0, {-0 * \yoffset})}];
\draw {(\xmodesaggregate, 0)} node [anchor=south west,inner sep=0] {\includegraphics[width=\imgwidth]{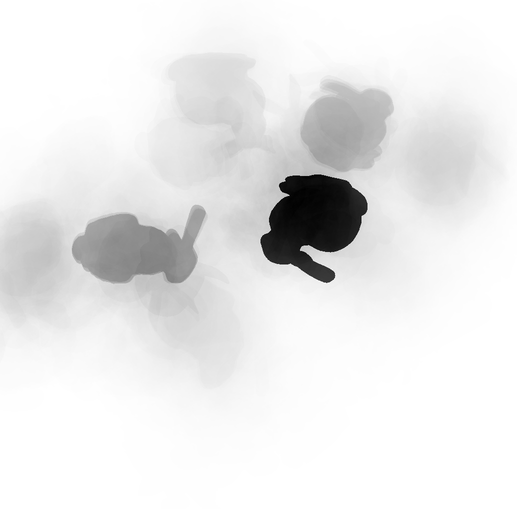}};
\begin{scope}[shift={(-3 * \xmodeoffset,0)}];
\def\wzero{649.818};
\def\wone{332.149};
\def\wtwo{181.151};
\def\wthree{152.338};
\def\wfour{139.325};
\def\wfive{131.656};
\pgfmathsetmacro\sone{{\wone / \wzero}};
\pgfmathsetmacro\stwo{{\wtwo / \wzero}};
\pgfmathsetmacro\sthree{{\wthree / \wzero}};
\pgfmathsetmacro\sfour{{\wfour / \wzero}};
\pgfmathsetmacro\sfive{{\wfive/ \wzero}};
\draw (0, {0.6 * \barscale}) node [rotate=90, align = center] {relative\\score};
\drawbargood{0}{1}{1}{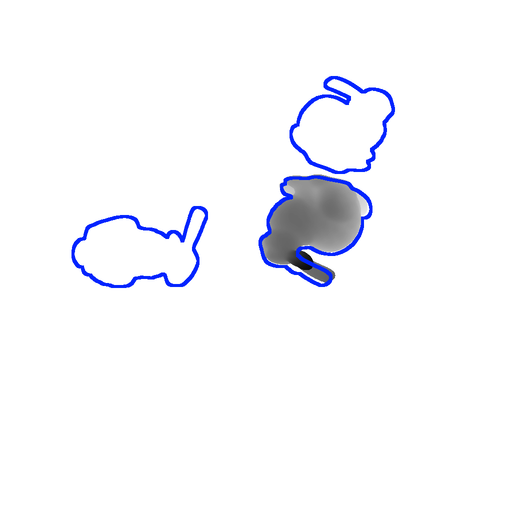};
\drawbargood{1}{2}{\sone}{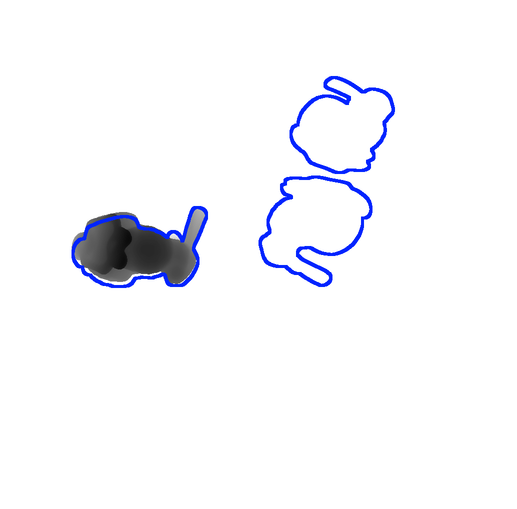};
\drawbargood{2}{3}{\stwo}{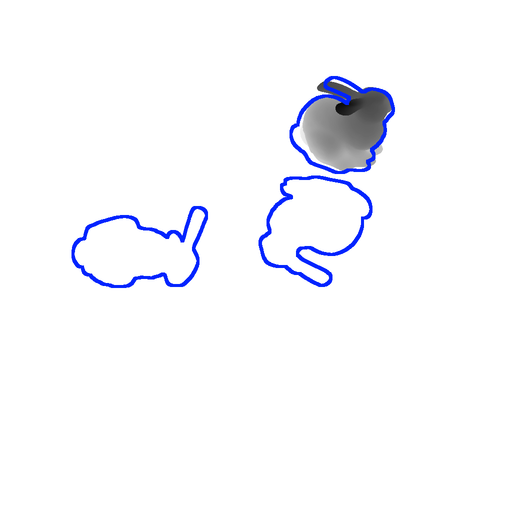};
\drawbarbad{3}{4}{\sthree}{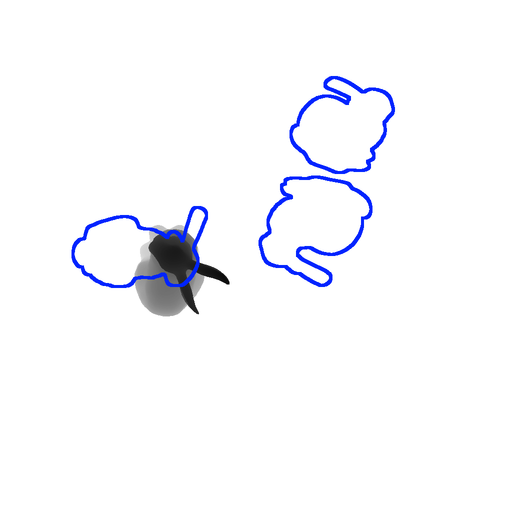};
\drawbarbad{4}{5}{\sfour}{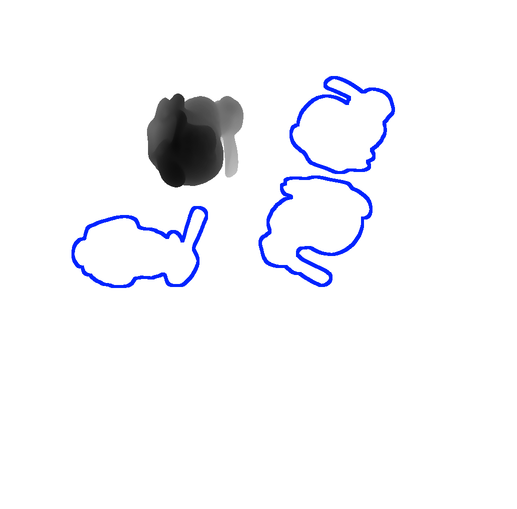};
\drawbarbad{5}{6}{\sfive}{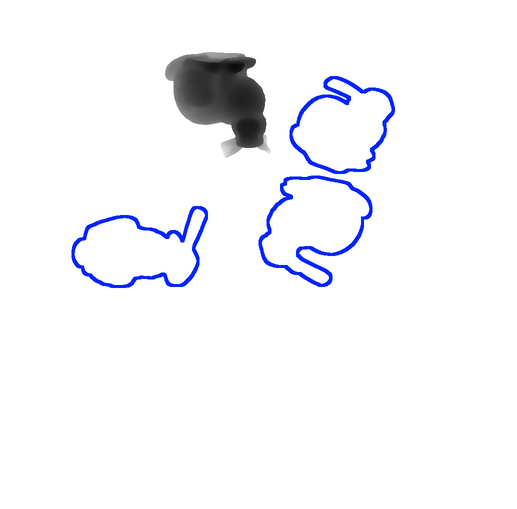};
\draw (0, 0) -- ({6 * \xmodeoffset}, 0);
\end{scope};
\end{scope};

\begin{scope}[shift={(0, {-1 * \yoffset})}];
\draw {(\xmodesaggregate, 0)} node [anchor=south west,inner sep=0] {\includegraphics[width=\imgwidth]{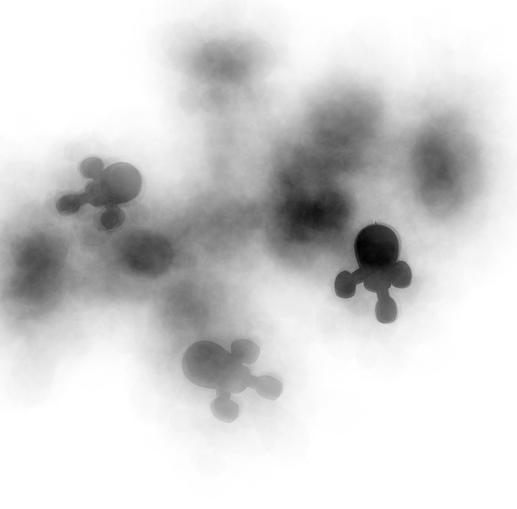}};
\begin{scope}[shift={(-3 * \xmodeoffset,0)}];
\def\wzero{196.605};
\def\wone{118.388};
\def\wtwo{102.116};
\def\wthree{89.5075};
\def\wfour{88.5678};
\def\wfive{78.993};
\pgfmathsetmacro\sone{{\wone / \wzero}};
\pgfmathsetmacro\stwo{{\wtwo / \wzero}};
\pgfmathsetmacro\sthree{{\wthree / \wzero}};
\pgfmathsetmacro\sfour{{\wfour / \wzero}};
\pgfmathsetmacro\sfive{{\wfive/ \wzero}};
\draw (0, {0.6 * \barscale}) node [rotate=90, align = center] {relative\\score};
\drawbargood{0}{1}{1}{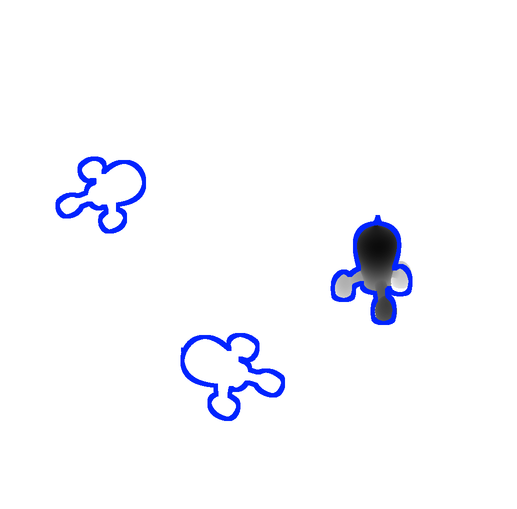};
\drawbarduplicate{1}{2}{\sone}{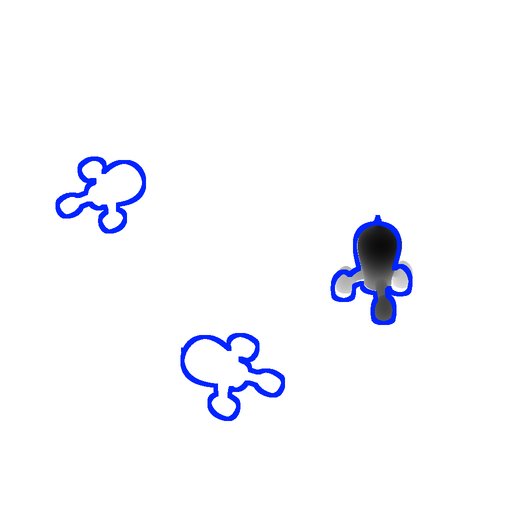};
\drawbargood{2}{3}{\stwo}{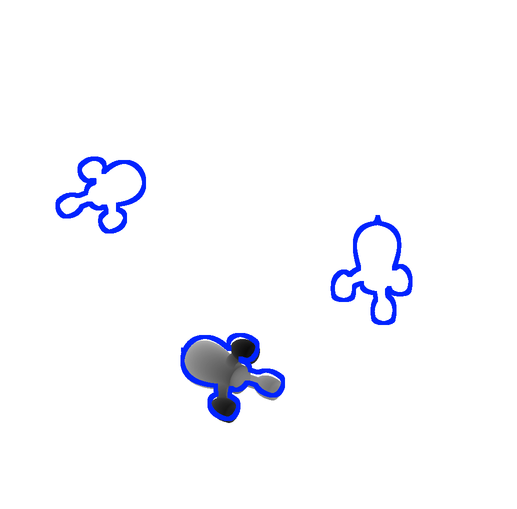};
\drawbargood{3}{4}{\sthree}{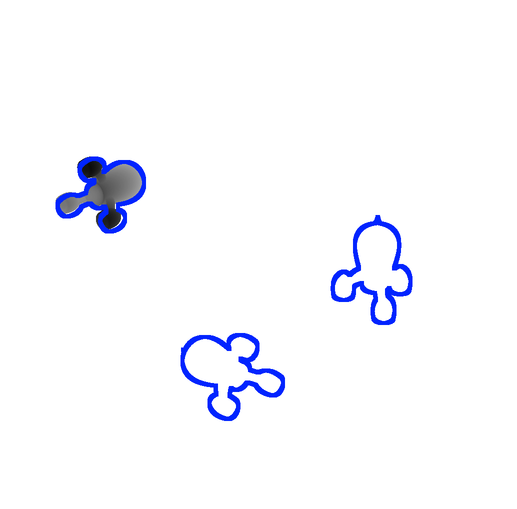};
\drawbarduplicate{4}{5}{\sfour}{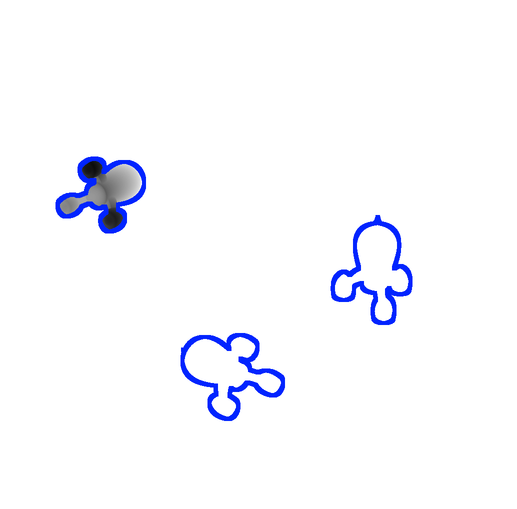};
\drawbarbad{5}{6}{\sfive}{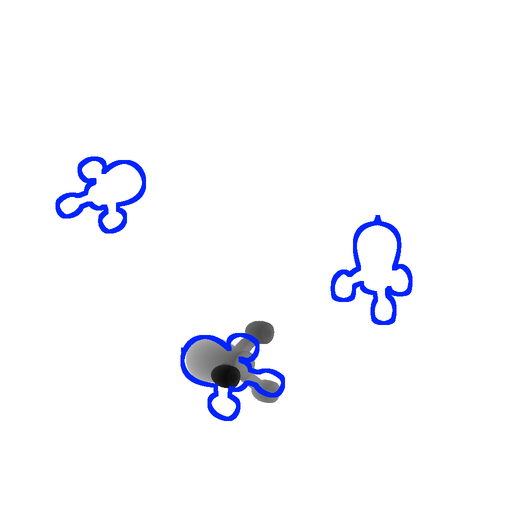};
\draw (0, 0) -- ({6 * \xmodeoffset}, 0);
\end{scope};
\end{scope};

\begin{scope}[shift={(0, {-2 * \yoffset})}];
\draw (0, -8pt) node [below, align = center] {(b) First modes of the pose distribution retrieved using the $SE(3)$ distance~\eqref{eq:distance_se3_experimental_comparison}.};
\draw {(\xmodesaggregate, 0)} node [anchor=south west,inner sep=0] {\includegraphics[width=\imgwidth]{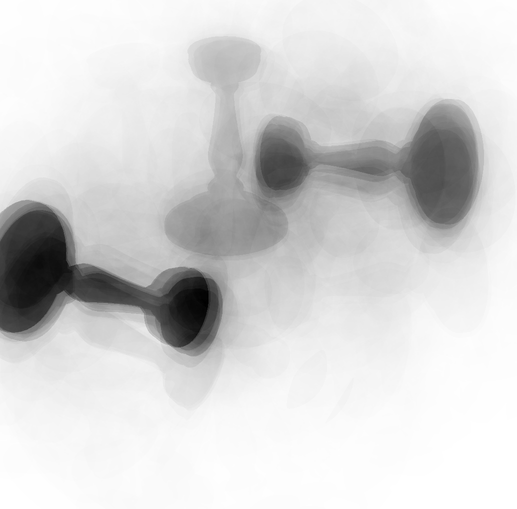}};
\begin{scope}[shift={(-3 * \xmodeoffset,0)}];
\def\wzero{485.365};
\def\wone{405.357};
\def\wtwo{264.003};
\def\wthree{253.508};
\def\wfour{251.887};
\def\wfive{245.773};
\def\wsix{243.878};
\def\wseven{230.597};
\def\wheight{227.006};
\pgfmathsetmacro\sone{{\wone / \wzero}};
\pgfmathsetmacro\stwo{{\wtwo / \wzero}};
\pgfmathsetmacro\sthree{{\wthree / \wzero}};
\pgfmathsetmacro\sfour{{\wfour / \wzero}};
\pgfmathsetmacro\sfive{{\wfive/ \wzero}};
\pgfmathsetmacro\ssix{{\wsix/ \wzero}};
\pgfmathsetmacro\sseven{{\wseven/ \wzero}};
\pgfmathsetmacro\sheight{{\wheight/ \wzero}};
\draw (0, {0.6 * \barscale}) node [rotate=90, align = center] {relative\\score};
\drawbargood{0}{1}{1}{z_img_example_modes_candlestick_peak_0_contours.png};
\drawbargood{1}{2}{\sone}{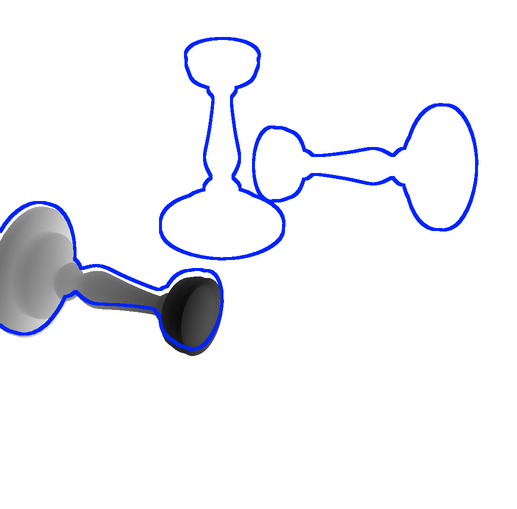};
\drawbarduplicate{2}{3}{\stwo}{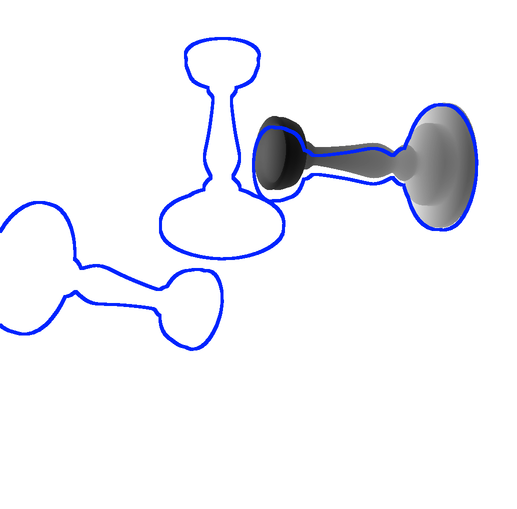};
\begin{scope}[shift={({(3 + 0.5) * \xmodeoffset}, 0)}];
\draw (0, 0) node [below] {\dots};
\draw (0, 1.7cm) node [text width=3cm, align=center] {4 other \\ duplicates.};
\end{scope};
\drawbargood{4}{8}{\sseven}{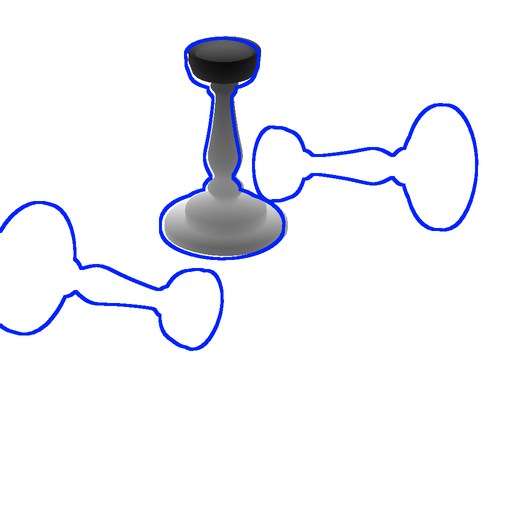};
\drawbarbad{5}{9}{\sheight}{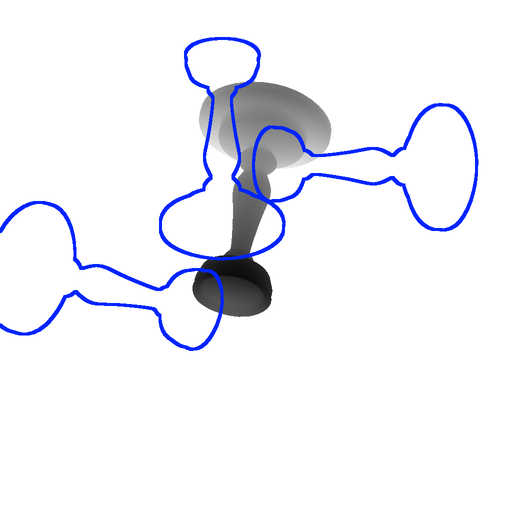};
\draw (0, 0) -- ({6 * \xmodeoffset}, 0);
\end{scope};
\end{scope};

\end{scope};
\end{tikzpicture}

\caption{\label{fig:comparison_our_distance_with_se3} Comparison of the proposed distance and a $SE(3)$ distance for pose estimation. Left: shifted poses weighted by the density
of the initial pose distribution. Right: first modes retrieved from the pose distribution, sorted by descending score (with the contours of the actual poses of object instances superimposed). Modes are supposed to be good hypotheses regarding the poses of actual instances, and are classified as true positives (blue), duplicates (green, strikethrough) and false positives (yellow, double strikethrough). Both distances perform similarly well for the \emph{bunny} object, which has no proper symmetry, and the first modes extracted correspond to the different object instances.
However, the $SE(3)$ distance does not account for symmetries of the \emph{rocket} and \emph{candlestick} objects, and therefore leads to the generation of duplicated pose hypotheses, requiring to consider many of them to recover the pose of every instances. The proposed distance better exploit the information contained in the initial pose distribution, leading to a generation of poses hypotheses with no duplicates, and with a greater relative score gap between modes corresponding to actual instances and spurrious ones.
}
\end{figure*}

\bibliographyInSubfile

\section{Summary and discussion}

In this paper, we address issues of the commonly used notion of pose of a rigid object, both in the 2D and 3D case.

While pose is usually assumed to be equivalent to a rigid transformation, this is not true in general due to potential symmetries of the object. We therefore propose a broader definition of the notion of pose, consisting in a distinguishable static state of the object. We show that with this definition, a pose can be considered as an equivalence class of the space of rigid transformations, thanks to the introduction of a proper symmetry group specific to the object. We believe this notion to be essential, as many of manufactured objects actually show some symmetry properties and could not be represented properly previously.

Based on this definition, we propose a metric over the pose space as a measure of the smallest displacement between two poses, the length of a displacement consisting in the RMS displacement distance of surface points of the object. Besides being defined for any physical rigid object, such metric is interesting in that it does not depend on some arbitrary choice of frames or of scaling factors, while accounting for the geometry of the object.

With computation efficiency in mind, we propose a coherent framework to represent poses in a Euclidean space of at most 12 dimensions, so as to enable efficient distance computations, neighborhood queries, and pose averaging, while providing theoretical proofs for those results.
 
Those developments enable the use of our metric for high level tasks such as pose estimation based on a set of votes, where it appears to provide better results than a metric suited for $SE(3)$.


\begin{acknowledgements}
We would like to thank the anonymous reviewers for their insightful comments and suggestions that greatly helped to improve this article.
Some of our illustrations are based on the following mesh models: ``Stanford bunny'', from the Stanford University Computer Graphics Laboratory; ``Eiffel Tower'' created by Pranav Panchal; and ``Şamdan 2'' (candlestick), from Metin N. Those were respectively available online at 
\myurl{http://graphics.stanford.edu/data/3Dscanrep}, and the GrabCAD and 3D Warehouse plateforms on May 2016.
\end{acknowledgements}

\appendix
\section{Distance simplification for a revolution object without rotoreflection invariance}
\label{app:proof_distance_simplification_revolution_object}

Using the same definition of $\matLambda$ as in section \ref{sec:no_invariance_object}, the rotation part of the proposed distance for a revolution object without rotoreflection invariance can be rewritten in the following way:
\begin{equation}
\label{eq:distance_rotation_revolution_object_via_covariance}
\begin{aligned}
&\mydistrotsquare(\mathcal{P}_1, \mathcal{P}_2) \\
&= \min_{\phi_1, \phi_2} \cfrac{1}{S} \int_{\mathcal{S}} \mu(\mat{x}) \| \mat{R}_2 \mat{R}_z^{\phi_2} \mat{x}  - \mat{R}_1 \mat{R}_z^{\phi_1} \mat{x}\|^2 ds \\
&=\min_{\phi_1, \phi_2} \| \mat{R}_2 \mat{R}_z^{\phi_2} \matLambda - \mat{R}_1 \mat{R}_z^{\phi_1} \matLambda \|_F^2.
\end{aligned}
\end{equation}

Frobenius norm being invariant under rotations, this expression can be rewritten with the relative rotation $\mat{R} \triangleq \mat{R}_1^{-1} \mat{R}_2$:
\begin{equation}
\mydistrotsquare(\mathcal{P}_1, \mathcal{P}_2) = \min_{\phi_1, \phi_2} \| \mat{R}_z^{-\phi_1} \mat{R} \mat{R}_z^{\phi_2} \matLambda -   \matLambda \|_F^2.
\end{equation}

We parametrize $\mat{R}$ using Euler angles $(\tilde{\psi}, \theta, \tilde{\phi}) \in \mathbb{R}^3$ such as $\mat{R} = \mat{R}_z^{\tilde{\psi}} \mat{R}_x^\theta \mat{R}_z^{\tilde{\phi}}$, considering the following elementary rotations:
\begin{equation}
\label{eq:elementary_rotations_definition}
\mat{R}_z^{\alpha} \triangleq 
\left( \begin{smallmatrix}
\cos(\alpha) & -\sin(\alpha) & 0 \\
\sin(\alpha) & \cos(\alpha) & 0 \\
0 & 0 &1
\end{smallmatrix} \right), 
\mat{R}_x^{\alpha} \triangleq 
\left( \begin{smallmatrix}
1 & 0 & 0 \\
0 & \cos(\alpha) & -\sin(\alpha) \\
0 & \sin(\alpha) & \cos(\alpha)
\end{smallmatrix} \right).
\end{equation}
Injecting this parametrization into the previous expression and performing the changes of variables $\psi \leftarrow \tilde{\psi} -\phi_1$ and $\phi \leftarrow \tilde{\phi} + \phi_2$ leads us to the following expression:
\begin{equation}
\begin{aligned}
\mydistrotsquare(\mathcal{P}_1, \mathcal{P}_2) &=   \min_{\phi_1, \phi_2} \| \mat{R}_z^{-\phi_1} \mat{R}_z^{\tilde{\psi}} \mat{R}_x^\theta \mat{R}_z^{\tilde{\phi}} \mat{R}_z^{\phi_2} \matLambda -   \matLambda \|_F^2 \\
&= \min_{\psi, \phi} \| \mat{R}_z^\psi \mat{R}_x^\theta \mat{R}_z^\phi \matLambda -  \matLambda \|_F^2.
\end{aligned}
\end{equation}
Because of the specific shape of $\matLambda$ (equation~\ref{eq:lambda_expression_revolution}), the term to minimize can be decomposed into two parts:
\begin{multline}
\| \mat{R}_z^\psi \mat{R}_x^\theta \mat{R}_z^\phi \matLambda -  \matLambda \|_F^2 = \lambda_z^2 \underbrace{\| \mat{R}_z^\psi \mat{R}_x^\theta \mat{R}_z^\phi \mat{e}_z - \mat{e}_z \|^2}_{a_{\psi, \phi}} \\
+ \!\lambda_r^2 \underbrace{(\| \mat{R}_z^\psi \mat{R}_x^\theta \mat{R}_z^\phi \mat{e}_x -\mat{e}_x \|^2  \!+ \| \mat{R}_z^\psi \mat{R}_x^\theta \mat{R}_z^\phi \mat{e}_y - \mat{e}_y \|^2)}_{b_{\psi, \phi}}.
\end{multline}

Developing this expression thanks to the definition of the elementary rotations \eqref{eq:elementary_rotations_definition}, we evaluate those terms into:
\begin{equation}
\left\lbrace
\begin{aligned}
a_{\psi, \phi} &= 2(1-\cos(\theta)) \\
b_{\psi, \phi} &= 4 - 2 \cos(\psi + \phi)(1 + \cos(\theta))). \\
\end{aligned}
\right.
\end{equation}
The first term is independent of $\psi$ and $\phi$. The second one can be minimized easily relatively to those two parameters, and  admits a minimum that appears to be equal to the first term: 
\begin{equation}
\min_{\psi, \phi} b_{\psi, \phi} = 2(1-\cos(\theta)). \\
\end{equation}

This result enables us to estimate the distance between the two poses in a closed form. However, having to refer to a relative rotation between the two poses and perform an Euler decomposition is cumbersome and would not enable to propose a representation of a pose efficient for neighborhood queries. We prefer instead to use the following property
\begin{equation}
\begin{aligned}
2(1-\cos(\theta)) &= \| \mat{R} \mat{e}_z - \mat{e}_z \|^2 \\
&= \| \mat{R}_2 \mat{e}_z - \mat{R}_1 \mat{e}_z \|^2
\end{aligned}
\end{equation}
in order to express the rotation part of the square distance as a function of the distance between the revolution axes of the object at the two poses:
\begin{equation}
\mydistrotsquare(\mathcal{P}_1, \mathcal{P}_2) = (\lambda_r^2 + \lambda_z^2) \| \mat{R}_2 \mat{e}_z - \mat{R}_1 \mat{e}_z \|^2.
\end{equation}


\section{Minimum distance between representatives of the same pose}
\label{ap:T_value_application_example}
In this appendix, we show how to compute the minimum distance  T between representatives of the same pose (see the definition~\ref{def:definition_T}) for the objects of our application example.

The bunny and the candlestick admit one representative per pose, hence $T=+\infty$ for those by convention.

The case of the rocket requires some calculus. For the sake of simplicity we consider an object frame whose z axis corresponds to the symmetry axis of the rocket. In this frame, the proper symmetry group of the rocket can be expressed as
\begin{equation}
G= \left\lbrace \mat{I}, \mat{R}_z^{2 \pi/3}, \mat{R}_z^{-2 \pi/3} \right\rbrace
\end{equation}
and the square root of the covariance matrix as
\begin{equation}
\matLambda= \diag(\lambda_r, \lambda_r, \lambda_z).
\end{equation}

We choose to consider the reference pose $\mathcal{P}_0$ and one of its representatives $\mat{p}$ (underbraced below) for the computation of T as it makes the computation simpler. Representatives $\mathcal{R}(\mathcal{P}_0)$ of this pose are
\begin{equation}
\left\lbrace 
\underbrace{\left( \begin{matrix} \vect(\matLambda) \\ \mat{0}_3 \end{matrix} \right)}_{\mat{p}},
\left( \begin{matrix} \vect(\mat{R}_z^{2 \pi/3} \matLambda) \\ \mat{0}_3 \end{matrix} \right),
\left( \begin{matrix} \vect(\mat{R}_z^{-2 \pi/3} \matLambda) \\ \mat{0}_3 \end{matrix} \right)
\right\rbrace.
\end{equation}

Thanks to those choices, we can evaluate T into:
\begin{equation}
\begin{aligned}
T&= \min_{\mat{q} \in \mathcal{R}(\mathcal{P}_0), \mat{q}\neq \mat{p}} \| \mat{q} - \mat{p} \| \\
&= \min \| \mat{R}_z^{\pm 2 \pi/3} \matLambda - \matLambda\|_F\\
&= \sqrt{6} \lambda_r.
\end{aligned}
\end{equation}

The threshold $\cfrac{T}{4}$ of proposition~\ref{prop:condition_representative_combination_consistency_ball} therefore corresponds  for the rocket to the value $\cfrac{\sqrt{3}}{2} \lambda_r$.

\section{Numerical recipes for a triangular mesh}
\label{sec:triangular_mesh}
Center of mass, area and covariance matrix of the surface of a triangular mesh $\mathcal{S} = \bigcup_i \mathcal{T}(\mat{a}_i, \mat{b}_i, \mat{c}_i)$ -- where $\mathcal{T}(\mat{a}, \mat{b}, \mat{c})$ is a triangle defined by three vertices $\mat{a}, \mat{b}, \mat{c} \in \mathbb{R}^3$  -- can be computed easily through the contributions of its triangles.

Let $\mathcal{T}(\mat{a}, \mat{b}, \mat{c})$ be a given triangle. Its area can be computed thanks to a cross product:
\begin{equation}
S_{\mat{a}, \mat{b}, \mat{c}} = \cfrac{\|(\mat{b} - \mat{a}) \times (\mat{c} - \mat{a})\|}{2},
\end{equation}
its center of mass through:
\begin{equation}
\mat{o}_{\mat{a}, \mat{b}, \mat{c}} = \cfrac{\mat{a} + \mat{b} + \mat{c}}{3},
\end{equation}
and its uncentered covariance matrix via:
\begin{equation}
\matSigma_{\mat{a}, \mat{b}, \mat{c}} = \cfrac{S_{\mat{a}, \mat{b}, \mat{c}}}{12} \left(9 \mat{o}_{\mat{a}, \mat{b}, \mat{c}} \mat{o}_{\mat{a}, \mat{b}, \mat{c}}^\top + \mat{a} \mat{a}^\top + \mat{b} \mat{b}^\top + \mat{c} \mat{c}^\top \right).
\end{equation}
From those results, we deduce the expression of the surface area of the mesh:
\begin{equation}
S= \sum_i S_{\mat{a}_i, \mat{b}_i, \mat{c}_i},
\end{equation}
its center of mass:
\begin{equation}
\mat{o}= \sum_i S_{\mat{a}_i, \mat{b}_i, \mat{c}_i} \mat{o}_{\mat{a}_i, \mat{b}_i, \mat{c}_i},
\end{equation}
and its normalized covariance matrix, if the center of mass of the mesh is chosen as origin of the object frame:
\begin{equation}
\matLambda^2 = \cfrac{1}{S} \sum_i \matSigma_{\mat{a}_i, \mat{b}_i, \mat{c}_i}.
\end{equation}

\bibliographystyle{spbasic}      
\bibliography{biblio}   

\end{document}